\PassOptionsToPackage{linesnumbered}{algorithm2e}
\documentclass[final,12pt]{colt2025} 

\title[Fast and Furious Symmetric Learning in Zero-Sum Games]
{Fast and Furious Symmetric Learning in Zero-Sum Games: \\
  Gradient Descent as Fictitious Play}

\usepackage{times}

\hypersetup{
  colorlinks=true,
  citecolor=blue,
  linkcolor=magenta,
}
\usepackage{notation}
\usepackage{thmtools}
\usepackage{bbm}
\usepackage{svg}
\usepackage{enumitem}
\usepackage{thm-restate}
\usepackage[normalem]{ulem}
\newcommand{\blackdiam}{\rotatebox[origin=c]{45}{$\blacksquare$}}
\usepackage{multicol}
\usepackage{subcaption}

\coltauthor{%
  \Name{John Lazarsfeld} \Email{john\_lazarsfeld@sutd.edu.sg} \\
  \addr Singapore University of Technology and Design
  \AND
  \Name{Georgios Piliouras} \Email{gpil@google.com}\\
  \addr Google DeepMind 
  \AND
  \Name{Ryann Sim} \Email{ryann\_sim@sutd.edu.sg} \\
  \addr Singapore University of Technology and Design
  \AND
  \Name{Andre Wibisono} \Email{andre.wibisono@yale.edu} \\
  \addr Yale University
}

\begin{document}

\maketitle

\vspace*{-2em}
\begin{abstract}

This paper investigates the sublinear regret guarantees
of two \textit{non}-no-regret algorithms in zero-sum games:
Fictitious Play, and Online Gradient Descent with
\textit{constant} stepsizes.
In general adversarial online learning settings, both algorithms
may exhibit instability and linear regret 
due to no regularization (Fictitious Play) 
or small amounts of regularization (Gradient Descent).
However, their ability to obtain tighter regret bounds in
two-player zero-sum games is less understood.
In this work, we obtain strong new regret guarantees for both
algorithms on a class of symmetric zero-sum games 
that generalize the classic three-strategy Rock-Paper-Scissors
to a weighted, $n$-dimensional regime.
Under \textit{symmetric initializations} of the players' strategies,
we prove that Fictitious Play with \textit{any tiebreaking rule} 
has $O(\sqrt{T})$ regret, establishing a new class of games 
for which  Karlin's Fictitious Play conjecture holds.
Moreover, by leveraging a connection between the geometry of 
the iterates of Fictitious Play and Gradient Descent 
in the dual space of payoff vectors, 
we prove that Gradient Descent, for \textit{almost all} 
symmetric initializations, obtains a similar $O(\sqrt{T})$ regret 
bound when its stepsize is a \textit{sufficiently large} constant.
For Gradient Descent, this establishes the first
``fast and furious'' behavior  (i.e., sublinear regret
\textit{without} time-vanishing stepsizes) 
for zero-sum games larger than $2\times2$. 

\begin{figure}[h!]
\centering
\subfigure[]{\label{fig:FPdual}
    \includegraphics[width=0.25\textwidth]{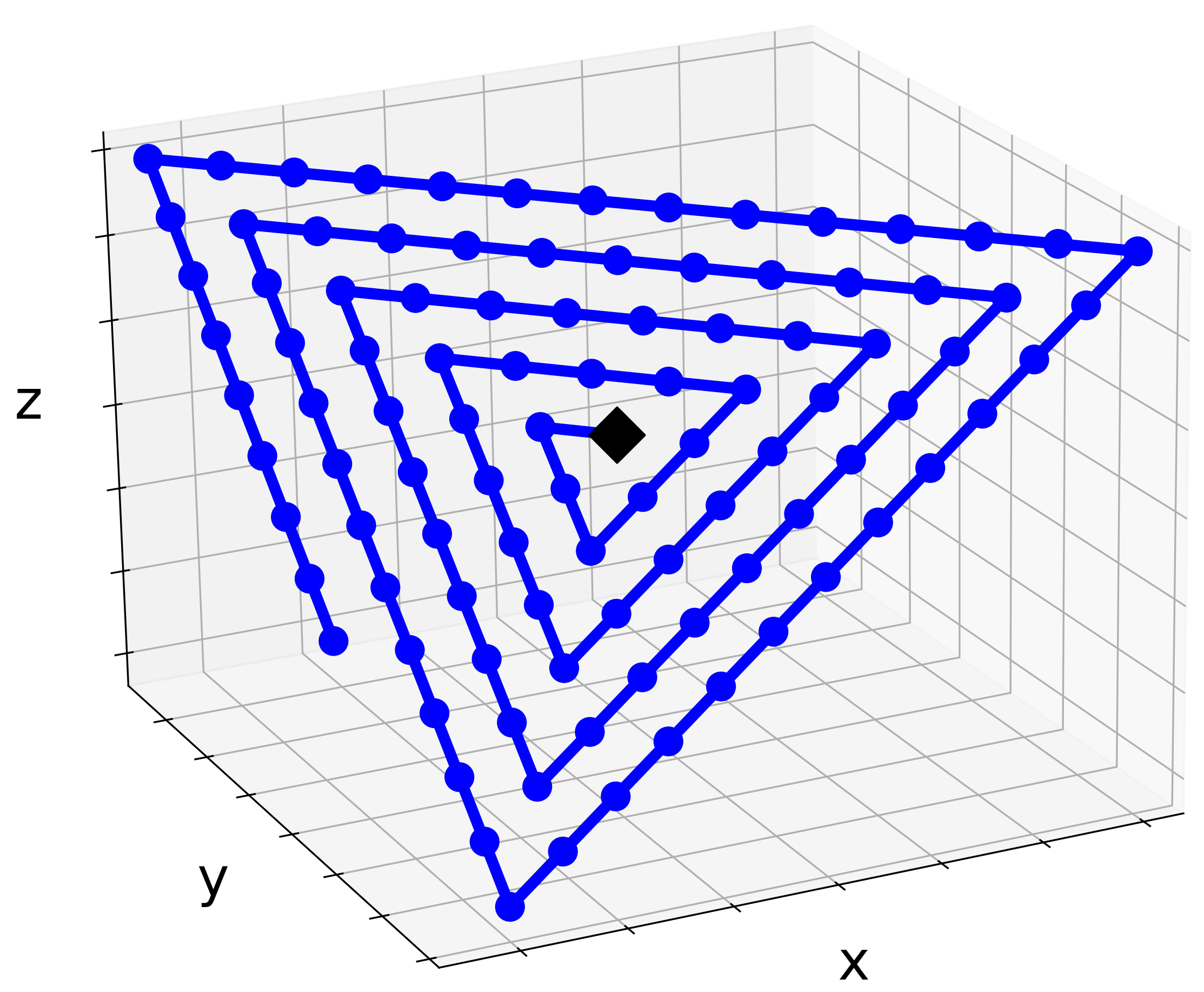}}
\subfigure[]{\label{fig:GDdual}
    \includegraphics[width=0.25\textwidth]{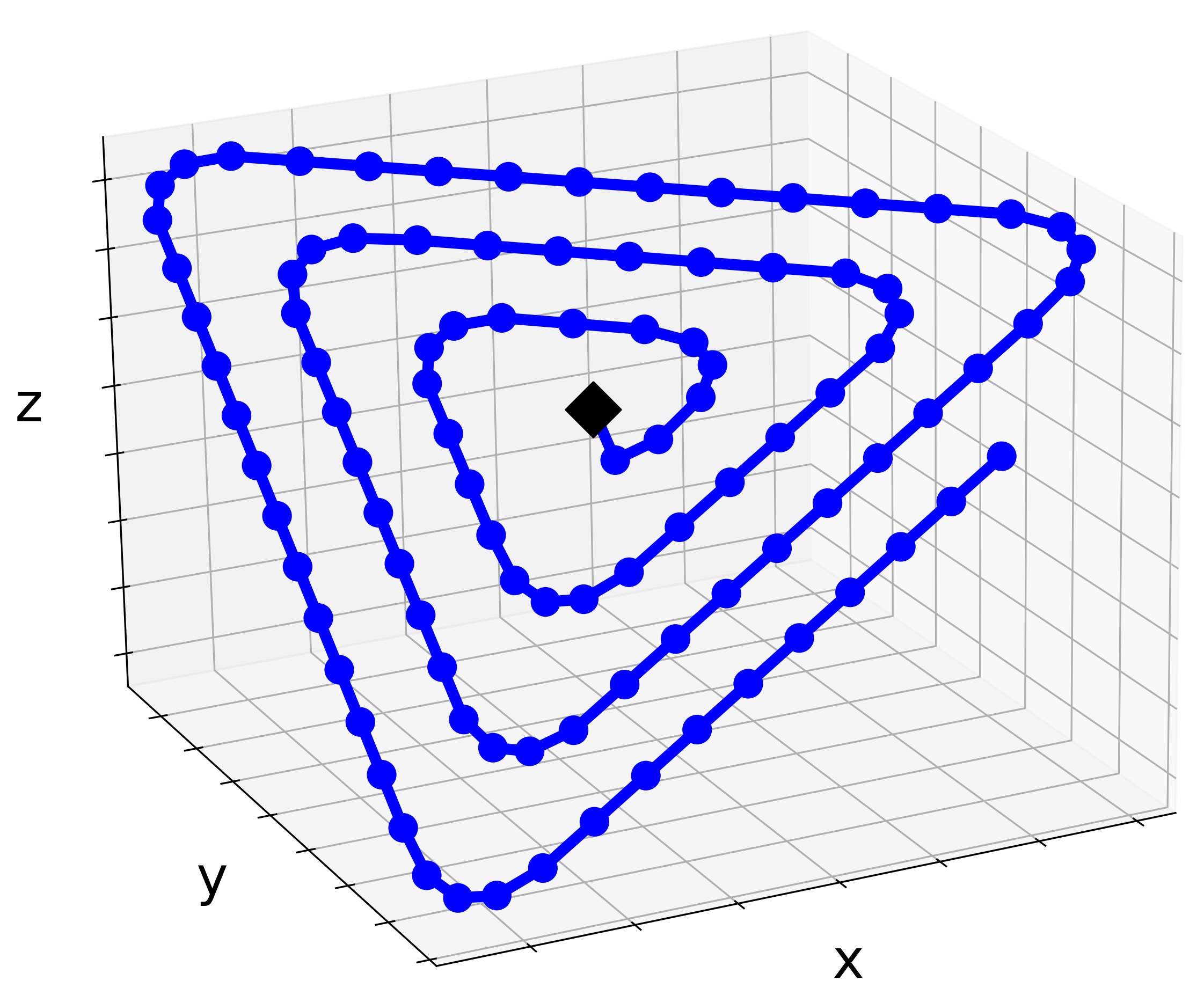}}
\subfigure[]{\label{fig:GDPrimal}
    \includegraphics[width=0.25\textwidth]{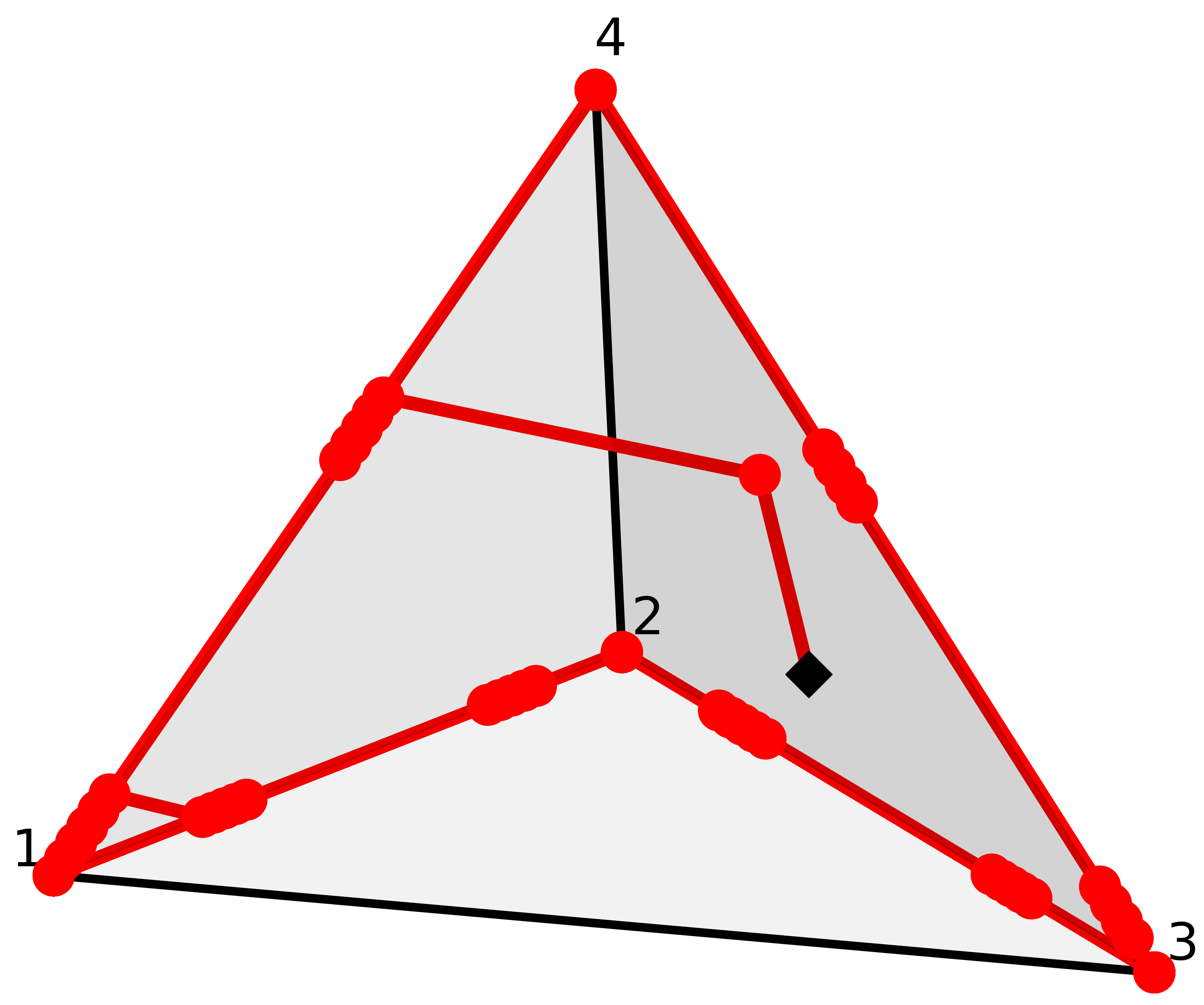}}
\caption{
    \footnotesize
    In $n=3$ Rock-Paper-Scissors for $200$ iterations
    and initialized at  $x^0=[1,0,0]$:
    (\emph{a}) the dual iterates of Fictitious Play, and
    (\emph{b}) the dual iterates of Gradient Descent with
    stepsize $\eta=0.5$.
    In $n=4$ Rock-Paper-Scissors initialized at
    $x^0 = [0.05, 0.35, 0.39, 0.21]$:
    (\emph{c}) the primal iterates of Gradient Descent
    using stepsize $\eta=1$.
    (\scalebox{0.6}{$\blackdiam$}) denotes the initial
    iterate of the dynamics.
    }
    \label{fig:abstract}
\end{figure}

\end{abstract}


\newpage
\section{Introduction}
\label{sec:intro}

In adversarial or adaptive online learning settings,
regularization is the key ingredient for obtaining optimal 
$O(\sqrt{T})$  regret bounds over $T$ iterations~\citep{cesa2006prediction,hazan2016introduction}.
Indeed, \textit{un}-regularized algorithms like Follow-The-Leader 
(FTL) are known to be highly-sensitive to even slight oscillations in 
reward sequences, leading to regret scaling linearly in $T$ in the 
worst case. Using online Gradient Descent (GD), an instantiation of 
the Follow-The-Regularized-Leader (FTRL) family, this sensitivity
to oscillating rewards is circumvented via regularization,
the magnitude of which is controlled 
via a \textit{stepsize} parameter $\eta > 0$:
smaller settings of $\eta$ correspond to more regularization,
and the iterates of the algorithm become more stable. 
Using time-horizon-dependent settings of $\eta \approx 1/\sqrt{T}$, 
it is well known that Gradient Descent and other instantiations 
of FTRL like Multiplicative Weights Update 
obtain optimal $O(\sqrt{T})$ regret bounds for adversarial online learning
(see e.g.~\cite{shalev2012online}). 

\smallskip

In the setting of two-player zero-sum games, 
a well-known connection with online learning 
establishes that the time-averaged iterates
of two players using \textit{no-regret algorithms}
(e.g., algorithms that guarantee \textit{sublinear} regret in any 
adaptive setting) converge to 
a Nash Equilibrium of the underlying game
\citep{hannan1957approximation,freund1999adaptive}. 
While using Gradient Descent or any FTRL algorithm with 
a time-horizon-dependent or time-decaying stepsize
guarantees such convergence, in the context of modeling strategic
interactions between agents, selecting stepsizes depending 
inversely on $T$ is somewhat unrealistic,
and it is natural to question whether 
algorithms for learning in games require
the same amount of regularization as 
in worst-case, adversarial settings. 
For example, the Fictitious Play (FP) algorithm 
of~\cite{brown1951iterative} is the result of both players
using the \textit{un}regularized FTL algorithm.
While FTL is \textit{not} a no-regret algorithm
in general, \cite{robinson1951iterative}
proved that the sum of the players' regrets \textit{is}
in fact sublinear in the zero-sum game setting, 
although scaling like $O(T^{1-1/(2n-2)})$ for $n\times n$ games. 
Later, \cite{karlin1959mathematical}
conjectured that in all zero-sum games, 
Fictitious Play achieves an even stronger 
regret guarantee of only $O(\sqrt{T})$. 

\smallskip

Karlin's conjecture was seemingly put to rest 
by~\cite{daskalakis2014counter}, who showed that even for
the $n$-dimensional identity payoff matrix,
when using a specific \textit{adversarial} tiebreaking rule,
Fictitious Play has regret scaling like $\Omega(T^{1-1/n})$. 
Moreover, for FTRL algorithms like Gradient Descent in the
\textit{under}-regularized regime of \textit{constant stepsizes}
(e.g., with no dependence on $T$), 
\cite{bailey2018multiplicative} showed that
the day-to-day iterates of these
algorithms are \textit{repelled} from interior 
Nash equilibria and converge towards the boundary of the simplex. 
Using FTRL algorithms with constant stepsizes was 
also shown to exhibit formally chaotic behavior, both in  
congestion/coordination games~\citep{palaiopanos2017multiplicative,
  chotibut2020route,chotibut2021family,bielawski2021follow}
as well as (symmetric) zero-sum games~\citep{cheung2019vortices,
  cheung2020chaos,cheung2022evolution}. 
In fact, in congestion/coordination games,
chaotic as well as more generally unstable behavior
(i.e., limit cycles) can result in provably
non-vanishing regret~\citep{blum2006routing,chotibut2020route}.
Such results seemed to indicate a much narrower path for success for 
Fictitious Play and other under-regularized, \textit{non}-no-regret 
learning algorithms to  achieve strong regret guarantees in
zero-sum games. 

\smallskip

On the other hand, several newer works have painted
a more optimistic picture: for example,
\cite{bailey2019fast} showed that
the day-to-day \textit{non}-equilibration behavior of 
Gradient Descent can actually be leveraged 
as a \textit{tool} for controlling
the algorithm's total regret. 
For $2 \times 2$ zero-sum games with a unique interior equilibrium, 
they proved that running Gradient Descent with
constant stepsizes
still leads to $O(\sqrt{T})$ regret,
exhibiting so-called ``fast and furious'' behavior
(i.e., regret-minimization 
 with constant stepsizes instead of vanishing/decreasing ones). 
However, their result left open the question
of establishing similar regret
guarantees for classes of larger, high-dimensional zero-sum games. 
\cite{abernethy2021fast} have also provided
a more nuanced view of the negative result
of~\cite{daskalakis2014counter} for Fictitious Play: they proved
for all $n\times n$ \textit{diagonal} payoff matrices that
Fictitious Play with a fixed \textit{lexicographical} tiebreaking rule indeed obtains $O(\sqrt{T})$
regret, thus proving (the weak version 
of) Karlin's Conjecture for this class. 

\smallskip

However, little progress has since been made on obtaining
similar, tighter regret bounds for either Fictitious Play 
or Gradient Descent in the constant stepsize regime.
Despite the increasing interest in the \textit{optimistic variants}
of FTRL algorithms (which can be used to obtain last-iterate 
convergence in zero-sum games~\citep{rakhlin2013optimization,syrgkanis2015fast,cai2024fast}),
the behavior and guarantees of their standard, non-optimistic counterparts 
still lacks a more fine-grained understanding,
even for simplified settings such 
as \textit{symmetric} zero-sum games
under \textit{symmetric learning} (i.e., 
when the players use the same initialization).
Can Fictitious Play and FTRL with constant stepsizes
minimize regret in such settings,
despite their un(der)-regularization?

\paragraph{Our Contributions.}
In this paper, we make new progress towards obtaining 
tighter regret bounds
for both Fictitious Play and Gradient Descent in zero-sum games. 
We study these dynamics under \textit{symmetric learning} on 
a large class of \textit{high-dimensional rock-paper-scissors (RPS)}
payoff matrices (Definition~\ref{def:rps}).
These games  generalize the canonical three-strategy Rock-Paper-Scissors
game -- perhaps the most well-studied symmetric zero-sum game -- 
to a weighted, $n$-dimensional regime. 
For \textit{all $n$-dimensional RPS matrices}, 
we prove that under \textit{symmetric learning}:

\begin{itemize}[
    rightmargin=2em
]
\item 
  (Theorem~\ref{thm:fp-rps}):
  Fictitious Play from any initialization obtains 
  worst-case $O(\sqrt{T})$ regret
  using \textit{any} arbitrary tie-breaking rule. 
\item 
  (Theorem~\ref{thm:gd-large-stepsize}):
  Gradient Descent, 
  from \textit{almost all} initializations
  and with \textit{sufficiently large constant stepsizes},
  obtains $O(\sqrt{T})$ regret. 
\end{itemize} 

Our results provide new evidence 
on the robustness of these \textit{non}-no-regret algorithms 
in obtaining sublinear regret in zero-sum games 
(and as a consequence, fast time-averaged convergence to Nash equilibria).
For Fictitious Play, Theorem~\ref{thm:fp-rps}
establishes a new class of high-dimensional
zero-sum games for which Karlin's Conjecture holds,
and in contrast to  \cite{abernethy2021fast} 
it does not rely on using a specific tiebreaking rule. 
For Gradient Descent with constant stepsizes,
Theorem~\ref{thm:gd-large-stepsize} establishes
the first class of \textit{large} zero-sum games 
(beyond the $2\times2$ case studied by~\cite{bailey2019fast})
with provable $O(\sqrt{T})$ regret.

\smallskip

Our proof techniques build off of recent work identifying
the Hamiltonian structure of \textit{continuous}-\textit{time} FTRL
dynamics~\citep{mertikopoulos2018cycles,
    bailey2019multi, 
    wibisono2022alternating}. 
We show in the dual space of payoff vectors
that the iterates of both \textit{discrete-time} Fictitious Play 
and Gradient Descent can be viewed as a \textit{skew-gradient descent}
with respect to a corresponding \textit{energy function},
which is closely related to regret. 
Using this geometric perspective, 
we establish a \textit{cycling} 
property in the primal space that is shared by
the iterates of both algorithms on RPS matrices.
In turn, this leads to a regularity in energy growth 
that allows us to obtain strong sublinear regret guarantees.

\smallskip

Somewhat surprisingly, we show for Gradient Descent 
that a \textit{sufficiently large}
constant stepsize is the key driver for establishing
this regularity.
This comes in stark contrast to much of the
conventional wisdom in online learning
(where small or even vanishing stepsizes
are believed necessary for obtaining low regret),
as well as in games settings 
(where numerous results suggest that
using large, constant stepsizes can lead to
unpredictable, chaotic behavior). 
To the contrary, we show that large stepsizes 
are precisely the catalyst for proving the 
desirable ``fast and furious'' behavior,
and they allow for viewing Gradient Descent 
and Fictitious Play from a shared perspective.


\section{Preliminaries}
\label{sec:prelims}

Let $[n] := \{1,\dots,n\}$.
Let $\Delta_n = \{x = (x_1,\dots,x_n) \in \R^n : \sum_{i\in[n]} x_i = 1, ~ x_i \ge 0 ~ \forall ~ i \in [n]\}$ denote the probability simplex in $\R^n$.
For $x \in \Delta_n$, let 
$\supp(x) = \{i \in [n]: x_i > 0\}$
denote the support of $x$.
If $\supp(x) = [n]$, then we say $x$ is \textit{interior}. 
For $x,y \in \R^n$, we denote the $\ell_2$-inner product by $\langle x,y \rangle = x^\top y = \sum_{i=1}^n x_i y_i$, and the $\ell_2$-norm by $\|x\|_2 = \sqrt{\langle x,x \rangle}$.

\subsection{Online Learning in Two-Player Zero-Sum Games}
\label{sec:prelims:online}
We consider a setting where two players (Player 1 and Player 2)
repeatedly play a zero-sum game with payoff matrix $A \in \R^{m \times n}$
over a sequence of $T$ rounds. 
At each round $t = 0, \dots, T$, the players choose mixed strategies 
$x^t_1 \in \Delta_m$ and $x^t_2 \in \Delta_n$ belonging to 
the $m$ and $n$-dimensional simplices.
Players 1 and 2 then obtain expected payoffs
$\langle x^t_1, A x^t_2\rangle$
and $-\langle x^t_2, A^\top x^t_1\rangle$,
and they observe the vector feedback $Ax^t_2$ and $-A^\top x^t_1$,
respectively.

The goal of each player is to maximize
their cumulative expected payoffs.
In this online learning setting, we quantify Player $i$'s
performance by measuring its \textit{regret} $\reg_i(T)$,
which is the difference between its cumulative expected payoff 
and the cumulative payoff of the best fixed strategy in hindsight
after $T$ rounds. 
More precisely, we define the regrets of Player $1$ and Player $2$ to be:
\begin{align*}
  \reg_1(T) &:=
  \max_{x_1 \in \Delta_m}\;
  \sum_{t=0}^T \langle x_1, Ax^t_2 \rangle 
  - \sum_{t=0}^T \langle x^t_1, Ax^t_2 \rangle 
  \;\\
  \reg_2(T) &:=
  \sum_{t=0}^T \langle x^t_1, Ax^t_2 \rangle 
  -\min_{x_2 \in \Delta_n}\;
  \sum_{t=0}^T \langle x_2, A^\top x^t_1 \rangle 
  \;.
\end{align*}
We define the \textit{total regret} of the two players to be
$\reg(T) := \reg_1(T) + \reg_2(T)$.

A well-known connection between online learning and 
game theory establishes that bounds on the growth rate 
of $\reg(T)$ implies the convergence of the players'
time-averaged mixed strategies to a \textit{Nash Equilibrium} (NE),
the canonical solution concept when 
two players play $A$ for a single round. 
Recall that an NE for $A$ is a pair
$(x_1^*, x_2^*) \in \Delta_{m} \times \Delta_n$
satisfying the inequalities
\begin{equation*}
\max_{x_1 \in \Delta_m}\; \langle x_1, A x^*_2 \rangle
\le
\langle x^*_1, A x^*_2 \rangle
\le
\min_{x_2 \in \Delta_n}\; \langle x^*_1, A x_2 \rangle \;.
\end{equation*}
Due to von Neumann's minimax theorem~\citep{v1928theorie},
every zero-sum game has at least one NE. 
Then the following relationship holds 
(see Appendix \ref{app:prelims:nashregret} for a proof):

\begin{restatable}{prop}{regretnash}
\label{prop:regret-nash}
Let $\tilde x_1^T := (\sum_{t=0}^T x^t_1)/T$
and $\tilde x_2^T := (\sum_{t=0}^T x^t_2)/T$
denote the time-averaged strategies of Players 1 and 2, respectively.
If $\reg(T) = o(T)$, then $(\tilde x^T_1, \tilde x^T_2)$
converges to an NE $(x^*_1, x^*_2)$
in duality gap at a rate of $\reg(T)/T = o(1)$.
\end{restatable}

\subsection{Symmetric Zero-Sum Games and Symmetric Learning}

A zero-sum game with payoff matrix $A \in \R^{n\times n}$
is called \textit{symmetric} if
$A = -A^\top$ (i.e., $A$ is a skew-symmetric matrix). 
In this paper, we consider such
symmetric zero-sum games under a \textit{symmetric initialization}
(i.e., when $x^0_1 = x^0_2$),
and when the players use the \textit{same algorithm}
to update their strategies. 
It follows that $x^t_1 = x^t_2$ for all $t = 0, \dots, T$,
and so we call this setting \textit{symmetric learning}.
Letting $x^t$ denote the common strategy
at time $t$, the total regret $\reg(T)$ can be written~as
\begin{equation}
  \reg(T)
  \;:=\; 2 \cdot \max_{x \in \Delta_n}\;
  \sum_{t=0}^T
  \langle x,  Ax^t \rangle \;.
  \label{eq:sym-regret}
\end{equation}
Letting $\tilde x^T = (\sum_{t=0}^T x^t)/T$,
it follows by Proposition~\ref{prop:regret-nash}
that the strategy profile 
$(\tilde x^T, \tilde x^T)$ converges to 
an NE of $A$ at a rate of $\reg(T)/T$. 
In such cases where an NE $(x^*_1, x^*_2) = (x^*, x^*)$
is symmetric, we simply say that $x^*$ is 
an NE of $A$. 
Nash equilibria for symmetric zero-sum games
have the following property (see Appendix~\ref{app:prelims:interiornash}
for a proof):

\begin{restatable}{prop}{symnash}
    \label{prop:sym-nash}
    Let $A = -A^\top$ be a symmetric zero-sum game,
    and let $x^*$ be an NE for $A$. 
    Then $Ax^* = 0$ (where $0 \in \R^n$
    is the all-zeros vector). 
\end{restatable}
Symmetric learning has its origins in both 
classical and evolutionary game theory \citep{v1928theorie,brown1950solutions,gale1950symmetric, weibull1997evolutionary}
and has received increasing interest due
to its connections with self-play in modern multi-agent learning settings~\citep{lanctot2017unified,balduzzi2019open}.
Recent work has also shown that  \textit{swap}-regret minimization in
symmetric learning settings results in stable,
Nash-convergent dynamics~\citep{leme2024convergence}.

\subsection{High-Dimensional RPS Matrices}

Our primary focus in this work is on a family 
of symmetric, $n$-dimensional zero-sum games 
that generalize the classic three-strategy
Rock-Paper-Scissors (RPS). 
These games are specified by skew-symmetric
payoff matrices with the following structure:

\begin{restatable}{definition}{rpsdef}
  \label{def:rps}
  For any $n \ge 3$, we say
  $A \in \R^{n\times n}$
  is an $n$-dimensional RPS matrix
  with positive constants $a_1, \dots, a_n > 0$ if
  its entries $A_{i, j}$ are given by
   \begin{equation*}
    A_{i,j}
    \;:=\;
    \begin{cases}
      -a_i &\text{if $j = i+1\; (\mod n)$}\\
      a_{i-1} &\text{if $j = i-1\; (\mod n)$} \\
      0 &\text{otherwise}
    \end{cases}
    \qquad
    \text{for all $i, j \in [n]$} \;\;.
  \end{equation*}
\end{restatable}
\noindent Concretely, an $n$-dimensional
RPS matrix has the following structure:
\begin{equation*}
  \setlength\arraycolsep{1pt}
  \text{For $n=3$}:\;\;
  A=
  \begin{pmatrix}
    0 & -a_1 & a_3 \\
    a_1 & 0 & -a_2 \\
    -a_3 & a_2 & 0  
  \end{pmatrix}.
  \quad
  \text{For $n > 3$:}\;\;
  A =
  \begin{pmatrix}
    0 & -a_1 & 0 & \dots & a_n \\
    a_1 & 0 & -a_2 & 0 & \dots  \\
    0 & a_2 & \ddots & \ddots & \ddots \\ 
    \vdots & \ddots & \ddots & \ddots & -a_{n-1} \\
    -a_n & 0 & \dots & a_{n-1} & 0 
  \end{pmatrix}.
\end{equation*}
When all $a_1 = \dots = a_n = 1$,
we say that $A$ is the \textit{unweighted}
$n$-dimensional RPS matrix. 
We write $a_{\max}$ and $a_{\min}$ to denote 
the maximum and minimum entries among $\{a_i\}$,
which we assume are absolute constants,
and we write $A_i \in \R^n$ to denote the $i$-th
column vector of $A$. 
By slight abuse of notation, the use of $(\mod n)$
assumes the indices stay within the set $\{1, \dots, n\}$
in the natural way.
For readability, we will usually 
omit the $(\mod n)$ operator when
referencing the indices of $\{a_i\}$:
it will be clear from context that
if $i=n$, then $i+1 = 1$, and
if $i=1$, then $i-1 = n$.

In Appendix~\ref{app:prelims},
we also prove that every $n$-dimensional RPS matrix $A$
has a (not-necessarily unique) \textit{interior} NE $x^*$.
\begin{restatable}{prop}{rpsinteriorne}
    \label{prop:A-interior}
    Let $A$ be an $n$-dimensional RPS matrix from Definition \ref{def:rps}
    with positive constants $a_1, \dots, a_n > 0$. 
    Then $A$ has an interior Nash equilibrium $x^*$.
\end{restatable}
RPS variants have been studied extensively in classical
game theory~\citep{von1944theory} and in 
evolutionary game theory~\citep{hofbauer1998evolutionary}. 
Other lines of work focusing on this fundamental class are
discussed in Appendix~\ref{app:related}.


\section{A Unifying View of FP and GD in Symmetric Games}
\label{sec:fp-ftrl-overview}

\subsection{Leader-Based Algorithms for Symmetric Learning}

Leader-based algorithms are the most ubiquitious methods
for online learning in games,
and Fictitious Play and Gradient Descent
can both be viewed in this perspective. 
We introduce these algorithms
in the context of symmetric learning 
with skew-symmetric payoff matrices $A = -A^\top$.

\paragraph{Fictitious Play.}
For symmetric learning in symmetric games,
Fictitious Play (FP) is initialized at a
strategy $x^0 \in \Delta_n$.
At each step $t+1$, the algorithm
selects the (pure) strategy $x^{t+1}$ given by
\begin{equation}
  x^{t+1}
  \;:=\;
  \argmax_{x \in \{e_i:\;i \in [n]\}}\;
  \Big\langle x, \sum_{k=0}^t Ax^k \Big\rangle \;.
  \label{eq:fp}
  \tag{FP}
\end{equation}
Here, $\{e_i:\;i \in [n]\}$ is the set 
of standard basis vectors in $\R^n$,
which corresponds to the vertices of $\Delta_n$.
For convenience, we use the shorthand
$\{e_i\}$ to denote this set. 
Moreover, we make the following remark
on the behavior of the $\argmax$ function
in the update rule:

\begin{remark}[Tiebreaking Rules]\label{remark:fp-tiebreak}
  At time $t$, the set
  $M^t = \{i \in [n] :  \langle e_i, y^{t}\rangle = \max_{j \in [n]} \langle e_j, y^{t} \rangle \}$
  may contain multiple vertices.
  For this, we assume the function
  $\argmax_{x \in\{e_i\}} \langle x, y^{t} \rangle$
  encodes a \textit{tiebreaking rule}
  that always returns a single element from $M^t$.
  Unless otherwise specified,
  we make no assumptions on the tiebreaking rule 
  (e.g., it may be adaptive/adversarial with respect
  to the history of previous iterates).
  This is in contrast to~\cite{daskalakis2014counter}
  (who assumed a specific adversarial tiebreaking rule)
  and \cite{abernethy2021fast} (who assumed
  a fixed lexicographical tiebreaking rule)
  for diagonal payoff matrices.
\end{remark}

\paragraph{FTRL and Gradient Descent.}
Gradient Descent (GD) is an instantiation of the more general 
Follow-the-Regularized-Leader (FTRL) algorithm.
Using a strictly convex regularizer
$\phi : \Delta_n \to \R$ and a fixed stepsize $\eta > 0$,
the iterates of FTRL update as:
\begin{equation}
  x^{t+1}
  \;:=\;
  \argmax_{x \in \Delta_n}\;
  \Big\langle x, \sum_{k=0}^t Ax^k \Big\rangle
    - \frac{\phi(x)}{\eta}\;.
  \label{eq:ftrl}
  \tag{FTRL}
\end{equation}
Due to the strict convexity of $\phi$,
the maximization problem~\eqref{eq:ftrl} is
strictly concave and has a unique solution
(and thus no tiebreaking is needed). 
In this work, we focus on
FTRL instantiated with the $\ell^2$ regularizer
$\phi(x) = \frac{1}{2} \|x\|^2_2$, which
is $1$-strongly convex.
This results in the online Gradient Descent (GD) algorithm.

\paragraph{Interpolating Between FTRL and FP.}
Our particular regime of interest is
when $\eta = \Theta(1)$ is a fixed absolute constant
with respect to $T$.
Note that using the time-horizon-dependent
setting of $\eta \approx 1/\sqrt{T}$
(which is used to obtain $O(\sqrt{T})$ regret bounds
in general online learning setings) implies
that the maximization problem at each time $t+1$
places  a \textit{weight} on the regularizer of
order $\sqrt{T}$. 
In contrast, when $\eta = \Theta(1)$, this weight is 
an absolute constant for all $t$, 
so as $\eta \to \infty$,
the \eqref{eq:ftrl} update approaches
that of~\eqref{eq:fp}.

\subsection{Geometry of the Dual Dynamics}
\label{sec:fp-ftrl-overview:dual}

\paragraph{Primal and Dual Updates.}
The primal iterates $x^t \in \Delta_n$
of both Fictitious Play and FTRL can be
expressed in terms of a sequence
of \textit{dual payoff vectors} $y^t \in \R^n$.
Let $y^0 = 0 \in \R^n$
be the all-zeros vector.
For a fixed $\eta > 0$,
we define for each $t \ge 1$:
\begin{equation}
  y^{t+1}
  \;=\;
  y^t + \eta \cdot Ax^t \;.
  \label{eq:payoff-vector}
  \tag{Dual Vector} 
\end{equation}
Then, the primal iterates of Fictitious Play and FTRL
are given by:
\begin{equation}
  \begin{aligned}
    \text{For FP:\quad$\eta = 1$\;\;and}
    \quad
    x^{t+1}
    &=
      \argmax_{x \in \{e_i\}}\;
      \langle x, y^{t+1} \rangle
    \\
    \text{For FTRL:\quad$\eta > 0$\;\;and}
    \quad
    x^{t+1}
    &=
      \argmax_{x \in \Delta_n}\;
      \langle x, y^{t+1} \rangle - \phi(x) \;.
  \end{aligned}
  \label{eq:new-primal}
\end{equation}
For each algorithm, we call the corresponding
sequence $\{y^t\}$ the set of \textit{dual iterates}. 

\paragraph{Regret and Energy of Dual Iterates.}
Using the definition
of~\eqref{eq:payoff-vector} and the 
primal iterates of either algorithm
from expression~\eqref{eq:new-primal},
the regret definition from ~\eqref{eq:sym-regret}
can be rewritten as:
\begin{equation*}
  \reg(T) \;=\;
  \frac{2}{\eta} \cdot \max_{x \in \Delta_n}\;
  \langle x, y^{T+1} \rangle  \;.
\end{equation*}
For both algorithms, regret is closely related to a corresponding
\textit{energy function} defined over the dual space $\R^n$.
For Fictitious Play, the energy is the \textit{support function}
$\Psi$ of $\Delta_n$ (which is the convex conjugate of the 
indicator function on $\Delta_n$). 
For FTRL, the energy is the \textit{convex conjugate}
$\phi^*$ of the regularizer $\phi$ over the domain $\Delta_n$. 
Specifically, for $y \in \R^n$, we define:
\begin{equation}
  \begin{aligned}
    \text{Energy function for FP:}
    &\quad
      \Psi(y)
      \;=\;
      \max_{x \in \Delta_n}\;
      \langle x, y \rangle
    \\
    \text{Energy function for FTRL:}
    &\quad
      \phi^*(y)
      \;=\;
      \max_{x \in \Delta_n}\;
      \langle x, y \rangle 
      - \phi(x) \;.
  \end{aligned}
  \label{eq:energy-defs}
\end{equation}
For Fictitious Play, the following
relationship between $\Psi$ and regret is then
immediate: 

\begin{restatable}{prop}{propregretenergyfp}
  \label{prop:regret-energy-fp}
  Let $\{y^t\}$ be the dual iterates of FP
  (with $\eta = 1$). 
  Then $\reg(T) = 2 \cdot \Psi(y^{T+1})$. 
\end{restatable}
\noindent For FTRL, we also establish the following
similar relationship (see Appendix~\ref{app:fp-ftrl:energyregret}
for a proof):

\begin{restatable}{prop}{propregretenergyftrl}
  \label{prop:regret-energy-ftrl}
  Let $\{y^t\}$ be the dual iterates of FTRL
  with $\eta > 0$. 
  Let $M = \max_{x \in \Delta_n} \phi(x)$.
  Then:
  \begin{equation*}
    \reg(T)
    \;\le\;
    \frac{2 \cdot \phi^*(y^{T+1})}{\eta} + \frac{2M}{\eta} \;.
  \end{equation*}
\end{restatable}

\paragraph{Energy Conservation in Continuous Time.}
Several recent works have identified
that the continuous time variants (i.e., the limit of vanishing step size $\eta \to 0$) of both FTRL
\citep{mertikopoulos2018cycles,
  bailey2019multi, wibisono2022alternating}
and Fictitious Play
\citep{ostrovski2011piecewise,
  van2011hamiltonian,abernethy2021fast}
have a Hamiltonian structure:
the dual iterates follow a \textit{skew-gradient flow}
that exactly conserves the corresponding energy function
over time. In both cases, continuous-time energy
conservation corresponds to constant regret bounds.

\paragraph{Skew-Gradient Descent in Discrete Time.}
In discrete time, the dual iterates of each
algorithm follow a first-order forward discretization
of the skew-gradient flow with respect
to its corresponding energy function.
By the convexity of the energy function, the energy 
along this \textit{skew-(sub)-gradient descent}
is always non-decreasing. 
Formally, following the presentation 
of~\citet{abernethy2021fast} and~\citet{wibisono2022alternating}, 
let $\{y^t\}$ be  the dual iterates of
either Fictitious Play
(with $\eta = 1$) or FTRL
(with $\eta > 0$)  from~\eqref{eq:new-primal},
and let $H: \R^n \to \R$ denote its
corresponding energy function from~\eqref{eq:energy-defs}.
By a slight abuse of notation, let $\partial H(y^t)$
denote any vector in the subgradient set\footnote{
  For convex $H$, its subgradient set 
  at $y \in \R^n$
  is given by 
  $\partial H(y)
  =
  \{ g \in \R^n : \forall z \in \R^n, 
  H(z) \ge H(y) + \langle g, z - y \rangle
  \}$.
} of $H$ at $y^t$.
It is then straightforward to show the following
(see Appendix~\ref{app:skew-gradient} for a derivation):
\begin{restatable}{prop}{propskewgradupdate}
  \label{prop:skew-gradient-descent}
  Let $\{y^t\}$ be the dual iterates of
  either Fictitious Play ($\eta = 1$) or FTRL ($\eta > 0$),
  and let $H$ be its corresponding energy function
  from~\eqref{eq:energy-defs}.
  Then for every $t \ge 1$, it holds that
  \begin{equation}
    y^{t+1}
    \;=\;
    y^t + \eta A \partial H(y^t) \;.
    \label{eq:skew-gradient}
  \end{equation}
  In particular: for FP, each $x^t \in \partial \Psi(y^t)$,
  and for FTRL, each $x^t = \nabla \phi^*(y^t)$. 
  Moreover, for all $t \ge 1$:
  $\Delta H(y^t) = H(y^{t+1}) - H(y^t) \ge 0$.
\end{restatable}

\subsection{Bounds on Regret via Controlling the Energy Growth}
\label{sec:fp-ftrl-overview:bounds}

In light of the geometric perspective
given in Section~\ref{sec:fp-ftrl-overview:dual}
and of the relationships between
energy and regret from Propositions~\ref{prop:regret-energy-fp}
and~\ref{prop:regret-energy-ftrl},
our approach to obtain regret bounds for both
algorithms is to control the energy 
growth of their respective dual iterates over time. 
While smoothness properties of
the energy function can be used
to derive one-step, worst-case growth bounds
(e.g., as in~\citet{wibisono2022alternating}),
this approach may be overly pessimistic.

In this work, we instead aim for 
a more fine-grained analysis (leveraging
structural properties of the underlying payoff matrices)
that captures the non-uniform energy growth 
implied by the geometric perspective developed above.
Below, we provide intuition as to why such an
approach may be possible:

\paragraph{Intuition for Fictitious Play.}
For example, for Fictitious Play,
we show in Section~\ref{sec:fp} (see Proposition~\ref{prop:fp-energy-growth})
that when $x^t = x^{t+1} = e_i$ for some $i \in [n]$, 
then $\Psi(y^{t+1}) - \Psi(y^t) = 0$
(and only when $x^t \neq x^{t+1}$ can $\Psi$ increase).
Here, the intuition is the following:
by definition of $\Psi$, 
the energy function is linear in the coordinate $y_i$ 
within the region $L_i \subseteq \R^n$, where 
\begin{equation*}
   L_i \;=\; \{y \in \R^n:  y_i > \max\nolimits_{j \neq i} y_j\}. 
\end{equation*}
Moreover, for any time $t$ such that $y^t\in L_i$, 
the definition of~\eqref{eq:dual-dynamics} shows that
$\Delta y^t = y^{t+1} - y^t = A_i$,
and Proposition~\ref{prop:skew-gradient-descent}
implies that $\Delta y^t$ follows a linear skew-(sub)gradient step
with respect to $\Psi$. 
Thus the coordinate $y^{t+1}_i = y^t_i$ is unchanged,
and if $y^{t+1}$ also falls in $L_i$, then
$y^{t+1}$ and $y^t$ must be on the same level set of 
$\Psi$ and energy is conserved
(in other words, the skew-gradient discretization 
behaves the same as the skew-gradient flow). 

On the other hand, if (due to the discretization) 
$y^{t+1}$ lands in a new region $L_j$,
the energy increases, and $\Psi$ becomes linear in $y_j$.
This explains the expanding (but linear)
trajectory of the FP dual iterates in Figure~\ref{fig:FPdual},
and it roughly implies that the total energy growth (and regret)
can be controlled by understanding how frequently 
the dual iterates switch between the $L_i$ regions.

\paragraph{Intuition for FTRL and Gradient Descent.}
For Fictitious Play, energy conservation between consecutive 
dual iterates is guaranteed when the two corresponding
primal iterates lie on the same vertex of $\Delta_n$. 
While the primal iterate of Fictitious Play 
is \textit{always} at a vertex (for $t \ge 1$), 
for FTRL instantiated with a \textit{Legendre} regularizer $\phi$
(for example, the negative entropy function
corresponding to the Multiplicative Weights algorithm),
the primal iterates will always remain on the \textit{interior}
of the simplex (see e.g.,~\cite{wibisono2022alternating}).
Thus in general, the geometry of the dual iterates of FTRL 
will not exactly coincide with those of Fictitious Play.

On the other hand, in this work we focus on Gradient Descent,
the FTRL instantiation with  $\phi(x) = \frac{1}{2}\|x\|^2_2$. 
For Gradient Descent, 
the primal update rule of~\eqref{eq:new-primal}
may require a projection onto $\Delta_n$,
and thus the primal iterates may in general lie
on the boundary of the simplex. 
In particular, when $x^t = e_i$ for some $i \in [n]$, 
the dual iterate $y^t$  \textit{must lie in a region of $\R^n$
where the geometry of $\phistar$ and $\Psi$ 
exactly aligns}
(specifically, $\phistar$ is linear in $y_i$
in this region). Within these regions
(defined formally in Section~\ref{sec:gd})
the dual trajectory of Gradient Descent is
identical to that of Fictitious Play (both aligning with
a linear skew-gradient flow),
and energy is conserved each step. 

However, when the primal iterate $x^t$ is on a non-vertex
boundary (or in the interior of $\Delta_n$),
the energy $\phistar$ is \emph{quadratic} in $y^t$,
and thus the first-order discretization of the skew-gradient flow
will strictly increase the energy. 
This explains the expanding trajectory of the dual
iterates of Gradient Descent in Figure~\ref{fig:GDdual},
and analogously to Fictitious Play, it implies 
that the energy growth can be controlled by
analyzing how frequently the dual iterates
switch between the linear and quadratic regions. 
We give more details on the 
Gradient Descent primal update and
energy function in Section~\ref{sec:gd}
and Appendix~\ref{app:gd:details}. 

\paragraph{Challenges in High Dimension.}
The preceding intuition establishes similarities
between the dual iterates and energy functions
for Fictitious Play and Gradient Descent when 
consecutive sequences of the primal iterates 
lie on the same vertex of $\Delta_n$. 
While this intuition holds for symmetric learning
on any skew-symmetric payoff matrix $A$, 
without tight control over how the  primal (and dual) iterates
evolve over longer sequences of steps, 
controlling the energy growth (and thus regret)
can be challenging, especially in high dimensions.
However, for the class of $n$-dimensional RPS matrices, 
we prove that such long-term control for both 
algorithms is possible, subsequently
leading to strong regret guarantees. 
We present these analyses in the following sections.


\section{Analysis of Fictitious Play on High-Dimensional RPS}
\label{sec:fp}

In this section, we introduce our analysis of Fictitious Play
on high-dimensional RPS matrices, for which
we prove a worst-case $O(\sqrt{T})$ regret bound. 
Using the dual perspective
from Section~\ref{sec:fp-ftrl-overview},
recall that the iterates of Fictitious Play are
given by
\begin{equation}
  \begin{cases}
    y^{t+1}
      &= y^t + Ax^t \\
      x^{t+1}
      &= \argmax_{x \in \{e_i\}}
        \langle x, y^{t+1} \rangle \;.
  \end{cases}
  \label{eq:fp-pd}
  \tag{FP Primal-Dual}
\end{equation}
Throughout this section, we let
$\{x^t\}$ and $\{y^t\}$ denote these
primal and dual iterates, respectively. 

\subsection{Cycling of Primal Iterates Under Arbitrary Tiebreaking}
\label{sec:fp:cycle}
For any $n$-dimensional RPS matrix $A$,
and using any (possibly adversarial) tiebreaking rule,
our analysis begins by proving the following \textit{cycling}
behavior of the Ficitious Play iterates:
if the energy of the dual iterates ever increases from its
initial value, then the subsequent
primal iterates \textit{cycle through a fixed order of
  the vertices of $\Delta_n$ for the remainder
of the dynamics}. 
Specifically, starting from some vertex $e_i \in \Delta_n$, 
the iterates $\{x^t\}$ cycle in the order
\begin{equation*}
e_i \to e_{i+1} \to \dots \to e_{n} \to e_{1} \to \dots \to e_i \to \dots
\end{equation*}
We call a sequence of consecutive iterates at the same vertex
a \textit{phase}, defined formally as follows:
\begin{restatable}{define}{defphases}
  \label{def:phases}
  Fix a time $t_0 > t$. For each $k \ge 1$, let
  $t_k := \min\; \{ t > t_{k-1} : x^{t} \neq x^{t_{k-1}}\}$.
  Then Phase k is the sequence of iterates
  at times $t = t_k, t_{k} + 1, \dots, t_{k+1} - 1$.
  Let $\tau_k = t_{k+1} - t_k$ denote the
  \textit{length} of Phase $k$.
  Let $K \ge 0$ be the total number of phases
  in $T$ rounds, where $T = \sum_{k=0}^K \tau_k$. 
\end{restatable}
Then for every RPS matrix $A$, and using
any tiebreaking method
(cf., Remark~\ref{remark:fp-tiebreak}),
the following behavior occurs:
if $x^t = e_j$ for $j \in [n]$,
then $x^{t'} = e_{j+1\; (\mod n)}$, where
$t' > t$ is the next time the primal iterate changes.
Formally:

\begin{restatable}{lem}{fpcyclinglemma}
  \label{lem:fp-cycling}
  Let $t_0$ be the first $t >0$
  where $\Psi(y^{t_0}) > \Psi(y^1)$,
  and suppose $x^{t_0} = e_i$ for
  some $i \in [n]$.
  Then $x^{t_{k}} = e_{i+k\;(\mod n)}$ for all $k \ge 1$. 
\end{restatable}

The proof of the lemma (see Appendix~\ref{app:fp})
relies on exactly characterizing the \textit{linear}
trajectory of the dual iterates within a phase. 
Specifically, suppose in the current phase that the
primal iterate is at vertex $e_i$. 
Then by definition of the update rule of~\eqref{eq:fp-pd},
$y^t_i \ge \max_{j \neq i} y^t_j$ 
for all iterates $t$ within the phase,
and thus the velocity $\Delta y^t = y^{t+1} - y^t = A_i$
is a fixed constant vector. 

\smallskip

Using the structure of RPS matrices, we can
then track the evolution of the coordinates of $y^t$ 
under this fixed velocity, and we prove
that, if ever $y^{t}_i = y^{t}_j$ for some $j \neq i$
at time $t$ within the current phase, 
then $j = i+1$. 
In other words, tiebreaking scenarios can only occur
between adjacent coordinates of the dual variable.
By leveraging the structure of the velocity
$\Delta y^t = A_i$ for RPS matrices, it is then straightforward
to establish that under any tiebreaking rule,
the primal iterate must eventually switch to 
vertex $e_{i+1}$ in the next phase. 

\subsection{Cycling Implies Worst-Case $O(\sqrt{T})$ Regret}
\label{sec:fp-regret}

The cycling behavior of Lemma~\ref{lem:fp-cycling}
establishes a \textit{regularity} in the
trajectory and energy growth of the dual iterates.
Given the relationship between energy and
regret from Proposition~\ref{prop:regret-energy-fp},
this regularity ultimately allows us to establish the following
worst-case $O(\sqrt{T})$ regret bound:

\begin{restatable}{theorem}{fpregret}
  \label{thm:fp-rps}
  Let $A$ be an $n$-dimensional
  RPS payoff matrix, and 
  let $\{x^t\}$ and $\{y^t\}$ be the iterates of \eqref{eq:fp-pd}
  on $A$ from any $x^0 \in \Delta_n$. 
  Then using any tiebreaking rule,
  $\reg(T) \le O(\sqrt{T})$. 
\end{restatable}

As mentioned in Section~\ref{sec:intro}, Theorem~\ref{thm:fp-rps}
establishes the first class of zero-sum games beyond
the case of diagonal payoff matrices (from~\cite{abernethy2021fast})
for which Fictitious Play has provable $O(\sqrt{T})$ regret,
and by Proposition~\ref{prop:regret-nash}, 
this also guarantees a $O(1/\sqrt{T})$ convergence rate 
(in duality gap)  of the time-averaged iterates to an NE of the game.
The full proof of the theorem is developed
in Appendix~\ref{app:fp}, but we sketch the main ideas below.

\paragraph{Proof Sketch of Theorem~\ref{thm:fp-rps}.}
In addition to the cycling behavior of Lemma~\ref{lem:fp-cycling},
the proof relies on (i) bounds on the energy growth between
phases, and (ii) bounds on the length of each phase. \\

\noindent
\textbf{(i) Bounds on energy growth.}
First, we establish the following two cases
of energy growth:
\begin{restatable}
  {prop}{fpenergygrowth}
  \label{prop:fp-energy-growth}
  For any $t$, define 
  $\Delta \Psi(y^t) := \Psi(y^{t+1}) - \Psi(y^t)$.
  Then:
  \begin{enumerate}[
    label={(\alph*)},
    leftmargin=3em,
    ]
  \item
    If $x^{t} = x^{t+1}$, then $\Delta \Psi(y^t) = 0$.
  \item
    If $x^t \neq x^{t+1}$, then $0 \le \Delta \Psi(y^t) \le a_{\max}$. 
  \end{enumerate}
\end{restatable}
In particular, the proposition implies that
$\Psi$ can only increase
when entering a new phase, 
and the total energy $\Psi(y^{T+1})$ is 
proportional to the number of phases
in which $\Psi$ has strictly increased.  
The proof of the proposition follows
along the lines of the intuition
introduced in Section~\ref{sec:fp-ftrl-overview:bounds}
for the energy growth of
skew-gradient descent. 
\\

\noindent
\textbf{(ii) Bounds on phase length.}
Then, using the cycling behavior of Lemma~\ref{lem:fp-cycling},
we prove that the \textit{length} $\tau_k$
of each Phase $k$ is roughly proportional
to the energy at the start of the phase:

\begin{restatable}{lem}{lemphaselength}
\label{lem:phase-length}
For $k = 1, 2, \dots, K$, let $\gamma_k = \Psi(y^{t_k})$
be the energy at the start of Phase $k$.
Then $\tau_k \ge \alpha_k \cdot \gamma_k - \beta_k$, where
$\alpha_k > 0$ and $\beta_k > 0$ are absolute constants. 
\end{restatable}

The proof of the lemma uses the following intuition:
first, assume for simplicity that $x^{t_k} = e_i$
at the start of Phase $k$. 
Then due to the fixed cycling order, we show that 
the phase length $\tau_k$ is bounded below by 
\begin{equation}
  \tau_k 
  \ge 
  \Omega\big(
    y^{t_k}_i - y^{t_k}_{i+1} 
  \big)
  \;=\;
  \Omega\big(
  \Psi(y^{t_k})
  - 
  y^{t_k}_{i+1}\big) \;.
  \label{eq:pl-intuition}
\end{equation}
the cycling order and structure of
RPS matrices further implies that each coordinate of
the dual vector can increase and decrease in exactly 
one phase during every consecutive sequence of $n$ phases. 
In turn, this allows for controlling the
growth of coordinate $(i+1)$ over time, 
which ensures in expression~\eqref{eq:pl-intuition}
that roughly $y^{t_k}_{i+1} \le O(\Psi(y^{t_k}))$. 

\smallskip 

Together, Proposition~\ref{prop:fp-energy-growth}
and Lemma~\ref{lem:phase-length}
establish a quadratic relationship between
the total number of phases and the total energy of the 
dual iterates, which is a similar property to those 
leveraged by both ~\citet{bailey2019fast} and~\citet{abernethy2021fast}. 
To see how such a relationship leads to $O(\sqrt{T})$ regret,
assume for simplicity that $\Psi$ strictly increases between phases, 
and thus $\Psi(y^{t_k}) - \Psi(y^{t_{k-1}}) = \Theta(1)$ for all $k$.
By Proposition~\ref{prop:fp-energy-growth},
this implies that $\Psi(y^{T+1}) \le O(K)$,
where $K$ is the total number of phases in $T$ rounds.
By Lemma~\ref{lem:phase-length}, this also implies that 
$\tau_k \ge \Omega(\Psi(y^{t_k})) \ge \Omega(k)$.
Together, we find:
\begin{equation}
  T = \sum_{k=0}^K \tau_k
  \ge \sum_{k=0}^k \Omega(k) \ge \Omega(K^2)
  \;\implies\;
  K = O(\sqrt{T})
  \;\implies
  \Psi(y^{T+1}) \le O(\sqrt{T}) \;.
  \label{eq:regret-recurrence}
\end{equation}
By Proposition~\ref{prop:regret-energy-fp},
this proves the claimed regret bound. 
Note that $\Psi$ might not be strictly increasing
between each phase, and in the full proof we account for
this behavior. 


\section{Analysis of Gradient Descent on High-Dimensional RPS}
\label{sec:gd}

In this section, we turn toward analyzing Gradient Descent
on high-dimensional RPS matrices. 
Using the dual perspective
from Section~\ref{sec:fp-ftrl-overview},
recall that the iterates of Gradient Descent
with stepsize $\eta > 0$ are given by:
\begin{equation}
  \begin{cases}
    y^{t+1}
      &= y^t + \eta Ax^t \\
      x^{t+1}
      &= \argmax_{x \in \Delta_n} 
        \langle x, y^{t+1} \rangle - \frac{\|x\|^2_2}{2} \;.
  \end{cases}
  \label{eq:gd-pd}
  \tag{GD Primal-Dual}
\end{equation}
Throughout this section, we let
$\{x^t\}$ and $\{y^t\}$ denote these
primal and dual iterates, respectively.

\paragraph{Closed-Form Expressions and Primal-Dual Map.}
As mentioned in Section~\ref{sec:fp-ftrl-overview},
the primal update rule of~\eqref{eq:gd-pd}
may require a projection onto the boundary of $\Delta_n$.
For this, we state in Proposition~\ref{prop:gd-closed}
(derived originally in \cite{bailey2019fast})
a closed-form characterization of the 
primal iterates $x^t$, which 
also leads to a closed-form characterization
of the energy function $\phistar(y^t)$.
To streamline the presentation, we defer
these details to Appendix~\ref{app:gd:details}. 
However, central to our analysis of Gradient Descent
is to identify regions of the dual space $\R^n$ that,
under the update, map to the vertices and (a subset of)
edges of $\Delta_n$. Specifically, we define 
the regions 
$P_i$ and $P_{i\sim (i+1)}$ as follows:
\begin{restatable}{define}{defgdpregions}
  \label{def:gd-p-regions}
  For each $i \in [n]$, let $P_i \subset \R^n$
  and $P_{i\sim(i+1)} \subset \R^n$ be the following sets:
  \begin{align*}
    P_i
    &\;:=\;
    \Big\{
      y \in \R^n :\quad
      y_i - y_j > 1
      \quad\text{for $j \in [n] \setminus \{i\}$}
      \Big\}  \\
    P_{i\sim(i+1)}
    &\;:=\;
      \left\{
      y \in \R^n :\quad
      \begin{aligned}
      |&y_i - y_{i+1}| \le 1 \;\text{and}\\
        \tfrac{1}{2}(&y_i +  y_{i+1}) - y_j > \tfrac{1}{2}
        \quad\text{for $j \in [n] \setminus \{i, i+1\}$}
      \end{aligned}
    \right\} \;.
  \end{align*}
\end{restatable}
\noindent Then, in Appendix~\ref{app:gd:details:pd-map}, we prove
the following relationship: 
\begin{restatable}{prop}{gdpdmap}
  \label{prop:gd-pd-map}
  For any $i\in [n]$: 
  $y^t \in P_i$ if and only if $x^t=e_i$,
  and $y^t \in P_{i\sim(i+1)}$ if and only if
  $\supp(x^t) = \{i, i+1\}$.
\end{restatable}

\noindent
In other words, if $y^t \in P_i$, then the primal
iterate $x^t$ must be at the vertex $e_i$,
and if $y^t \in P_{i \sim (i+1)}$, then the primal iterate
is on the edge of $\Delta_n$ between the vertices $e_i$ 
and $e_{i+1}$ (and vice versa).

\subsection{$O(\sqrt{T})$ Regret with Large Stepsizes}  
\label{sec:gd:vertex-convergence}

The behavior of Fictitious Play on RPS matrices
established in Section~\ref{sec:fp}, 
together with the shared geometric characterization 
of Fictitious Play and Gradient Descent introduced 
in Section~\ref{sec:fp-ftrl-overview}, 
suggests the following intuition: 
if the primal iterates of Gradient Descent 
eventually reach a \textit{vertex} of $\Delta_n$, 
then we could expect all subsequent primal iterates to demonstrate
a similar cycling behavior as Fictitious Play,
and to show a similar regularity in energy growth and regret.

In this section, we prove in Theorem~\ref{thm:gd-large-stepsize}
that this intuition is indeed correct in the regime of 
\textit{large} constant stepsizes:
for almost all initializations $x^0 \in \Delta_n$,
we prove that when the stepsize $\eta > 0$ is a
\textit{sufficiently large} constant, then Gradient Descent obtains
$O(\sqrt{T})$ regret on every $n$-dimensional RPS matrix.
The proof of this result relies on establishing
that the primal iterates
(i) converge to a vertex of $\Delta_n$, and
(ii) exhibit a cycling behavior that leads to 
patterns of energy growth and phase lengths similar to 
those of Fictitious Play. 
We introduce these components below: 

\paragraph{Fast Convergence to a Vertex.}
First, we prove that using a large enough stepsize
ensures that the primal iterate
reaches a vertex after only a single iteration. 
For this, fixing an RPS matrix $A$, 
for any $x \in \Delta_n$, let $\gamma(x)$ be the constant 
\begin{equation}
  \gamma(x)
  \;:=\;
  \min_{k \neq \ell \in [n]}
  \Big|
  \big(a_{k-1} \cdot x_{k-1} - a_k \cdot x_{k+1}\big)
  - 
  \big(a_{\ell-1} \cdot x_{\ell-1} - a_\ell \cdot x_{\ell+1}\big)
  \Big| \;.
  \label{eq:gamma-main}
\end{equation}
Then the following holds:

\begin{restatable}{lem}{gdlargeinitial}
  \label{lem:gd-large-initial}
  If $x^0 \in \Delta_n$ is such that $\gamma(x^0) > 0$,
  then along one step of~\eqref{eq:gd-pd} with
  stepsize $\eta > 1/\gamma(x^0)$, the iterate $x^1$
  is a vertex of $\Delta_n$.
\end{restatable}

To prove the lemma (see Appendix~\ref{app:gd:large:convergence}), 
we show that when $\eta > 1/\gamma(x^0)$,
the dual iterate $y^1 = \eta Ax^0$ must fall in some region $P_i$
(from Definition~\ref{def:gd-p-regions}),
and thus by Proposition~\ref{prop:gd-pd-map}, $x^1 = e_i$.
By definition of $\gamma$, the set of points $x \in \Delta_n$
where $\gamma(x) = 0$ are the solutions to the linear
constraint in~\eqref{eq:gamma-main} and has Lebesgue measure zero
(note also by the definition of RPS matrices and 
Proposition~\ref{prop:sym-nash} 
that every Nash equilibrium $x^*$ of $A$
has $\gamma(x^*) = 0$). 

\paragraph{Cycling, Energy Growth, and Phase Lengths.}
In this large stepsize regime, 
we further establish that the primal iterates 
eventually cycle between vertices of $\Delta_n$
in the same order as in Lemma~\ref{lem:fp-cycling}
for Fictitious Play.
However, between the vertices $e_i$ and $e_{i+1}$, 
the iterates may spend a constant number of steps
on the edge of $\Delta_n$ where $\supp(x) = \{i, i+1\}$
(which by Proposition~\ref{prop:gd-pd-map} means that 
the corresponding dual iterates lie in the region $P_{i\sim(i+1)}$).
This explains the behavior of the primal iterates on $n=4$ unweighted RPS 
in Figure~\ref{fig:GDPrimal}
and is formally captured in Lemma~\ref{lem:gd-cycling-pl-overview},
which we state and prove in Appendix~\ref{app:gd:large}.

Lemma~\ref{lem:gd-cycling-pl-overview}
additionally gives bounds on the energy growth between phases 
(at most an absolute constant) and the length of each phase
(growing proportionally to energy) similar to the analysis of
Fictitious Play from Section~\ref{sec:fp-regret}.
In turn, this leads to a similar worst-case $O(\sqrt{T})$ 
regret bound  for Gradient Descent on high-dimensional RPS matrices:

\begin{restatable}{theorem}{gdrpslargestepsize}
  \label{thm:gd-large-stepsize}
  Let $A$ be an $n$-dimensional RPS matrix.
  Then for nearly all initial distributions $x^0 \in \Delta_n$,
  the following holds:
  letting $\{x^t\}$ be the iterates of running 
  \eqref{eq:gd-pd} on $A$ with
  $\eta > \min\{2/a_{\min}, 1/\gamma(x^0)\}$, 
  then $\reg(T) \le O(\sqrt{T})$.
\end{restatable}

The full proof of Theorem~\ref{thm:gd-large-stepsize} is
developed in Appendix~\ref{app:gd:large}
and follows similarly to that of 
Theorem~\ref{thm:fp-rps} for FP.
We remark that while the constraint on $\eta$ may depend on 
the initialization $x^0$, the theorem 
and its proof yield the following moral conclusion:
with large enough constant stepsizes, the dual trajectory 
of Gradient Descent becomes increasingly similar to 
that of Fictitious Play,
and GD thus inherits the same cycling and 
regularity in energy growth that are sufficient 
for establishing $O(\sqrt{T})$ regret. 

\subsection{Behavior of Gradient Descent with Smaller Stepsizes}

In light of the regret guarantees for Gradient Descent
on RPS matrices in the \textit{large} stepsize regime,
a natural question to ask is how the algorithm
(and its regret) behaves using smaller stepsizes. 
Fully answering this question is more challenging,
and our work leaves open a complete characterization
of the behavior of GD on RPS matrices.
However, in Appendix~\ref{app:gd:small}, we
prove several auxiliary results in this direction.
To summarize:

\begin{itemize}[
  leftmargin=2em,
  ]
\item
  In Theorem~\ref{thm:gda-rps-midstepsize},
  we prove a \textit{boundary invariance} property
  that holds for any $\eta > 0$:
  when the energy ever exceeds a game-dependent
  constant value, then every subsequent primal
  iterate of Gradient Descent must lie on
  the boundary of $\Delta_n$.
  In particular, this extends the boundary invariance
  result of~\cite{bailey2019fast} for $2\times2$ games
  to the present $n$-dimensional symmetric setting.
  However, beyond boundary invariance, obtaining
  $O(\sqrt{T})$ regret bounds for Gradient Descent
  using \textit{any} constant stepsize remains open. 
\item
  At the other extreme, we prove in
  Lemma~\ref{lem:gd-smallsteps} that
  with a small time-horizon-dependent stepsize
  (e.g., $\eta = \Theta(1/\sqrt{T})$), even if
  all primal iterates remain \textit{interior},
  we can still only guarantee a worst-case regret
  bound for Gradient Descent scaling like $O(\sqrt{T})$.
  Interestingly, in Appendix~\ref{app:simulations},
  we present simulations showing empirically that
  Gradient Descent with large constant stepsizes
  obtains tighter regret than its time-dependent
  $\eta = \Theta(1/\sqrt{T})$ counterpart.
\end{itemize}
We give the precise statements of these results
and more discussion in Appendix~\ref{app:gd:small}.

\section{Discussion and Future Work}
This paper establishes new $O(\sqrt{T})$ regret guarantees 
(and thus $O(1/\sqrt{T})$ time-averaged convergence to Nash equilibria)
for two \textit{non}-no-regret algorithms 
for symmetric learning in zero-sum games. 
For Fictitious Play, our regret bound for high-dimensional
RPS establishes a new class of matrices for
which Karlin's Conjecture is true, and importantly, 
this result holds
using \textit{any} tiebreaking rule.
Interestingly, we show in Theorem~\ref{thm:tournament-regret}
of Appendix~\ref{app:fp:tournament} that 
on unweighted RPS, and using a specific 
\textit{fixed} tiebreaking rule, FP
obtains only \textit{constant} regret.
Better understanding the interplay between tiebreaking and regret
for other classes of matrices is left as open.

For Gradient Descent, our result establishes the first 
sublinear regret guarantees using constant stepsizes
in high-dimensional zero-sum games, beyond the $2\times2$ case. 
Our analysis leveraged the shared
geometry of the two algorithms that emerges under 
\textit{large constant stepsizes}, 
and we believe this insight may be useful for deriving
regret bounds for both algorithms
in other classes of symmetric zero-sum games. 
Finally, obtaining regret guarantees for other instantiations
of FTRL (beyond Gradient Descent) with constant stepsizes 
remains an important open challenge.

\paragraph{Acknowledgements}
This research was supported by the 
MOE Tier 2 Grant (MOE-T2EP20223-0018), 
the National Research Foundation, Singapore, 
under its QEP2.0 programme (NRF2021-QEP2-02-P05),
the National Research Foundation Singapore and
DSO National Laboratories under the 
AI Singapore Programme (Award Number: AISG2-RP-2020-016).
AW was supported by NSF Award CCF \#2403391.
The authors thank Yang Cai, 
Anas Barakat, and Joseph Sakos for helpful discussions.

\bibliography{references}

\newpage
\setcounter{tocdepth}{2}
\renewcommand*\contentsname{Table of Contents}
\tableofcontents
\section{Additional Related Work}
\label{app:related}

\paragraph{Rock-Paper-Scissors Variants.}
Rock-Paper-Scissor is the canonical example of a symmetric zero-sum game
\citep{von1944theory,weibull1997evolutionary,sandholm2010population},
and variants of RPS matrices have been studied extensively in the context of 
evolutionary game theory \citep{smith1974theory,may1975nonlinear,
szolnoki2014cyclic,mai2018cycles} and population dynamics and 
economics~\citep{semmann2003volunteering,xu2013cycle,cason2014cycles}. 
Recent works have also framed RPS as the prototypical example of a `cyclically 
dominant' zero-sum game defined on tournament 
graphs~\citep{paik2023completely,visomirski2024integrability,griffin2024spatial}.

\paragraph{Fictitious Play.}
Fictitious Play was originally introduced by~\cite{brown1949some,brown1951iterative}. 
Though it may fail to converge in general-sum games,~\citep{shapley1963some,monderer1996a2}, FP has been shown to converge to equilibria in zero-sum games~\citep{robinson1951iterative,harris1998rate}, identical-interest/potential games~\citep{monderer1996fictitious} and even in restricted classes 
of general-sum  games~\citep{berger2005fictitious,
miyasawa1963convergence,sela1999fictitious}. 
Beyond the works studying the convergence \emph{rates} of FP in zero-sum games described in Section~\ref{sec:intro}, the convergence rate of FP has also been studied in potential games~\citep{panageas2024exponential,swenson2017exponential} and general (non zero-sum) games~\citep{brandt2010rate}. Moreover, the simplicity of the FP algorithm has led to numerous applications in multi-agent learning~\citep{heinrich2015fictitious,sayin2022zssg,sayin2022fictitious,baudin2022fictitious,perrin2020fictitious,hofbauer2002global}. 

\paragraph{Fast Regret Minimization in Games.} 
As described in Section \ref{sec:intro}, the black-box 
$O(\sqrt{T})$ regret bound obtained by FTRL with a 
time-dependent $\theta(1/\sqrt{T})$ stepsize holds 
for general (and possibly adversarial) online learning settings.
Many recent works have focused on improving over this 
worst-case regret bound when using FTRL in games settings, 
which in general implies faster convergence 
to various classes of equilibria.
The most widely studied modification to the vanilla FTRL setup is that of `optimism'~\citep{syrgkanis2015fast,rakhlin2013optimization}, where learners update their strategies taking into account their previously encountered payoffs in addition to the payoffs observed in the current timestep. This modification has led to algorithms that achieve faster time-average convergence~\citep{daskalakis2021near,chen2020hedging,anagnostides2022near} and even last-iterate convergence~\citep{daskalakis2018last,cai2024fast} in various game settings. Indeed, optimism can be applied to other algorithms in the literature~\citep{daskalakis2011near,hsieh2021adaptive}, often resulting in improved regret bounds. 

Optimism is far from the only approach in the literature that can obtain sharper regret bounds in games. In line with the constant stepsize regime studied in this paper, several works have been able to obtain good regret bound with \emph{absolute constant} stepsizes by modifying the standard FTRL algorithm. In particular, \cite{bailey2020finite,wibisono2022alternating} study \emph{alternating} variants of FTRL, while~\cite{piliouras2022beyond} introduced a `clairvoyant' version of MWU, both approaches resulting in good regret guarantees in their respective settings.


\section{Symmetric Zero-Sum Games}
\label{app:prelims}

\subsection{Convergence of No-Regret Learning 
to Nash in Two-Player Zero-Sum Games}
\label{app:prelims:nashregret}

Recall the \textit{duality gap} 
$\mathrm{DG} \colon \Delta_m \times \Delta_n \to \R$ 
for the game $A$ is defined by, 
for all $(x_1, x_2) \in \Delta_m \times \Delta_n$:
\begin{equation*}
\mathrm{DG}(x_1, x_2) 
\;:=\;
\max_{x_1' \in \Delta_m} 
\langle x_1', A x_2 \rangle
- 
\min_{x_2' \in \Delta_n}
\langle x_1, A x_2' \rangle \;.
\end{equation*}
Note by construction, $\mathrm{DG}(x_1,x_2) \ge 0$, 
and $\mathrm{DG}(x_1,x_2) = 0$ if and only if
$(x_1,x_2)$ is an NE for the game $A$.
Therefore, we can use the duality gap as a 
measure of convergence to NE.

\regretnash*
\begin{proof}
  By definition of $\reg(T) = \reg_1(T) + \reg_2(T)$ 
  from Section~\ref{sec:prelims:online}, we have
  \begin{equation*}
    \reg(T) 
    \;=\;
    \max_{x_1 \in \Delta_m} 
    \sum_{t=0}^T
    \langle x_1, A x^t_2 \rangle
    - 
    \min_{x_2 \in \Delta_n}
    \sum_{t=0}^T
    \langle x^t_1, A x_2 \rangle \;.
  \end{equation*}
  Then recalling that 
  $\tilde x_1^T := (\sum_{t=0}^T x^t_1)/T$ and 
  $\tilde x_2^T := (\sum_{t=0}^T x^t_2)/T$,
  observe that 
  \begin{equation*}
    \mathrm{DG}(\tilde x^T_1, \tilde x^T_2) 
    \;=\;
    \max_{x_1 \in \Delta_m} 
    \langle x_1, A \tilde x^T_2 \rangle
    - 
    \min_{x_2 \in \Delta_m} 
    \langle \tilde x^T_1, A x_2 \rangle
    \;=\;
    \frac{\reg(T)}{T}\;.
  \end{equation*}
  Since $\reg(T) = o(T)$ by assumption, the average iterate 
  $(\tilde x^T_1, \tilde x^T_2)$ converges to an NE of $A$ 
  in duality gap at a rate of $\reg(T)/T = o(1)$.
\end{proof}

\subsection{Property of Nash Equilibria in Symmetric Zero-Sum Games}
\label{app:prelims:interiornash}

\symnash*

\begin{proof}
    Let $x^*$ be a Nash equilibrium for $A$. By definition of Nash, $Ax^* = c\cdot\mathbf{1}$ is a constant vector for some $c\in\R$.
    Since $A$ is skew-symmetric, this implies that
    $0=\langle x^*, A x^*\rangle = \langle c x^*, \mathbf{1} \rangle= c$, 
    so $Ax^*=0$.
\end{proof}

\subsection{Existence of Interior Equilibrium for High-Dimensional RPS}
\label{app:prelims:rps-interior}

\rpsinteriorne*

\begin{proof}
    Let $A$ be an $n$-dimensional RPS matrix. 
    At a Nash equilibrium, we have from 
    Proposition \ref{prop:sym-nash} that $(Ax)_i = 0$ for all $i \in [n]$. 
    This results in a system of $n$ linear equations of the form:
    \begin{equation*}
      a_{i-1} x_{i-1(\mathrm{mod}\ n)} - a_{i} x_{i+1(\mathrm{mod}\ n)} = 0
    \end{equation*}
    Moreover, $x$ lies in the simplex, so $\sum_{i \in [n]} x_i = 1$. 
    Thus, we have a system of $n+1$ linear equations and $n$ variables. 
    In such a system, there must always be a solution where all $x_i > 0$. 
    Indeed, due to the fact that all $a_{i-1}, a_{i} > 0$ and $x_i\geq 0$ by definition, if we set some $x_i = 0$ then it follows that $x_i = 0$ for all $i \in [n]$, which violates the simplex constraint. 
    Since a solution exists by consistency of the linear system, 
    the solution $x$ must be such that $x_i > 0$, implying it is interior.    
\end{proof}

Note that when $n$ is odd, the interior Nash equilibrium is unique,
and when $n$ is even, there exists a continuum of interior Nash Equilibria (see \cite[Chapter 9]{sandholm2010population}).

\section{Dual Dynamics of Fictitious Play and Gradient Descent}
\label{app:fp-ftrl-overview}

\subsection{Properties of Conjugate Functions}

\begin{proposition}
  \label{prop:phi-Q}
  Let $\phi : \Delta_n \to \R$ be a strictly convex regularizer.
  Let $\phi^*: \R^n \to \R$
  and $Q: \R^n \to \Delta_n$ be the functions given by
  \begin{align*}
    \phi^*(y)
    &\;=\;
      \max_{x \in \Delta_n}\;
      \langle x, y\rangle - \phi(x) \\
    Q(y)
    &\;=\;
      \argmax_{x \in \Delta_n}\;
      \langle x, y\rangle - \phi(x) \;.
  \end{align*}
  Then the following properties hold:
  \begin{enumerate}[label={(\roman*)}]
  \item
    $\phi^*$ is convex and continuously differentiable,
    and $\nabla \phi^*(y) = Q(y)$ for all $y \in \R^n$. 
  \item
    The map $Q$ is surjective: for every $x \in \Delta_n$,
    there exists some $y \in \R^n$ such that $Q(y) = x$. 
  \end{enumerate}
\end{proposition}

\begin{proof}
  Claim (i) follows from standard arguments of conjugate
  functions and using the
  strict convexity of $\phi^*$ over $\Delta_n$
  (see e.g., \citet{boyd2004convex}, \citet{shalev2012online},
  and \citet{vandenberhge2022}).

  For Claim (ii), fix $x \in \Delta_n$, and let
  $y = \nabla \phi(x) \in \R^n$.
  We will show that $Q(y) = x$,
  which by the maximizing principle is equivalent to showing that
  $\phi^*(y) = \langle x, y \rangle - \phi(x)$.
  For this, let $f_y: \Delta_n \to \R$ be the function
  $f_y(x') = \langle x', y \rangle - \phi(x')$ for $x' \in \Delta_n$. 
  Observe by the strict convexity of $\phi$
  that $f_y$ is strictly concave. 
  Moreover, since we defined $y = \nabla \phi(x)$,
  then $\nabla f_y(x) = y - \nabla \phi(x) = 0$.
  Thus by concavity, $x \in \Delta_n$ is the unique maximizer of $f_y$,
  and $\max_{x' \in\Delta_n} f_y(x') = \langle x, y \rangle - \phi(x)$. 
  Toegether, we have
  \begin{equation*}
    \phi^*(y)
    \;=\;
    \max_{x' \in \Delta_n}
    \langle x', y \rangle - \phi(x')
    \;=\;
    \max_{x' \in \Delta_n} f_y(x')
    \;=\;
    \langle x, y \rangle - \phi(x) \;,
  \end{equation*}
  which implies that $Q(y) = x$, as desired. 
\end{proof}

Note that Claim (ii) of Proposition~\ref{prop:phi-Q} is
equivalent to saying that
$\nabla \phi^*(\nabla \phi(x)) = x$
for every $x \in \Delta_n$.
However, $\nabla \phi^*$ is in general
not the inverse of $\nabla \phi$,
as $Q(y) = \nabla \phi^* (y)$ is not necessarily injective. 
For example, when $\phi = (\|\cdot\|^2_2)/2$ as in Gradient Descent,
the preimage under $\nabla \phi^*$ of any $x$ on the
boundary of $\Delta_n$ may consist of multiple $y \in \R^n$.
On the other hand, this inverse property does hold
for the case of Legendre regularizers (such as negative entropy)
as studied by~\cite{wibisono2022alternating}. 

\subsection{Properties of Bregman Divergences}

\begin{definition}[Bregman Divergence]
  \label{def:bregman}
  Let $f: \R^n \to \R$ be a differentiable function.
  Then $D_f$ is the Bregman Divergence of $f$,
  where for all $x, x' \in \R^n$:
  \begin{equation*}
    D_f(x', x) \;=\;
    f(x') - f(x) - \langle \nabla f(x), x'-x\rangle \;.
  \end{equation*}
\end{definition}

Note that when $f$ is convex, we have $D_f(x, x') \ge 0$ for all $x, x' \in \R^n$. 
We recall that $D_f$ satisfies the following
\textit{three-point identity}
(see e.g., \cite{wibisono2022alternating}). 

\begin{proposition}
  \label{prop:bregman-3point}
  Let $f: \R^n \to \R$ be a differentiable function.
  Then for any $w, x, y, z \in \R^n$:
  Then for any $x, y, z \in \R^n$:
  \begin{equation*}
    \langle \nabla f(z) - \nabla f(y), x - z \rangle
    \;=\;
    D_f(x, y) - D_f(x, z) - D_f(z, y)\;.
  \end{equation*}
\end{proposition}

\subsection{Energy-Based Regret Bounds for Fictitious Play and FTRL}
\label{app:fp-ftrl:energyregret}
\propregretenergyfp*

\begin{proof}
  Recall from expression~\eqref{eq:sym-regret} that
  $\reg(t) = 2 \cdot \max_{x \in \Delta_n} \sum_{t=0}^T \langle x, Ax^t\rangle$.
  Moreover, we have by definition of~\eqref{eq:fp-pd} that
  $y^{t+1} = \sum_{k=0}^t Ax^k$ for each $t$,
  which means $y^{T+1} = \sum_{k=0}^T Ax^k$. 
  Thus we can write
  \begin{equation*}
    \reg(T)
    \;=\;
    2 \cdot \max_{x \in \Delta_n}
    \langle x, y^{T+1} \rangle
    \;=\;
    2 \cdot \Psi(y^{T+1}) \;,
  \end{equation*}
  where the second equality follows from the
  definition of $\Psi$ from expression~\eqref{eq:energy-defs}.
\end{proof}

\propregretenergyftrl*
  
\begin{proof}
  The proof of the proposition follows similarly
  to~\citet[Theorem C.2]{wibisono2022alternating},
  but specialized to the present symmetric learning setting. 
  For this, recall by definition of~\eqref{eq:ftrl}
  and~\eqref{eq:payoff-vector} that the
  primal and dual iterates evolve as
  \begin{align*}
    y^{t+1}
    &\;=\;
      y^t + \eta Ax^t \\
    x^{t+1}
    &\;=\;
      \argmax_{x \in \Delta_n} \langle x, y^{t+1}\rangle
    - \phi(x) \;,
  \end{align*}
  where $y^0 = 0$ is the zero vector. 
  By claim (i) of Proposition~\ref{prop:phi-Q}, we then have for each $t$
  that $x^t = \nabla \phi^* (y^t)$.
  Now observe by definition of $\reg(T)$ from expression~\eqref{eq:sym-regret}
  and by the skew-symmetry of $A$, we can write
  \begin{equation}
    \reg(T)
    \;=\;
    2 \cdot
    \max_{x' \in \Delta_n}
    \sum_{t=0}^T
    \langle x', Ax^t \rangle 
    \;=\;
    2 \cdot
    \max_{x' \in \Delta_n}
    \sum_{t=0}^T
    \langle x' - x^t, Ax^t \rangle \;.
    \label{eq:reg-2}
  \end{equation}
  Our goal will then be to derive a uniform upper bound on
  $\sum_{t=0}^T
  \langle x' - x^t, Ax^t \rangle$ over all $x' \in \Delta^n$.
  For this, fix some $x \in \Delta_n$, and let $y \in \R^n$ be
  a vector satisfying $\nabla \phi^*(y) = x$
  and $\phi^*(y) = \langle x, y\rangle - \phi(x)$,
  which we know must exist from claim (ii) of
  Proposition~\ref{prop:phi-Q}.
  Then at each time $t$, we use the three-point identity
  of Bregman divergences (Proposition~\ref{prop:bregman-3point})
  and the fact that $y^{t+1} - y^t = \eta Ax^t$ to write
  \begin{align*}
    \langle x - x^t, Ax^t \rangle
    &\;=\;
      \frac{1}{\eta}
      \big\langle
      \nabla \phi^*(y) - \nabla \phi^*(y^t), y^{t+1} - y^t
      \big\rangle \\
    &\;=\;
      \frac{1}{\eta}
      \big(
      D_{\phi^*}(y^{t+1},  y^t)
      + D_{\phi^*}(y^{t}, y)
      - D_{\phi^*}(y^{t+1}, y)
      \big) \;.
  \end{align*}
  Then summing over all $t$ and telescoping, we find
  \begin{align}
    \sum_{t=0}^T
    \langle
    x - x^t, Ax^t
    \rangle
    &\;=\;
      \frac{1}{\eta}
      \sum_{t=0}^T
      D_{\phi^*} (y^{t+1}, y^{t})
      +
      \frac{1}{\eta}
      \sum_{t=0}^T
      \big(
      D_{\phi^*}(y^{t}, y)
      - D_{\phi^*}(y^{t+1}, y)
      \big) \nonumber  \\
    &\;=\;
      \frac{1}{\eta}
      \sum_{t=0}^T
      D_{\phi^*} (y^{t+1}, y^{t})
      +
      \frac{1}{\eta}
      \big(
      D_{\phi^*} (y^0, y)
      -
      D_{\phi^*} (y^{T+1}, y)
      \big) \nonumber \\
    &\;\le\;
      \frac{1}{\eta}
      \sum_{t=0}^T
      D_{\phi^*} (y^{t+1}, y^{t})
      +
      \frac{1}{\eta}
      D_{\phi^*} (y^0, y) \;,
      \label{eq:re1}
  \end{align}
  where the final inequality comes
  from the non-negative of
  $D_{\phi^*}(y^{T+1}, y)$ given that $\phi^*$ is convex. 
  Observe further that we can write
  \begin{align}
    D_{\phi^*}(y^0, y) 
    &\;=\;
      \phi^*(y^0) - \phi^*(y)
      - \langle \nabla \phi^*(y), y^0 - y \rangle
      \nonumber \\
    &\;=\;
      \phi^*(y^0) - (\langle x, y\rangle - \phi(x))
      -
      \langle x, y^0 - y \rangle
      \nonumber \\
    &\;=\;
      \phi^*(y^0) + \phi(x)
      - \langle x, y \rangle
      + \langle x, y \rangle
      - \langle x, y^0 \rangle
      \nonumber  \\
    &\;=\;
      \phi^*(y^0) + \phi(x) \;.
      \label{eq:re2}
  \end{align}
  Here, we use in the second equality that
  $\phi^*(y) = \langle x,y\rangle - \phi(x)$
  and $\nabla \phi^*(y) = x$
  by definition of $x$ and $y$,
  and in the last equality, we use the fact
  that $y^0 = 0$. 
  
  Now observe also for each $t$ that 
  \begin{align}
    D_{\phi^*} (y^{t+1}, y^{t})
    &\;=\;
      \phi^*(y^{t+1}) - \phi^*(y^t)
      -
      \langle \nabla \phi^*(y^t), y^{t+1} - y^t \rangle
      \nonumber \\
    &\;=\;
      \phi^*(y^{t+1}) - \phi^*(y^t)
      - \eta \langle x^t, Ax^t \rangle
      \nonumber \\
    &\;=\;
      \phi^*(y^{t+1}) - \phi^*(y^t) \;,
      \label{eq:re3}
  \end{align}
  where the final two equalities follow by definition
  of $x^t$ and the dual update rule, and
  by the skew-symmetry of $A$. 
  Then substituting expressions~\eqref{eq:re2}
  and~\eqref{eq:re3} back into~\eqref{eq:re1}
  and summing, we find 
  \begin{align}
    \sum_{t=0}^T \langle x - x^t, Ax^t \rangle
    &\;\le\;
      \frac{1}{\eta}
      \sum_{t=0}^T
      \left(\phi^*(y^{t+1}) - \phi^*(y^t)\right)
      +
      \frac{\phi^*(y^0) + \phi(x)}{\eta} \\
   &\;=\;
      \frac{\phi^*(y^{T+1})- \phi^*(y^0)}{\eta}
      + \frac{\phi^*(y^0) + \phi(x)}{\eta} \\
    &\;=\;
      \frac{\phi^*(y^{T+1})}{\eta}
      + \frac{\phi(x)}{\eta}
  \end{align} 
  Maximizing over all $x \in \Delta^n$,
  substituting the inequality into~\eqref{eq:reg-2},
  and recalling the definition of $M$ from
  the statement of the proposition then
  yields the desired claim.
\end{proof}

\subsection{Dual Dynamics as Skew-Gradient Descent}
\label{app:skew-gradient}

\propskewgradupdate*

\begin{proof}
  Recall for a convex function $H : \R^n \to \R$, 
  that its subgradient set  at $y \in \R^n$
  is given by 
  \begin{equation*}
  \partial H(y)
  \;=\;
  \{ g \in \R^n : \forall z \in \R^n, 
  H(z) \ge H(y) + \langle g, z - y \rangle
  \} \;.
  \end{equation*}
  For Fictitious Play, the subgradient
  set of the energy function 
  $\Psi(y) = \max_{x \in \Delta_n} \langle x, y \rangle$
  is the set of maximizers 
  $\partial \Psi(y) = \argmax_{x \in \Delta_n} \langle x, y\rangle$. 
  Thus by the definition of Fictitious Play from~\eqref{eq:new-primal}, 
  $x^{t} \in \partial \Psi(y^t)$ for all $t \ge 1$.
  For FTRL with strictly convex regularizer $\phi$, 
  Proposition~\ref{prop:phi-Q} implies for all $t \ge 1$ that
  $x^t = \nabla \phistar(y^t)$, and thus 
  $x^t \in \partial \phistar(y^t)$ by definition. 

  Letting $H$ be the energy function
  for either Fictitious Play or FTRL and $\{x^t\}$ and $\{y^t\}$ 
  the corresponding iterates, then it follows
  from~\eqref{eq:payoff-vector} that 
  \begin{equation*}
    y^{t+1} 
    \;=\; 
    y^t + \eta A x^t 
    \;=\;
    y^t + \eta A \partial H(y^t) \;,
  \end{equation*}
  as claimed. For the second statement, observe 
  by definition of the subgradient set of $H$ and the fact that 
  $x^t \in \partial H(y^t)$, 
  \begin{equation*}
    H(y^{t+1}) - H(y^t) 
    \;\ge\;
    \langle
    x^t, y^{t+1} - y^t
    \rangle
    \;=\;
    \eta
    \langle
    x^t, A x^t
    \rangle
    \;=\;
    0 \;,
  \end{equation*}  
  where the final equality follows by the skew-symmetry of $A$.
\end{proof}

\paragraph{Skew-gradient flow.}
The continuous-time limit (i.e., the limit as $\eta \to 0$) 
of the skew-gradient descent algorithm~\eqref{eq:skew-gradient} 
is the following skew-gradient flow dynamics:
\begin{align*}
    \dot Y_t = A \nabla H(Y_t).
\end{align*}
Here, we assume $H$ is differentiable for simplicity.
Since $A$ is skew-symmetric, the vector field $A \nabla H(Y_t)$ 
is orthogonal to the level set of $H$, and thus the dynamics 
conserves the energy function $H$. Concretely, we can compute:
\begin{align*}
    \frac{d}{dt} H(Y_t) 
    \;=\;
    \langle \nabla H(Y_t), \dot Y_t \rangle
    \;=\;
    \langle \nabla H(Y_t), A \nabla H(Y_t) \rangle 
    \;=\; 0
\end{align*}
so $H(Y_t) = H(Y_0)$ for all $t \ge 0$.

We note that for a zero-sum game with a general payoff matrix, 
it is the  \textit{joint strategy} of the two players in the 
dual space that becomes a skew-gradient flow in continuous time
(see e.g.,~\citet[Section~3.1]{wibisono2022alternating}).
In the present setting of symmetric learning with a 
skew-symmetric payoff matrix, the strategy of each player 
(which is the same for both players) itself follows the 
skew-gradient flow in the dual space in continuous time,
as discussed above.

\section{Worst-Case Regret Bound for FP on High-Dimensional RPS}
\label{app:fp}

In this section, we develop the proof of Theorem~\ref{thm:fp-rps},
which shows Fictitious Play obtains worst-case $O(\sqrt{T})$
regret on every $n$-dimensional RPS matrix. 

\subsection{Cycling Behavior}

We begin by establishing the primal cycling behavior of Fictitious Play:

\fpcyclinglemma*

For this, we first state prove the following proposition,
which characterizes the trajectory of the dual iterates $\{y^t\}$
for a sequence of consecutive primal iterates all 
at the same vertex $e_i$:

\begin{restatable}{prop}{fpcyclingprop}
  \label{prop:fp-rps-structure}
  Suppose at time $t$ that $x^t = e_i$ and 
  $y^t_i \ge y^t_{i-1} > y^t_j$
  for all other $j\in [n] \setminus\{i, i+1\}$.
  Let $\tau := \Big\lceil \frac{y^t_i - y^{t}_{i+1}}{a_i}\Big\rceil$.
  Then $y^{t+s} = y^t + s\cdot A_i$ for all $1 \le s \le \tau$.
\end{restatable}

\begin{proof}
  Without loss of generality (and for readability),
  assume $1 < i < n$, and thus $i+1\;(\mod n) = i+1$
  and $i-1\;(\mod n) = i-1$. 
  Observe from~\eqref{eq:fp-pd} that for any time
  $\ell$ such that $x^\ell = e_i$, we have
  $\Delta y^\ell = y^{\ell+1} - y^\ell = A_i$,
  whose coordinates are given by
  \begin{equation}
    \Delta y^{\ell}_k =  A_{k,i} :=
    \begin{cases}
      -a_{i-1} &\text{if $k = i-1$} \\
      +a_i &\text{if $k = i+1$} \\
      0 &\text{otherwise} \;. 
    \end{cases}
    \label{eq:dual-dynamics}
  \end{equation}
  This implies that at time $t+1$, we have
  $y^{t+1}_i = y^t_i$, $y^{t+1}_{i+1} = y^{t}_{i+1} + a_i$,
  $y^{t+1}_{i-1} = y^{t}_{i-1} - a_{i-1}$, and $y^{t+1}_j = y^t_j$
  for all other $j \in [n] \setminus \{i-1, i, i+1\}$. 
  
  If $y^t_i - y^t_{i+1} \le a_i$, then the statement of the
  proposition trivially holds for $\tau = 1$.
  Otherwise, $y^t_i - y^t_{i+1} > a_i$,
  and we have $y^{t+1}_{i+1} < y^{t+1}_i$ and
  also $y^{t+1}_{i-1} < y^t_{i-1} \le y^{t}_i = y^{t+1}_i$.
  Thus for $\tau \ge 1$, we have  
  $y^{t+1}_i > \max_{j \neq i} y^{t+1}_j$
  which means $x^{t+1} = e_i$ and $\Delta y^{t+1} = A_i$. 
  It follows by induction that for $s \ge 1$,
  we have $x^{t+s} = e_i$ (and $\Delta y^{t+s} = A_i$),
  so long as
  \begin{equation*}
    y^t_{i+1} + s \cdot a_i < y^{t+s}_i = y^t_i
    \quad\iff\quad
    s < \frac{y^t_i - y^{t}_{i+1}}{a_i} \;.
  \end{equation*}
  Recall that
  $\tau := \Big\lceil \frac{y^t_i - y^{t}_{i+1}}{a_i}\Big\rceil$,
  and thus the  latter condition holds for all $s \le \tau - 1$.
  It then follows for all $1 \le s \le \tau$ that
  $y^{t+s} = y^t + s \cdot A_i$, which yields the statement of
  the proposition. 
\end{proof}

\noindent We now use Proposition~\ref{prop:fp-rps-structure} to
prove the cycling property of Lemma~\ref{lem:fp-cycling}:

\medskip

\begin{proof}{\textbf{(of Lemma~\ref{lem:fp-cycling})}}
  We will prove the claim by induction on $k$. 
  First, we show that $x^{t_1} = e_{i+1}$.
  For this, observe by definition of $t_0$ that $\Psi(y^{t_0}) > \Psi(y^1)$.
  Since $x^{t_0} = e_i$ by assumption, this implies that
  $y^{t_0}_i > \max_{j \neq i} y^{t_0}_j$.
  Then setting
  $\tau = \Big\lceil\frac{y^{t_0}_i - y^{t_0}_{i+1}}{a_i}\Big\rceil$,
  Proposition~\ref{prop:fp-rps-structure} implies that
  $x^{t_0 + s} = e_i$ for all $1 \le s \le \tau - 1$ and that 
  $y^{t_0 + \tau} = y^{t_0} + \tau A_i$. This means that
  $t_1 > t_0 + \tau-1$, and also
  \begin{align*}
    y^{t_0 + \tau}_i
    &\;=\; y^{t_0}_i \\
    y^{t_0 + \tau}_{i+1}
    &\;=\;
      y^{t_0}_{i+1} + \tau \cdot a_i \\
    y^{t_0 + \tau}_{j}
    &\;\le\;
      y^{t_0}_j
      \;<\;
      y^{t_0}_i
      \;=\;
      y^{t_0 + \tau}_i
      \quad\text{for all other $j \neq i \neq i+1$} \;.
  \end{align*}
  From here, we separate the argument into two cases:
  first, when $\frac{y^{t_0}_i - y^{t_0}_{i+1}}{a_i} \notin \N$,
  and second when
  $\frac{y^{t_0}_i - y^{t_0}_{i+1}}{a_i} \in \N$.
  In the first case,
  the condition $\frac{y^{t_0}_i - y^{t_0}_{i+1}}{a_i} \notin \N$
  implies that
  $\tau > \frac{y^{t_0}_i - y^{t_0}_{i+1}}{a_i}$ and thus
  \begin{align*}
    y^{t_0 + \tau}_{i+1}
    &\;=\;
      y^{t_0}_{i+1} + \tau \cdot a_i \\
    &\;>\;
    y^{t_0}_{i+1} +
      \Big(\frac{y^{t_0}_i - y^{t_0}_{i+1}}{a_i}\Big) \cdot a_i
      \;=\;
      y^{t_0}_i
      \;=\;
      y^{t_0 + \tau}_i \;.
  \end{align*}
  It follows that $y^{t_0 + \tau}_{i+1} > \max_{j\neq i+1} y^{t_0+\tau}_j$,
  which means $x^{t_0+\tau} = e_{i+1}$, and thus
  the desired claim holds for $t_1 = t_0 + \tau$.

  In the second case, when $\frac{y^{t_0}_i - y^{t_0}_{i+1}}{a_i} \in \N$,
  we instead have
  $y^{t_0 + \tau}_{i+1} = y^{t_0}_{i+1} + \tau \cdot a_i
  = y^{t_0}_i = y^{t_0 + \tau}_{i}$.
  Thus in this case, a tiebreaking rule determines whether
  $x^{t_0 + \tau} = e_{i}$ or  $x^{t_0 + \tau} = e_{i+1}$.
  If $x^{t_0 + \tau} = e_{i+1}$, then the desired claim
  again holds for $t_1 = t_0 + \tau$.
  Moreover, we have $y^{t_0 + \tau}_{i+1} \ge y^{t_0 + \tau}_i >
  y^{t_0 + \tau}_j$ for all other $j \in [n] \setminus \{i, i+1\}$. 
  
  On the other hand, if $x^{t_0 + \tau} = e_{i}$, then
  $y^{t_0+\tau+1} = y^{t_0+\tau} + A_i$, meaning:
  \begin{align*}
    y^{t_0 + \tau + 1}_i
    &\;=\; y^{t_0 + \tau}_i \\
    y^{t_0 + \tau + 1}_{i+1}
    &\;=\;
      y^{t_0 + \tau}_{i+1} + a_i
      \;>\; y^{t_0 + \tau}_i
      \;=\; y^{t_0 + \tau + 1}_{i}
    \\
    y^{t_0 + \tau + 1}_{j}
    &\;\le\;
      y^{t_0 + \tau}_j
      \;<\;
      y^{t_0 + \tau}_i
      \;=\;
      y^{t_0 + \tau + 1}_i
      \quad\text{for all other $j \in [n] \setminus \{i, i+1\}$}\;.
  \end{align*}
  It follows in this case that
  $y^{t_0 + \tau + 1}_{i+1} > \max_{j\neq i+1} y^{t_0 + \tau + 1}_j$,
  which means $x^{t_0 + \tau + 1} = e_{i+1}$, and thus the
  desired claim is satisfied with $t_1 = t_0 + \tau+ 1$. 
    
  In all cases, we find that $x^{t_1} = e_{i+1}$, and that
  $y^{t_1}_{i+1} \ge y^{t_1}_i > y^{t_1}_j$
  for all other $j \neq i \neq i+1$.
  Now assume the claim holds up to phase $k$,
  meaning $y^{t_k}_{i+k\;(\mod n)} > y^{t_k}_j$
  for all other $j$. 
  Then setting  $\tau = \Big \lceil
  \frac{y^{t_k}_{i+k\;(\mod n)} - y^{t_k}_{i+k+1\;(\mod n)}}{a_{i+k\;(\mod n)}}
  \Big\rceil$
  and using identical calculations via an application
  of Proposition~\ref{prop:fp-rps-structure},
  we find that $x^{t_{k+1}} = e_{i+k+1\;(\mod n)}$,
  and thus the claim also holds at phase $k+1$. 
  By induction, this proves the statement of the lemma.
\end{proof}

\subsection{Energy Growth Bound}
\label{app:fp:energy}

\fpenergygrowth*                

\begin{proof}
  For the first case, assume $x^t =  x^{t+1}$.
  Recall that $y^{t+1} = y^t + Ax^t$, and
  using the definition of $\Psi$ and the maximizing principle:
  $\Psi(y^{t+1}) = \langle x^{t+1}, y^{t+1} \rangle$
  and $\Psi(y^t) = \langle x^{t}, y^{t} \rangle$.
  Together, this means:
  \begin{align*}
    \Psi(y^{t+1}) - \Psi(y^t)
    &\;=\;
      \langle x^{t+1}, y^{t+1} \rangle
      -
      \langle x^t, y^{t} \rangle \\
    &\;=\;
      \langle x^{t}, y^{t} + Ax^t \rangle
      - \langle x^t, y^{t} \rangle  \\
    &\;=\;
      \langle x^{t}, y^{t} \rangle
      + \langle x^t, A x^t \rangle
      - \langle x^t, y^{t} \rangle
      \;=\;
      0 \;.
  \end{align*}
  Here, the final equality is due to the
  skew-symmetry of $A$
  (and thus $\langle x, Ax\rangle = 0$
  for all $x \in \R^n$).

  For the second case, where
  $x^t \neq x^{t+1}$, suppose
  $x^t = e_i$ and $x^{t+1} = e_j$ for
  some $i \neq j \in [n]$. 
  Then observe that
  \begin{align*}
    \Psi(y^{t+1}) - \Psi(y^t)
    &\;=\;
      \langle x^{t+1}, y^{t+1} \rangle
      -
      \langle x^t, y^{t} \rangle \\
    &\;=\;
      \langle e_j, y^{t+1} \rangle
      -
      \langle e_i, y^{t} \rangle \\ 
    &\;=\;
      y^{t+1}_j - y^t_i \;.
  \end{align*}
  Let $\alpha^t = y^{t+1}_j - y^t_i$.
  By definition of~\eqref{eq:fp-pd} and the assumption that
  $x^t = e_i$, we have
  \begin{align*}
    y^{t+1}_j
    &\;=\;
      y^t_j + (Ax^t)_j \
      \;=\; y^t_j + A_{ji} \;.
    \\
    y^{t+1}_i
    &\;=\; y^t_i + (Ax^t)_i
      \;=\; y^t_i + A_{ii} = y^t_i \;.
  \end{align*}
  It follows that 
  $
  \alpha^t
  \;=\;
  y^{t+1}_j - y^t_i
  \;=\;
  y^t_j + A_{ji} - y^t_i
  \;\le\;
  y^t_i + A_{ji} - y^t_i
  = A_{ji} \le a_{\max}
  $.
\end{proof}

\subsection{Phase Length Bound}
\label{app:fp:pl}

\lemphaselength*

The cycling property of
the primal iterates helps to establish this
phase length bound in the following way: 
supposing $x^{t_k} = e_i$, then we show roughly 
$\tau_k \approx \gamma_k - y^{t_k}_{i+1}$.
Using the fixed cycling order, and using the fact
that energy increases by at most a constant
between phases, we can deduce that
the magnitude of $y^{t_k}_{i+1}$ is at most
a constant fraction of $\gamma_k$, and the claim follows. 

\medskip

\begin{proof}
  We prove the claim by induction on $k$. First, 
  recall from Lemma~\ref{lem:fp-cycling} that
  for any $\ell \ge 0$, 
  $x^{t_\ell} = e_{i+\ell\;(\mod n)}$
  and $x^{t_{\ell + 1}} = e_{i+\ell +1\;(\mod n)}$. 
  By Proposition \ref{prop:fp-rps-structure} and the fact that
  $y^{t_{\ell} + 1} = y^{t_\ell} + Ax^{t_\ell}$,
  this implies that $\tau_\ell$ must satisfy:
  \begin{equation*}
    y^{t_{\ell+1}}_{i+\ell+1\;(\mod n)}
    \;=\;
    y^{t_{\ell}}_{i+\ell+1\;(\mod n)}
    + \tau_\ell \cdot a_{i+\ell\;(\mod n)}
    \;\ge\;
    \gamma_{\ell}
    \quad\implies\quad
    \tau_\ell
    \;\ge\;
    \frac{\gamma_\ell
      -
      y^{t_{\ell}}_{i+\ell+1\;(\mod n)}}{a_{i+\ell\;(\mod n)}}\;.
  \end{equation*}
  
  For phases $\ell = 0, \dots, n$, note first that
  by definition of~\eqref{eq:fp-pd}, $y^1 = Ax^0$.
  Thus the coordinates of $y^1$ are absolute constants.
  Moreover, as $\Psi(y^{t_0}) - \Psi(y^0) < O(1)$ due
  to Proposition~\ref{prop:fp-energy-growth}, then
  each coordinate of $y^{t_0}$ is also an absolute constant.
  By the cycling property of Lemma~\ref{lem:fp-cycling},
  and using expression~\eqref{eq:dual-dynamics}, observe that
  for each coordinate $i + \ell + 1\;(\mod n)$ for $\ell = 0, \dots, n$,
  we must have $\Delta y^{t}_{i+\ell+1\;(\mod n)} \le 0$
  until $t \ge t_{\ell}$. It follows that
  $y^{t_\ell}_{i + \ell + 1\;(\mod n)} \le \kappa$
  for some absolute constant $\kappa > 0$,
  and thus
  $\tau_\ell \ge (\gamma_\ell - \kappa) / a_{i+\ell\;(\mod n)}$.
  Thus for each phase $0 \le \ell \le n$,
  the statement of the lemma holds for
  $\alpha_\ell = 1/a_{i+\ell\;(\mod n)} > 0$
  and $\beta_{\ell} = \kappa/a_{i+\ell\;(\mod n)}$. 
   
  Now assume that the claim holds for
  all phases up to $k-1$. We will prove it holds
  also for phase $k > n$. For this, recall that we have
  \begin{equation}
    \tau_{k}
    \;\ge\;
    \frac{\gamma_k - y^{t_k}_{i+k+1\;(\mod n)}}{a_{i+k\;(\mod n)}} \;.
    \label{eq:pl-0}
  \end{equation}
  Now by the cycling property of Lemma~\ref{lem:fp-cycling}
  and using expression~\eqref{eq:dual-dynamics},
  it follows that $\Delta y^t_{i+k+1\;(\mod n)} = 0$
  for all $t_{k+3-n} \le t < t_k$,
  and thus $y^{t_k}_{i+k+1\;(\mod n)} =
  y^{t_{k+3-n}}_{i+k+1\;(\mod n)}$.
  Moreover, we also have from
  Proposition~\ref{prop:fp-rps-structure} that 
  \begin{align}
    y^{t_k}_{i+k+1\;(\mod n)}
    \;=\;
    y^{t_{k+3-n}}_{i+k+1\;(\mod n)}
    &\;=\;
      y^{t_{k+2-n}}_{i+k+1\;(\mod n)}
      -
      a_{i+k+1\;(\mod n)} \cdot \tau_{k+2-n}
    \nonumber \\
   &\;\le\;
      \gamma_{k+2-n} + a_{\max}
     -
     a_{i+k+1\;(\mod n)} \cdot \tau_{k+2-n} \;,
     \label{eq:pl-1}
  \end{align}
  where $a_{\max} > 0$ is an absolute constant,
  and where the inequality follows from the
  energy growth bound of Proposition~\ref{prop:fp-energy-growth}.
  By the inductive hypothesis, we have
  $\tau_{k+2-n} \ge \alpha_{k+2-n}\cdot \gamma_{k+2-n} - \beta_{k+2-n}$
  for absolute constants $\alpha_{k+2-n}, \beta_{k+2-n} > 0$.
  Substituting this into~\eqref{eq:pl-1}, we then find:
  \begin{align}
    y^{t_k}_{i+k+1\;(\mod n)}
    &\;\le\;
      \gamma_{k+2-n}
      -
      a_{i+k+1\;(\mod n)} \cdot
      (\alpha_{k+2-n}\cdot \gamma_{k+2-n} - \beta_{k+2-n}) 
      + a_{\max} \\
    &\;=\;
      \gamma_{k+2-n}\cdot
      (1 - a_{i+k+1\;(\mod n)} \cdot \alpha_{k+2-n})
      + \beta_{k+2-n} + a_{\max} \;.
      \label{eq:pl-2}
  \end{align}
  Without loss of generality, we will assume
  that expression~\eqref{eq:pl-2} is positive. Otherwise,
  we would trivially have in expression~\eqref{eq:pl-0}
  that $\tau_k \ge \gamma_k / a_{i+k\;(\mod n)}$,
  which would yield the desired claim.
  Similarly, we assume additionally that
  $1 - a_{i+k+1\;(\mod n)} \cdot \alpha_{k+2-n} > 0$,
  as otherwise the bound in expression~\eqref{eq:pl-2}
  is at most an absolute constant $\epsilon > 0$, which would imply
  $\tau_k \ge (\gamma_{k} - \epsilon)/a_{i+k\;(\mod n)}$,
  again trivially yielding the desired claim.
  Then under these assumptions, we can substitute
  expression~\eqref{eq:pl-2} back into~\eqref{eq:pl-0} to find
  \begin{align}
    \tau_k
    &\;\ge\;
      \frac{\gamma_k -
      \gamma_{k+2-n}\cdot (1 - a_{i+k+1\;(\mod n)} \cdot \alpha_{k+2-n})
      }{a_{i+k\;(\mod n)}}
      - \frac{\beta_{k+2-n} + a_{\max}}{a_{i+k\;(\mod n)}}
    \\
    &\;\ge\;
      \frac{\gamma_k \cdot a_{i+k+1\;(\mod n)} \cdot \alpha_{k+2-n}}{
      a_{i+k\;(\mod n)}}
      - \frac{\beta_{k+2-n} + a_{\max}}{a_{i+k\;(\mod n)}}       \;,
      \label{eq:fp-pl-end}
  \end{align}
  where the final inequality comes from the fact that
  $\gamma_{k+2-n} \le \gamma_{k}$ and the assumption
  that $1 - a_{i+k+1\;(\mod n)} \cdot \alpha_{k+2-n} > 0$.
  It follows that $\tau_k \ge \alpha_k \gamma_k  - \beta_k$ for
  \begin{equation*}
    \alpha_k
    \;:=\;
    \frac{a_{i+k+1\;(\mod n)} \cdot \alpha_{k+2-n}}{a_{i+k\;(\mod n)}} > 0
    \quad\text{and}\quad
    \beta_k
    \;:=\;
    \frac{\beta_{k+2-n} + a_{\max}}{a_{i+k\;(\mod n)}} > 0\;,
  \end{equation*}
  which completes the proof. 
\end{proof}

\subsection{Proof of Theorem~\ref{thm:fp-rps}}
\label{app:fp:regret-thm}

Recall from Definition~\ref{def:phases} that $K \ge 1$ is
the total number of phases undergone by the
dynamics in $T$ rounds. 
For each phase $k = 0, 1 \dots, K$ define the 
indicator variable $c_k \in \{0, 1\}$ as follows:
\begin{equation*}
  c_k
  \;:=\;
  \begin{cases}
    0 &\text{if $\Psi(y^{t_k}) = \Psi(y^{t_{k-1}})$} \\
    1 &\text{if $\Psi(y^{t_k}) > \Psi(y^{t_{k-1}})$}
  \end{cases} \;.
\end{equation*}
In other words, $c_k = 1$ if and only if the energy
of the dynamics strictly increases from phase $k-1$
to phase $k$.
Observe that the energy growth
bound of Proposition~\ref{prop:fp-energy-growth}
further implies that if $c_k = 1$, then
$\Psi(y^{t_k}) - \Psi(y^{t_{k-1}}) = \Theta(1) = \Theta(c_k)$.
This yields the following corollary to Proposition~\ref{prop:fp-energy-growth}:
\begin{corollary}
  \label{cor:regret-energy}
  Assume the setting of Theorem~\ref{thm:fp-rps}. Then
  $\reg(T) = 2 \cdot \Psi(y^{T+1}) =
  \Theta\Big(\sum_{k=1}^K c_k\Big)$. 
\end{corollary}

We can now prove the main result of Theorem~\ref{thm:fp-rps}.
In particular, the use of the indicators $\{c_k\}$
allows for handling the case of
non-increasing energy between phases, as mentioned
in the sketch presented in Section~\ref{sec:fp}. 

\medskip

\begin{proof}
  For readability, we will use the notation $f \lesssim g$
  and $f \gtrsim g$ to indicate $f = O(g)$ and $f = \Omega(g)$ respectively. 
  Now without loss of generality, assume $t_0 < T$,
  as otherwise $\Psi(y^T) = \Psi(y^1)$ is a constant,
  and the statement of the theorem trivially holds. 
  Then by Corollary~\ref{cor:regret-energy}, our goal is
  to derive an upper bound on $\sum_{k=0}^K c_k$.
  For this, recall from Definition~\ref{def:phases} that
  $T = \sum_{k=0}^K \tau_k$.
  By Lemma~\ref{lem:phase-length}, we have for
  all $k$ that:
  \begin{equation*}
    \tau_k
    \;\ge\;
    \alpha_k \cdot 
    \Psi(y^{t_k}) - \beta_k
    \;\gtrsim\;
    \sum_{i=1}^k c_i - \beta_k \;,
  \end{equation*}
  for positive absolute constants $\alpha_k, \beta_k > 0$.
  It follows that we can write
  \begin{equation} 
    T
    \;=\;
    \sum_{k=0}^K \tau_k
    \;\gtrsim\;
    \sum_{k=0}^K \sum_{i=1}^k c_i - \sum_{k=0}^K \beta_k
    \;\gtrsim\;
    \sum_{k=0}^K \sum_{i=1}^k c_i - T \;,
    \label{eq:regret-1}
  \end{equation}
  where the final inequality comes from the fact that $K \le T$. 
  Observe that we can also write
  \begin{equation}
    \sum_{k=0}^K \sum_{i=1}^k c_i
    \;=\;
    Kc_1 + (K-1)c_2 + \dots + c_K
    \;=\;
    \sum_{k=0}^K (K-k+1)c_k
  \end{equation}
  
  Now let $L$ be the set of indices in $\{0,1, \dots, K\}$ where $c_i = 1$, and
  let $\mathbbm{1}_L$ be the indicator function of membership to the set $L$.
  Since there are $\vert L\vert$ non-zero elements of $\{c_i\}$, we can write:
  \begin{equation}
    \sum_{k=0}^K (K-k+1)c_k
    \;=\;
    \vert L\vert \cdot K - \sum_{k=0}^K (\mathbbm{1}_L \cdot k) + \vert L\vert
  \end{equation}
  The second term is the sum of indices in $L$, which can be upper-bounded
  by the sum of the top-$\vert L\vert$ indices. This means:
  \begin{align}
    \vert L\vert \cdot K - 
    \sum_{k=0}^K (\mathbbm{1}_L \cdot k) + 
    \vert L\vert
    &\;\geq\;
      \vert L\vert \cdot K - 
      \frac{\vert L\vert}{2}(2K-\vert L \vert) + \vert L\vert\\
    &\;=\; \frac{\vert L\vert}{2}^2 + \vert L\vert
      \;\geq\; \frac{1}{2} \vert L\vert^2 \;.
      \label{eq:fp-regret-end}
  \end{align}
  By definition, $\vert L\vert = \sum_{k=0}^K c_k$,
  and thus we conclude
  $
  \sum_{k=0}^K \sum_{i=1}^k c_i
  \gtrsim
  (|L|)^2
  = 
  \Big(\sum_{k=0}^K c_k \Big)^2 .
  $
  Substituting this into expression~\eqref{eq:regret-1}
  and rearranging, we then obtain via Corollary~\ref{cor:regret-energy}:
  \begin{equation*}
    \Big(\sum_{k=0}^K c_k \Big)^2
    \;\lesssim\;
    T
    \quad\implies\quad
    \reg(T) \;\lesssim\; \sum_{k=0}^K c_k \;\lesssim\; \sqrt{T} \;,
  \end{equation*} 
  which completes the proof.
\end{proof}

\section{Constant Regret for FP
  via Tournament Tiebreaking}
\label{app:fp:tournament}

The cycling behavior of Lemma~\ref{lem:fp-cycling}
also leads to an improved regret
bound under the following ``tournament'' tiebreaking rule,
which mirrors the cyclical tournament graph structure latent
in the definition of generalized RPS:
\begin{definition}[Tournament Tiebreaking Rule]
  \label{def:tournament-tiebreak}
  Let $A$ be an $n$-dimensional RPS matrix
  from Definition~\ref{def:rps}. 
  Using the \textit{tournament} tiebreaking rule,
  ties between coordinates $i \in [n]$ and $j\in[n]$ are 
  broken lexicographically, except for ties between coordinates
  $1$ and $n$, which are broken in favor of coordinate $1$.
\end{definition}

For \textit{unweighted} RPS matrices,
we show in Theorem~\ref{thm:tournament-regret} 
that when $x^0$ is a vertex $e_i \in \Delta_n$,
the energy of the dual iterates is 
exactly conserved over time, and the
regret is thus constant. 

\begin{restatable}{theorem}{fptournamentregret}
  \label{thm:tournament-regret}
  For any $n\ge3$, let $A$ be an $n$-dimensional RPS matrix from
  Definition~\ref{def:rps} with all
  $a_i = 1$. Let $\{x^t\}$ and $\{y^t\}$ be the iterates
  of \eqref{eq:fp-pd} initialized at $x^0 = e_i$
  for some $i\in[n]$.
  Then using the tournament tiebreaking rule,
  $\Psi(y^T) = \Psi(y^1)$ and $\reg(T) = O(1)$. 
\end{restatable}

\begin{proof}
  We will show that under the tournament tiebreaking rule
  of Definition~\ref{def:rps}, for any
  $t$ such that $x^{t} \neq x^{t+1}$, the energy
  $\Psi(y^t) = \Psi(y^{t+1})$.
  First, since all $a_i=1$ in the unweighted case, 
  we have that $y^1_i = Ax_i^0 \in \mathbb{Z}$ for $i\in[n]$. 
  Recall that by definition, $y^t = y^1 + \sum_{k=0}^{t-1} Ax^k$. 
  Since, under~\eqref{eq:fp-pd}, each $x^k = e_j$
  for some $j \in [n]$, it follows for each $t$ that we can write 
  $y^t =  y^1 + \sum_{i=1}^n c^t_i \cdot A_i$
  for non-negative integer constants $c^t_1, \dots, c^t_n$.
  Altogether, we have
  $y^t_i
  \;=\; y^1_i + c^t_{i-1} - c^t_{i+1} = y^1_i + d^t_i$ for any $i \in [n]$,
  where $d^t_i \in \Z$. 
  Thus for any $i, j \in [n]$, we have:
  $y^t_i - y^t_j = 
  y^1_i - y^1_j
  + d^t_{i} - d^t_j$.
  As argued above, $y^1_i - y^1_j \in \mathbb{Z}$ 
  and thus $y^t_i - y^t_j \in \mathbb{Z}$.

  Let $i := \argmax_{j \in [n]} y^t_j$ using the tournament
  tiebreaking rule of Definition~\ref{def:tournament-tiebreak},
  and set $\tau = \big\lceil y^t_i - y^t_{i+1} \big\rceil$.
  Without loss of generality, assume $i < n$,
  so that $i+1\;(\mod n) = i+1$. 
  Then using identical calculations as in
  Proposition~\ref{prop:fp-rps-structure}, we have
  \begin{align*}
    y^{t+\tau}_i
    &\;=\; y^t_i \\
    y^{t+\tau}_{i+1}
    &\;=\;
      y^t_{i+1} + \big\lceil y^t_i - y^t_{i+1} \big \rceil \\
    \text{and}\quad
    y^{t+\tau}_j 
    &\;\le\;
      y^{t}_j
      &\text{for all other $j \neq i \neq i+1$} \;.
  \end{align*}
  Moreover, we also have $x^{t+s} = e_i$ 
  and $\Psi(y^{t+s}) = y^{t+s}_i = \Psi(y^t)$
  for all $1 \le s \le \tau-1$.
  As $y^{t}_i - y^t_{i+1} \in \Z$ by the arguments above,
  we have $\big \lceil y^{t}_i - y^t_{i+1} \big \rceil \in \mathbb{Z}$
  and thus $y^{t+\tau}_{i+1} = 
  y^t_{i+1} + y^t_i - y^t_{i+1} = y^{t}_i = y^{t+\tau}_{i}$.
  Using the tournament tiebreaking rule,
  it follows that $x^{t+\tau} = e_{i+1} \neq e_i = x^{t+\tau-1}$. Moreover, $\Psi(y^{t+\tau}) = \Psi(y^{t+\tau -1}) = \dots = \Psi(y^t) = \dots = \Psi(y^1)$.
  Thus by Proposition~\ref{prop:regret-energy-fp} we have
  $\mathrm{Reg}(T) = 2\cdot\Psi(y^T) = O(1)$, which proves the claim. 
\end{proof}


\section{Details on Gradient Descent Under Symmetric Learning}
\label{app:gd:details}

\subsection{Closed Form Primal Update and Energy Function}
\label{app:gd:details:closed}

Under~\eqref{eq:gd-pd}, the primal
iterates can be characterized in
closed-form using the KKT conditions
of the constrained optimization problem
over $\Delta_n$.
In particular, under Gradient Descent,
\cite{bailey2019fast} give
the following characterization
of the primal iterate $x^t$,
which in turn leads to a
closed-form characterization of the energy
function $\phi^*(y^t)$. 
At time $t$, both expressions are defined with
respect to a set $S^t$, which characterizes the support of $x^t$.
The set $S^t$ can be found
using the iterative method of \textsc{FindSupport} 
given in Algorithm~\ref{alg:optimalset}.

\begin{restatable}[\cite{bailey2019fast}]{prop}{gdclosed}
  \label{prop:gd-closed}
  Let $\{x^t\}$ and $\{y^t\}$ be
  iterates of~\eqref{eq:gd-pd}.
  At each time $t \ge 1$,
  let $S^t = \textsc{FindSupport}(y^t) \subseteq [n]$,
  and let $|S^t| = m$.
  Then 
  \begin{align}
    x^{t}_i
    \;=\;
    \begin{cases}
      0 &\quad\text{if $i \notin S^t$} \\
      y^t_i - \frac{1}{m} \cdot \sum_{j \in S^t} y^t_j + \frac{1}{m}
        &\quad\text{if $i \in  S^t$} \;.
    \end{cases}
    \label{eq:gd-primal}
  \end{align}
  Moreover, the energy $\phi^*(y^t)$ is given by
  \begin{equation}
    \phi^*(y^t)
    \;=\;
    \frac{1}{2}
    \sum_{j \in S^t} \left(y^t_j\right)^2
    +
    \frac{1}{m}
    \sum_{j \in S^t}
    y^t_j
    -
    \frac{1}{2m}
    \bigg( \sum_{j \in S^t} y^t_j \bigg)^2
    - \frac{1}{2m} \;.
    \label{eq:gd-energy}
  \end{equation}
\end{restatable}

\setlength{\algomargin}{2em}
\begin{algorithm2e}
  \SetAlgoLined
  \caption{\textsc{FindSupport}}
  \label{alg:optimalset}
  \KwIn{$y \in \R^n$}
  $S \gets [n]$ \\
  \emph{Search():} \\
  \hspace*{1.5em}
  Select $i \in \argmin_{j \in S} y_j$ \\
  \hspace*{1.5em}
  \textbf{if}
  $y_i - \frac{1}{|S|}
  \big(\sum_{j \in S} y^t_j\big)
  + \frac{1}{|S|} < 0$
  \textbf{then:}\\
  \hspace*{3em}
  $S\gets S\setminus \{i\}$ \\
  \hspace*{3em}
  \textbf{goto} \emph{Search()} \\
  \hspace*{1.5em}
  \textbf{else return $S$} 
\end{algorithm2e}

While Proposition~\ref{prop:gd-closed}
is stated with respect to the iterates of~\eqref{eq:gd-pd},
given any dual vector $y \in \R^n$, under the maps of
$\phi^*: \R^n \to \R$ and $Q: \R^n \to \Delta_n$
from Proposition~\ref{prop:phi-Q}, 
we can more generally describe the corresponding primal
iterate and energy value in closed form similar
to expressions~\eqref{eq:gd-primal}
and~\eqref{eq:gd-energy}.
Specifically, in the case of Gradient Descent
where $\phi(x) = \frac{\|x\|^2_2}{2}$,
recall for $y \in \R^n$ that
$\phi^*$ and $Q$ are given by:
\begin{equation}
  \begin{aligned}
    \phi^*(y)
    &\;=\;
      \max_{x \in \Delta_n}\;
      \langle x, y \rangle - \frac{\|x\|^2_2}{2} \\
    Q(y)
    &\;=\;
      \argmax_{x \in \Delta_n}\;
      \langle x, y \rangle - \frac{\|x\|^2_2}{2} \;.
  \end{aligned}
  \label{eq:gd-maps}
\end{equation}
Then as a corollary to Proposition~\ref{prop:gd-closed}
(and Proposition~\ref{prop:phi-Q}),
we state the following
(see \citet[Appendix B]{bailey2019fast}):

\begin{corollary}
  \label{cor:gd-closed}
  Let $\phi^*: \R_n \to \R$ and
  $Q: \R^n \to \Delta_n$
  be the functions in expression~\eqref{eq:gd-maps}.
  For any $y \in \R^n$,
  let $S = \findsupport(y) \subseteq [n]$,
  and let $|S| = m$.
  Let $x := Q(y) \in \Delta_n$.
  Then for $i \in S$, the coordinate $x_i$
  is given by
  \begin{equation*}
    x_i
    \;=\;
    y_i - \frac{1}{m} \cdot \sum_{j \in S} y_j
    + \frac{1}{m} \;,
  \end{equation*}
  and $x_i = 0$ otherwise. 
  Additionally, the energy $\phi^*(y)$ is given by
  \begin{equation}
    \phi^*(y)
    \;=\;
    \frac{1}{2}
    \sum_{j \in S} (y_j)^2
    +
    \frac{1}{m}
    \sum_{j \in S}
    y_j
    -
    \frac{1}{2m}
    \big( \sum_{j \in S} y_j \big)^2
    - \frac{1}{2m} \;.
    \label{eq:gd-energy-closed}
  \end{equation}
  Moreover, $Q(y) = \nabla \phi^*(y)$. 
\end{corollary}

In the case where $\phi = (\|\cdot \|^2_2)/2$
as in Gradient Descent, the key property of
Proposition~\ref{prop:gd-closed}
and Corollary~\ref{cor:gd-closed} is that
the map $Q(y) = \nabla \phi^*(y)$
is not injective: in particular,
multiple dual vectors of $\R^n$ map to
the same distribution on the boundary of $\Delta_n$
under $Q$.

\subsection{Primal-Dual Mapping for Gradient Descent}
\label{app:gd:details:pd-map}

In light of Proposition~\ref{prop:gd-closed}
and Corollary~\ref{cor:gd-closed},
under expression~\eqref{eq:gd-energy-closed}
(or equivalently, under the map $Q$
from expression~\eqref{eq:gd-maps}),
we define a convenient mapping
between the primal and dual spaces
that holds for certain regions of
the boundary of the simplex.
For this, we first recall the definition
of the sets $P_i$ and $P_{i\sim(i+1)}$
from Section~\ref{def:gd-p-regions}. 

\defgdpregions*

Then we prove the following relationship
(for simplicity, we state the result
for the
iterates of~\eqref{eq:gd-pd},
but the same statement holds
for $x\in \Delta_n$
and $y \in \R^n$ such that $Q(y) = x$):

\gdpdmap*

\begin{proof}
  We first prove the equivalence
  $y^t \in P_i \iff x^t = e_i$.
  For the forward direction, assume $y^t \in P_i$.
  By definition of $P_i$,
  we have $y^t_j - y^t_i < -1$ for any $j \neq i$.
  Then by definition of $\findsupport$
  (Algorithm~\ref{alg:optimalset})
  it follows that $j \notin S^t$.
  Thus $S^t = \supp(x^t) = \{i\}$, meaning $x^t = e_i$. 

  For the backward direction,
  assume $x^t = e_i$, which means
  $S^t=\findsupport(y^t)=\{i\}$.
  By definition of $\findsupport$, this means
  for all $j \neq i$ that $y^t_j - y^t_i < -1$,
  and thus $y^t \in P_i$.

  For the second part of the proposition,
  we prove the equivalence
  $y^t\in P_{i\sim(i+1)} \iff \supp(x^t) = \{i, i+1\}$.
  For the forward direction, assume $y^t \in P_{i\sim(i+1)}$.
  By definition of $P_{i\sim(i+1)}$,
  we have $\frac{y^{t}_i + y^t_{i+1}}{2} \le y^t_i + \frac{1}{2}$
  and $\frac{y^{t}_i + y^t_{i+1}}{2} \le y^t_{i+1} + \frac{1}{2}$,
  which implies $\min\{y^t_i, y^t_{i+1}\} > y^t_j$
  for all other $j\neq i \neq i+1$.
  Then by definition of $\findsupport$ and $P_{i\sim(i+1)}$,
  it follows for all such $j$ that the condition in line 4
  is violated, and thus $j \notin S^t$.
  On the other hand, assuming without loss of generality
  that $y^t_{i+1} = \min\{y^t_i, y^t_{i+1}\}$,
  then by definition of $P_{i\sim(i+1)}$
  we find
  $y^t_i - \frac{y^t_{i}+y^t_{i+1}}{2}
  = \frac{y^t_i - y^t_{i+1}}{2} \ge -1/2$.
  Thus the condition in line 4 of $\findsupport$
  fails, and we have $S^t = \supp(x^t) = \{i, i+1\}$.

  For the backward direction, assume $\supp(x^t) = \{i, i+1\}$.
  Then by definition of $\findsupport$
  we must have $y^t_j - \frac{y^t_i + y^t_{i+1}}{2} < -\frac{1}{2}$
  for all other $j$.
  Moreover, since $x^t_i, x^t_{i+1} > 0$,
  we find 
  using expression
  \eqref{eq:gd-primal}  that
  $y^t_{i} - y^t_{i+1} \ge - 1$
  and   $y^t_{i+1} - y^t_{i} \ge - 1$.
  Thus $y^t \in P_{i\sim(i+1)}$.
\end{proof}

\subsection{Geometry of the Energy Function}
\label{app:gd:details:energy-max}

\paragraph{Energy function on the simplex boundary.}
For~\eqref{eq:gd-pd}, 
when the dual iterate $y^t$ is in
the region $P_i$ or $P_{i\sim(i+1)}$,
the energy function $\phi^*$
(from expression~\eqref{eq:gd-energy} in
Proposition~\ref{prop:gd-closed})
has the following simplified form:
\begin{align}
  \phistar(y^t)
  &\;=\;
  y^t_i - \frac{1}{2}
  &\text{for}
    \;y^t \in P_i.
    \label{eq:energy-pi}\\
  \phistar(y^t)
  &\;=\;
    \frac{1}{4}(y^t_i-y^t_{i+1})^2
    + \frac{1}{2} (y^t_{i} +y^t_{i+1}) - \frac{1}{4}
  &\text{for}
    \;y^t \in P_{i\sim(i+1)}.
    \label{eq:energy-psimi}
\end{align}
Then we have the following relationship
between energy and the maximum coordinate
of the dual iterate within these regions:

\begin{lemma}
  \label{lem:gd-energy-max}
  Let $\phi^*$ be the energy
  function from Proposition~\ref{prop:gd-closed}.
  Then for each $i \in [n]$:
  \begin{enumerate}[label={(\roman*)}]
  \item
    If $y^t \in P_i$, then
    $\phi^*(y^t)
    \le y^t_i
    \le
    \phi^*(y^t) + \frac{1}{2}
    $. 
  \item
    If $y^t \in P_{i\sim(i+1)}$, then
    $\phi^*(y^t) - \frac{1}{2}
    \le
    y^t_i, y^t_{i+1}
    \le
    \phi^*(y^t) + \frac{3}{4}
    $    
  \end{enumerate}
\end{lemma}

\begin{proof}
  Fix $y^t \in P_i$.
  Then by~\eqref{eq:energy-pi}, we have 
  $\phi^*(y^t) = y^t_i - \frac{1}{2}$.
  Thus $\phi^*(y^t) \le y^t_i \le \phi^*(y^t) + \frac{1}{2}$.
  For $y^t \in P_{i\sim(i+1)}$, 
  we have by definition that
  $y^t_i - 1 \le y^t_{i+1} \le y^t_i + 1$.
  Then using the form of $\phistar$
  from~\eqref{eq:energy-psimi}, we can bound
  \begin{align*}
    y^t_i + \frac{1}{2}
    \;\ge\;
    \frac{y^t_i + y^t_{i+1}}{2}
    \;=\;
      \phi^*(y^t)
      -
      \frac{1}{4}(y^t_i - y^t_{i+1})^2 + \frac{1}{4} 
    \;\ge\;
      \phi^*(y^t) \;,
  \end{align*}
  where the final inequality comes from
  the fact that $|y^t_i - y^t_{i+1}| \le 1$
  when $y^t \in P_{i\sim(i+1)}$.
  Rearranging, we have $y^t_i \ge \phi^*(y^t) - \frac{1}{2}$.
  An identical calculation also finds that
  $y^t_{i+1} \ge \phi^*(y^t) - \frac{1}{2}$.

  For the upper bound, we similarly have
  \begin{align*}
    y^t_i - \frac{1}{2}
    \;\le\;
    \frac{y^t_i + y^t_{i+1}}{2}
    \;=\;
    \phi^*(y^t)
    -
    \frac{1}{4}(y^t_i - y^t_{i+1})^2 + \frac{1}{4} 
    \;\le\;
    \phi^*(y^t) + \frac{1}{4}\;.
  \end{align*}
  Rearranging, we have
  $y^t_i \le \phi^*(y^t) + \frac{3}{4}$.
  An identical calculation also yields
  $y^t_{i+1} \le \phi^*(y^t) + \frac{3}{4}$.
\end{proof}

\paragraph{Energy of dual iterate when
  primal iterate has full support.}
For $y \in \R^n$ such that $S = \findsupport(y) = [n]$,
then by Corollary~\ref{cor:gd-closed},
the primal iterate $x = Q(y)$ is \textit{interior}
and has full support (e.g., $\supp(x) = [n]$). 
In this case, letting $\1 \in \R^n$ denote the
all-ones vector,
the energy function can be simplified as
\begin{equation}
  \phi^*(y^t)
  \;=\;
  \frac{\|y^t\|^2_2}{2} +
  \frac{\langle y^t, \1 \rangle}{2n}
  -
  \frac{(\langle y^t, \1 \rangle)^2}{2n} - \frac{1}{2n}
  \;.
  \label{eq:energy-interior}
\end{equation}


\section{Gradient Descent on High-Dimensional RPS
  in Large Stepsize Regime}
\label{app:gd:large}

In this section we develop the proof of
Theorem~\ref{thm:gd-large-stepsize},
which establishes that Gradient Descent
obtains $O(\sqrt{T})$ regret
on high-dimensional RPS games
in a regime of \textit{large constant stepsizes}:

\gdrpslargestepsize*

\noindent
The organization of this section is as follows:
\begin{itemize}[
    leftmargin=2em,
]
\item
    In Section~\ref{app:gd:large:convergence},
    we give the proof of Lemma~\ref{lem:gd-large-initial},
    which shows that, for nearly all initializations,
    the primal iterates of Gradient Descent
    quickly converge to a vertex of $\Delta_n$
    when the stepsize is sufficiently large.
\item
    In Section~\ref{app:gd:large:cycle-pl},
    we then develop the proof of 
    Lemma~\ref{lem:gd-cycling-pl-overview}.
    Here, we show in the large stepsize regime
    that the primal iterates of Gradient Descent
    cycle through the vertices of $\Delta_n$
    in phases, in a manner analogous to Fictitious Play
    (c.f., Lemma~\ref{lem:fp-cycling}).
    Lemma~\ref{lem:gd-cycling-pl-overview}
    also establishes similar bounds on the length
    of each phase and on the energy growth
    between phases (c.f., Lemma~\ref{lem:phase-length}
    and Proposition~\ref{prop:fp-energy-growth} for Fictitious Play). 
\item 
    Using these ingredients, the full
    proof of Theorem~\ref{thm:gd-large-stepsize}
    is given in Section~\ref{app:gd:large:regret}
    and again follows similarly to the
    proof of Theorem~\ref{thm:fp-rps}
    for Fictitious Play.
\end{itemize}

\subsection{Fast Convergence to a Vertex}
\label{app:gd:large:convergence}

\gdlargeinitial*

\begin{proof}
  Fix $x^0 \in \Delta^n$, and recall by definition of
  \eqref{eq:gd-pd} that $y^1 = \eta Ax^0$.
  In particular, by the structure of the entries of
  an RPS matrix $A$, we have for all $i \in [n]$:
  \begin{equation*}
    y^1_i
    \;=\;
    \eta 
    \big(a_{i-1} \cdot x^0_{i-1} - a_i \cdot x^0_{i+1}\big)\;.
  \end{equation*}
  Without loss of generality, assume
  indices $i$ and $j$ are the largest and
  second largest coordinates of $y^1$ satisfying
  $y^1_i = \max_{k \in [n]} y^1_k$ and
  $y^1_j = \max_{k \in [n] \setminus \{i\}} y^1_k$.
  Observe from Proposition~\ref{prop:gd-pd-map} that
  if $y^1_i - y^1_j > 1$, then $y^1 \in P_i$,
  and thus $x^1 = e_i$. In particular, this
  condition will be satisfied when
  \begin{equation}
    \eta 
    \Big(
    \big(a_{i-1} \cdot x^0_{i-1} - a_i \cdot x^0_{i+1}\big)
    - 
    \big(a_{j-1} \cdot x^0_{j-1} - a_j \cdot x^0_{j+1}\big)
    \Big)
    \;>\; 1 \;.
    \label{eq:gdl-1}
  \end{equation}
  Recall from expression~\eqref{eq:gamma-main} 
  that the constant $\gamma(x^0)$ is given by
  \begin{equation}
    \gamma(x^0)
    \;=\;
    \min_{k, \ell \in [n]}
    \Big|
    \big(a_{k-1} \cdot x^0_{k-1} - a_k \cdot x^0_{k+1}\big)
    - 
    \big(a_{\ell-1} \cdot x^0_{\ell-1} - a_\ell \cdot x^0_{\ell+1}\big)
    \Big| \;.
    \label{eq:gd-gamma}
  \end{equation}
  Observe that if $\gamma(x^0) > 0$, 
  and so long as $\eta > 1/\gamma(x^0)$,
  then the inequality in~\eqref{eq:gdl-1}
  is always satisfied, which ensures $y^1 \in P_i$.   
  Fixing the matrix $A$,
  the set of points $x \in \Delta_n$
  where $\gamma(x) = 0$
  is restricted to the linear constraint in
  ~\eqref{eq:gd-gamma} and thus has
  (Lebesgue) measure zero. 
  Thus for almost all initial conditions $x^0$,
  setting $\eta > 1/\gamma(x^0)$ guarantees $x^1 = e_i$.
\end{proof}

\subsection{Cycling, Energy Growth, and Phase Length Bounds}
\label{app:gd:large:cycle-pl}

With sufficiently large stepsizes,
once the primal iterate reaches a vertex $e_i$
we prove in Lemma~\ref{lem:gd-cycling-pl-overview} below
that the iterates subsequently
cycle through the vertices in phases in the order
\begin{equation*}
e_i \to e_{i+1} \to \dots \to e_{n} \to e_{1} \to \dots \to e_i \to \dots
\end{equation*}
similar to the behavior of Fictitious Play
from Lemma~\ref{lem:fp-cycling}.
For a sequence of consecutive primal iterates
all at the same vertex, the energy
growth of $\phi^*$ is zero. 

On the other hand, between vertices
$e_i$ and $e_{i+1}$, the primal
iterates may spend a constant number of
iterations on the simplex edge between
vertices $e_i$ and $e_{i+1}$
(e.g., where $\supp(x) = \{i, i+1\}$ for $x \in \Delta_n$).
Consecutive primal iterates on these edges
correspond to a constant energy growth per step.
However, using a sufficiently large constant stepsize,
we can ensure that at most a \textit{single round} is spent
on any such edge before transitioning to the next vertex.
This ensures a regularity in the cycling behavior
closely mirroring that of Fictitious Play,
ultimately leading to bounds on the energy growth and
length of each phase analogous to those in 
Lemma~\ref{lem:phase-length}. 

To formally analyze this cycling behavior
and its consequences, we reuse the notation
of phases from the analysis of Fictitious Play, 
which we restate here:

\defphases*

\noindent
We then formally state the previously-described
behavior in the following lemma:

\begin{restatable}{lem}{gdcyclingphaselengthoverview}
  \label{lem:gd-cycling-pl-overview}
  Let $x^0 \in \Delta_n$ be
  such that $\gamma(x^0) > 0$
  (for $\gamma$ defined in~\eqref{eq:gd-gamma}),
  and assume $\eta > \max \{2/a_{\min}, 1/\gamma(x_0)\}$. 
  Then there exists $t_0 \le O(n)$ 
  such that $x^{t_0} = e_i$ for some $i \in [n]$. 
  Moreover, for each $k = 1, \dots, K$:
  \begin{enumerate}[
    label={(\roman*)}, 
    leftmargin=3em
  ]
  \item
    $y^{t_k} \in P_{i+k\;(\mod n)}$,
    and
    $y^{t_{k+1}-1} \in
    \big\{
    P_{i+k-1\;(\mod n)},
    P_{(i+k+1\;(\mod n))\sim(i+k(\mod n))}
    \big\}$.
  \item
    $\phi^*(y^{t_{k+1}}) - \phi^*(y^{t_{k}})
    \le O(1)$
  \item
    $\tau_k \ge \alpha_k \cdot \phi^*(y^{t_k}) - \beta_k$,
    where $\alpha_k > 0$ and
    $\beta_k > 0$ are constants.
  \end{enumerate}
\end{restatable}

Observe that, using the primal-dual mapping
from Proposition~\ref{prop:gd-pd-map},
Part (i) of the lemma establishes the cycling
of the primal iterates by characterizing
the behavior of the corresponding dual iterates.
In particular, Part (i) shows that
$x^{t_k} = e_{i+k\;(\mod n)}$
and $x^{t_{k+1}} = e_{i+k+1\;(\mod n)}$,
and either $x^t = e_{i+k\;(\mod n)}$
for all $t_k \le t \le t_{k+1}$,
or $\supp(x^{t_{k+1} - 1}) = \{i+k\;(\mod n), i+k+1\;(\mod n)\}$.
In other words, the primal iterates either
transition directly from vertex $e_{i+k\;(\mod n)}$
to $e_{i+k+1\;(\mod n)}$,
or at most one iterate is spent on the simplex edge
connecting these vertices before transitioning
to $e_{i+k+1\;(\mod n)}$.
The constraint $\eta > \frac{2}{a_{\min}}$ ensures
at most a single iterate is spent on this edge. 

The proof of Lemma~\ref{lem:gd-cycling-pl-overview}
relies on several helper propositions,
which we develop in the following two subsections. 

\subsubsection{Establishing the Cycling Order}

Recall that for any $n$-dimensional RPS matrix,
for $x^t \in \Delta_n$, then
$\Delta y^t = y^{t+1} - Y^t = \eta A x^t$
with coordinates given by
\begin{equation}
  \Delta y^t_i
  \;=\;
  \eta \big(
  a_{i-1} x^t_{i-1}
  -
  a_i x^t_{i+1}
  \big) \;.
  \label{eq:gd-velocity}
\end{equation}
In the following two propositions, we establish the order
in which the dual iterates cycle through the regions 
from Proposition~\ref{def:gd-p-regions}.

\begin{restatable}{prop}{propgdcyclingpi}
  \label{prop:pi-cycling}
  Fix $i\in [n]$ and suppose
  $y^t \in P_i$.
  Assume $\eta > 2/a_{\min}$, 
  and let
  $\tau := \Big\lceil
  \frac{y^t_i - y^t_{i+1} - 1}{\eta \cdot a_i}\Big\rceil$.
  Then
  $y^{t+s} \in P_i$ for all $1 \le s \le \tau-1$.
  Moreover, at time $t+\tau$, we have either 
  (i) $y^{t+\tau} \in P_{i+1}$,
  or (ii) $y^{t+\tau} \in P_{i\sim(i+1)}$,
  and additionally: 
  \begin{itemize}
    \item 
     $y^{t+\tau}_i = y^{t}_i$
    \item
    $y^{t+\tau}_{i-1} < y^{t}_{i-1} - 2\tau$
    \item 
    $y^{t+\tau}_{i+1} - y^{t+\tau}_j > 1$
    for all $j \in [n] \setminus \{i, i+1\}$.
  \end{itemize}
\end{restatable}

\begin{proof}
  By Proposition~\ref{prop:gd-pd-map},
  if $y^{\ell} \in P_i$, then $x^{\ell} = e_i$. 
  Thus $\Delta y^{\ell} = y^{\ell+1}-y^\ell
  = \eta A x^{\ell} = \eta A_i$,
  which is a constant vector.
  Inductively (and following the
  proof of Proposition~\ref{prop:fp-rps-structure}
  in the Fictitious Play case),
  in order to prove $y^{t+s} \in P_i$
  for all $1 \le s \le \tau-1$, it
  suffices to show that
  \begin{equation}
    y^{t+s}_j
    \;=\;
    y^t_j + s \cdot \eta A_{ji}
    \;<\;
    y^{t}_i + s \cdot \eta A_{ii} - 1
    \;=\;
    y^{t}_i  - 1
    \quad
    \text{for all $j \neq i$}.
    \label{eq:p1-goal}
  \end{equation}
  For this, we can directly check using
  the definition of the entries of $A_i$ that
  \begin{equation}
    \begin{aligned}
      y^{t+s}_{i-1}
      &\;=\;
        y^t_{i-1} - s \cdot \eta a_{i-1} \\
      y^{t+s}_{i+1}
      &\;=\;
        y^t_{i+1} + s \cdot \eta a_i  \\
      \text{and}\quad
      y^{t+s}_j
      &\;=\;
        y^t_j
      &\text{for $j \in [n] \setminus \{i-1, i, i+1\}$}\;.
    \end{aligned}
      \label{eq:p1-1}
  \end{equation}
  Because $y^t \in P_i$,
  we have $y^t_i - y^t_{i-1} > 1$.
  Then using the positivity of $a_{i-1}$,
  we have at coordinate $i-1$ for all $s \ge 1$ that
  \begin{equation}
    y^{t+s}_{i-1} - y^{t+s}_i
    \;=\;
    y^t_{i-1} - y^t_i
    - s \cdot \eta a_{i-1}
    \;\le\;
    -1 - 2s\;,
    \label{eq:p1-2}
  \end{equation} 
  where the final inequality comes from
  the fact that $\eta > 2/a_{\min}$,
  and thus  $\eta a_{i-1} \ge 2$. 
  For coordinates $j \in [n] \setminus \{i, i+1\}$,
  as $y^t_i - y^t_j > 1$ by assumption of $y^t \in P_i$, we similarly
  find that $y^{t+s}_i - y^{t+s}_j > 1$
  for all $s \ge 1$.
  Thus~\eqref{eq:p1-goal} holds for both cases.
  Now at coordinate $i+1$, observe that 
 that 
  \begin{equation*}
    y^{t+s}_i - y^{t+s}_{i+1}
    \;\ge\;
    y^{t}_i - y^t_{i+1} - s \cdot \eta \cdot a_i 
    \;>\;
    y^{t}_i - y^t_{i+1} - (y^t_i - y^t_{i+1} - 1)
    \;=\; 1 
  \end{equation*}
  for all $1 \le s \le \tau-1 <
  \frac{y^t_i - y^t_{i+1} -1}{\eta \cdot a_i}$.
  Thus expression~\eqref{eq:p1-goal} also holds
  for coordinate $i+1$, which means
  $y^{t+s} \in P_i$ for all $1 \le s \le \tau-1$.
  
  To prove the second part of the proposition,
  observe that the first part establishes
  at time $t+\tau$ that
  $y^{t+\tau}_i = y^{t}_i + \tau \cdot \eta A_{ii}
  = y^t_i$.
  Also, from expression~\eqref{eq:p1-2}, we
  conclude that
  $y^{t+\tau}_{i-1} < y^{t}_{i-1} - 2\tau$. 
  Moreover, at time $t+\tau$ we also have
  \begin{align*}
    y^{t+\tau}_i - y^{t+\tau}_{i+1}
    &\;\le\;
      y^{t}_i - y^t_{i+1} - \tau \cdot \eta \cdot a_i \\
    &\;\le\;
      y^{t}_i - y^t_{i+1} - (y^t_i - y^t_{i+1} - 1)
      \;=\; 1 \;.
  \end{align*}
  If  $y^{t+\tau}_i - y^{t+\tau}_{i+1} < -1$,
  then it follows that from the previous arguments
  that also $y^{t+\tau}_{i+1} - y^{t+\tau}_j > 1$
  for all $j \in [n] \setminus \{i, i+1\}$.
  This establishes that $y^{t+\tau} \in P_{i+1}$.

  If instead $1 \ge y^{t+\tau}_i - y^{t+\tau}_{i+1} \ge -1$,
  then it still remains that
  $y^{t+\tau}_i - y^{t+\tau}_j \ge 1$
  for all other $j \in [n] \setminus \{i, i+1\}$.
  In this case, as $y^{t+\tau}_{i+1} \ge y^{t+\tau}_i - 1$,
  we can then further verify for all such $j$ that
  $\frac{y^{t+\tau}_i + y^{t+\tau}_{i+1}}{2} - y^{t+\tau}_j
  \ge \frac{1}{2}$,
  and thus $y^{t+\tau} \in P_{i\sim(i+1)}$.
\end{proof}

\begin{restatable}{prop}{propgdcyclingpisim}
  \label{prop:pisim-cycling}
  Fix $i\in [n]$ and suppose
  $y^t \in P_{i\sim(i+1)}$ such that  
  $y^{t+\tau}_{i+1} - y^{t+\tau}_j > 1$
  for all $j \in [n] \setminus \{i, i+1\}$.
  Assume $\eta > 2/a_{\min}$.
  Then at time $t+1$,
  we have $y^{t+1}_{i-1} < y^t_{i-1} - 2c$
  for some positive constant $c> 0$,
  and one of three scenarios occurs:
  if 
  $y^t_{i+1} - y^t_{i+2} \ge 1 + \eta a_{i+1}$,
  then (i) $y^{t+1} \in P_{i+1}$. 
  Otherwise, either:
  (ii) $y^{t+1} \in P_{i+2}$ 
  or (iii) $y^{t+1} \in P_{(i+1)\sim(i+2)}$.
\end{restatable}

\begin{proof}
  First, recall from Proposition~\ref{prop:gd-pd-map}
  that since $y^t \in P_{i\sim(i+1)}$,
  we have $\supp(x^t) = \{i, i+1\}$.
  Then the coordinates of
  $\Delta y^t = y^{t+1} - y^t_i = \eta Ax^t$
  are specified by
  \begin{equation}
    \begin{aligned}
      &\Delta y^{t}_{i-1}
      &=\;&
        - \eta \cdot a_{i-1} \cdot x^{t}_i \\
      &\Delta y^{t}_{i}
        &=\;&
        - \eta \cdot a_i \cdot x^{t}_{i+1} \\
      &\Delta y^{t}_{i+1}
      &=\;&
        + \eta \cdot a_i \cdot x^{t}_{i} \\
      &\Delta y^{t}_{i+2}
      &=\;&
            + \eta \cdot a_{i+1} \cdot x^{t}_{i+1}  \\
      &\Delta y^{t}_{j}
      &=\;& 0
            \quad\text{for all other $j$}
            \;\;.
    \end{aligned}
    \label{eq:psim-veloc}
  \end{equation}
  We start by showing the behavior of the differences
  $y^{t+1}_{i+1} - y^t_j$ for all
  $j \in [n] \setminus\{i+1, i+2\}$.
  
  \begin{itemize}
  \item
    At coordinate $i$, observe that:
    \begin{equation*}
      y^{t+1}_{i+1} - y^{t+1}_i
      \;=\;
      y^t_{i+1} - y^t_i
      + \eta a_i (x^t_{i+1} + x^t_i)
      \;\ge\;
      y^t_{i+1} - y^t_i + 2 \;,
    \end{equation*}
    where the inequality follows from
    the assumptions that $\eta > 2/a_{\min}$
    and $\supp(x^t) = \{i, i+1\}$. 
    As $y^{t}_{i+1} - y^t_i \ge -1$ by
    assumption of $y^t \in P_{i\sim(i+1)}$, we have
    $y^{t+1}_{i+1} - y^{t+1}_i > 1$.
    Moreover, if $y^t_{i+1} - y^t_i > 0$,
    it follows that $y^{t+1}_{i+1} \ge  y^t_i$.
    Similarly, if $-1 \le y^t_{i+1} - y^t_i < 0$,
    then we still have $y^{t+1}_{i+1} > y^t_{i+1} + 1$,
    and thus also $y^{t+1}_{i+1} \ge y^t_i$.
  \item
    At coordinate $i-1$, we have
    $y^{t+1}_{i-1}
    = y^t_{i-1} - \eta a_{i-1} x^t_i
    <  y^t_{i-1} - 2 c$,
    where $c = x^t_i > 0$ is some
    positive constant.
    Moreover, the inequality holds given
    $\eta > 2/a_{\min}$.
    Note that this proves the first claim of the
    proposition. 
    
    Observe further that we have
    $y^{t+1}_{i+1} - y^{t+1}_{i-1}
    \ge y^t_{i} - y^{t}_{i-1} + 
    2c > 1$,
    where, the final inequality comes
    from the assumption that $y^t_i - y^t_{i-1} \ge 1$.
  \item
    For all other $j \neq i+2$, we have 
    we have $\Delta y^t_j = 0$. It follows that
    $y^{t+1}_{i+1} - y^{t+1}_j
    \ge y^t_i - y^t_j > 1$.
  \end{itemize}
  We now examine the difference
  $y^{t+1}_{i+1} - y^{t+1}_{i+2}$.
  Observe from expression~\eqref{eq:psim-veloc} that
  \begin{align*}
    y^{t+1}_{i+1} - y^{t+1}_{i+2}
    &\;=\;
      y^t_{i+1} - y^t_{i+2}
      + \eta a_i x^t_{i}
      - \eta a_{i+1} x^t_{i+1} \\
    &\;=\;
      y^t_{i+1} - y^t_{i+2}
      + \eta a_i (1-x^t_{i+1})
      - \eta a_{i+1} x^t_{i+1} \\
    &\;\ge\;
      y^t_{i+1} - y^t_{i+2}
    - \eta a_{i+1} \;,
  \end{align*}
  where the inequality holds given
  that $x^t_{i+1} \le 1$.
  If $y^t_{i+1} - y^t_{i+2} > 1 + \eta a_{i+1}$,
  then $y^{t+1}_{i+1} - y^t_{i+2} > 1$,
  and claim $(i)$ of the proposition holds by combining
  the preceding arguments.

  If on the other hand
  $1 \ge y^{t+1}_{i+1} - y^{t+1}_{i+2}$, then either
  $y^{t+1}_{i+2} - y^{t+1}_{i+1} > 1$,
  and claim $(ii)$ of the proposition
  holds by combining the preceding arguments,
  or $1 \ge y^{t+1}_{i+2} - y^{t+1}_{i+1} \ge -1$,
  and claim $(iii)$ of the proposition holds.
\end{proof}

\subsubsection{Energy Growth Bounds}

In the following lemma, we establish lower and upper
bounds on the energy growth when the dual iterates
move between the regions from Definition~\ref{def:gd-p-regions}:

\begin{lemma}
  \label{lem:gd-energy-growth}
  Fix $i \in [n]$, and let
  $\Delta \phistar(y^t) = \phistar(y^{t+1}) - \phistar(y^t)$.
  Then the following hold:
  \begin{enumerate}[
    label={(\roman*)}, 
    leftmargin=4em,
    ]
  \item
    If $y^t \in P_i$ and $y^{t+1} \in P_i$,
    then $\Delta \phistar(y^t) = 0$.
  \item
    If $y^t \in P_i$ and $y^{t+1} \in P_{i+1}$,
    then $1 < \Delta \phistar(y^t) < \eta a_{\max}$.
  \item
    If $y^t \in P_i$ and $y^{t+1} \in P_{i\sim(i+1)}$,
    then $0 \le \Delta \phistar(y^t) \le 1$.
  \item
    If $y^t \in P_{i\sim(i+1)}$
    and $y^{t+1} \in P_{i+1}$ or $y^{t+1} \in P_{i+2}$,
    then
    $0 \le \Delta \phistar(y^t) \le \frac{1}{4}(\eta a_{\max})^2$.
  \item
    If $y^t \in P_{i\sim(i+1)}$
    and $y^{t+1} \in P_{(i+1)\sim(i+2)}$ then
    $0 \le \Delta \phistar(y^t) \le \eta a_{\max} + \frac{5}{4}$. 
  \end{enumerate}
  Moreover, in cases (iii) through (v),
  if $\phistar(y^t) > 0$, then there exists an absolute
  constant $c > 0$ such that $\phistar(y^t) \ge c$. 
\end{lemma}

\begin{proof}
  We prove each part independently:
  \begin{itemize}[leftmargin=1em]
  \item
    \textit{Part (i)}:
    By Proposition~\ref{prop:gd-pd-map},
    $x^t = x^{t+1} = e_i$,
    and thus by Proposition~\ref{prop:gd-closed}:
    $\Delta \phistar(y^t) = y^{t+1}_i - y^t_i$. 
    As $y^{t+1}_i = y^t_i + \eta (Ae_i)_i = y^t_i$,
    we then have $\Delta \phistar(y^t) = 0$. 
  \item
    \textit{Part (ii)}: 
    Using Proposition~\ref{prop:gd-closed}
    and  expression~\eqref{eq:psim-veloc}, for
    $y^t \in P_i$, we have
    $\phistar(y^{t+1}) = y^{t+1}_{i+1} - \frac{1}{2}
    = y^t_{i+1} + \eta a_{i} + \frac{1}{2}$,
    and $\phistar(y^{t+1}) = y^{t+1}_i - \frac{1}{2}
    = y^t_i - \frac{1}{2}$.
    Thus
    \begin{equation*}
      \Delta \phistar(y^t)
      =
      y^t_{i+1} - y^t_i + \eta a_{i}
      <
      -1 + \eta a_{i+1} < \eta a_{i} \;,
    \end{equation*}
    where we use the fact that $y^t_{i+1} - y^t_i < -1$
    for $y^t \in P_i$.
    On the other hand, $y^{t+1}_{i+1} - y^{t+1}_i > 1$
    by definition of $P_{i+1}$,
    and thus $\Delta \phistar(y^t) > 1$. 
   
  \item
    \textit{Part (iii)}:
    As $y^t \in P_i$, we have $\Delta y^t_i = 0$,
    and thus $y^{t+1}_i = y^t_i$.
    Then by
    Proposition~\ref{prop:gd-closed}, the energy 
    at times $t$ and $t+1$ are given by
    \begin{align*}
      \phistar(y^{t+1})
      &\;=\;
        \frac{1}{4}(y^{t+1}_i - y^{t+1}_{i+1})^2
        + \frac{1}{2}(y^{t+1}_{i+1} + y^{t+1}_i)
        - \frac{1}{4} \\
      \phistar(y^t)
      &\;=\;
        y^t_i - \frac{1}{2}
        =
        y^{t+1}_i - \frac{1}{2} \\
      \text{and}\quad
      \Delta \phistar(y^{t})
      &\;=\;
      \frac{1}{4}(y^{t+1}_i - y^{t+1}_{i+1})^2
      + \frac{1}{2}(y^{t+1}_{i+1} - y^{t+1}_i)
      + \frac{1}{4} \;.
    \end{align*}
    By definition of $P_{i\sim(i+1)}$,
    we have $|y^{t+1}_i - y^{t+1}_{i+1}| \le 1$,
    and thus by convexity, we have
    $\Delta \phistar(y^{t+1}) \le 1$.
    If $y^{t+1}_i - y^{t+1}_{i+1} = 0$,
    then $\Delta \phistar(y^{t+1}) = 0$.
    On the other hand, recall that
    $y^{t+1}_{i+1} = y^t_{i+1} + 
    \Delta y^t_{i+1}$,
    and that $\Delta y^{t}_{i+1} = \eta a_i$
    is an absolute constant.
    Thus if
    $\Delta \phistar(y^{t+1}) > 0$,
    then we must have $\Delta \phistar(y^{t+1}) \ge c$
    for some absolute constant $c > 0$.
  \item
    \textit{Part (iv)}:
    By Proposition~\ref{prop:gd-closed}, the energy 
    at times $t$ and $t+1$ are given by
    \begin{align*}
      \phistar(y^{t+1})
      &\;=\;
        y^{t+1}_{i+1} - \frac{1}{2} \\
      \phistar(y^{t})
      &\;=\;
        \frac{1}{4}(y^{t}_i - y^{t}_{i+1})^2
        + \frac{1}{2}(y^{t}_{i+1} + y^{t}_i)
        - \frac{1}{4} \;.
    \end{align*}
    As $y^t \in P_{i \sim (i+1)}$, by
    expresion~\eqref{eq:gd-velocity} we have
    $y^{t+1}_{i+1} = y^t_{i+1} + \eta a_i x^t_i$.
    Then using the closed-form expression of $x^t_i$
    in Proposition~\ref{prop:gd-closed},
    we have $x^t_i = \frac{1}{2}(y^t_i - y^t_{i+1} + 1)$.
    It follows that
    \begin{equation*}
      \Delta \phistar(y^t)
      \;=\;
      \frac{1}{2}(1-\eta a_i)(y^{t}_{i+1} -y^t_i)
      -
      \frac{1}{4}(y^t_i - y^t_{i+1})^2
      + \frac{\eta a_i}{2} - \frac{1}{4} \;.
    \end{equation*}
    By convexity, we find
    $\Delta \phistar(y^t) \le \frac{1}{4}(\eta a_i)^2$.
    If $y^t_{i+1} - y^t_i = 1$,
    then $\Delta \phistar(y^t) = 0$.
    On the other hand,
    if $-1 \le y^t_{i+1} - y^t_i < 1$,
    then $\Delta \phistar(y^t) > 0$,
    and the there must exist an aboslute constant
    $c > 0$ such that
    $\Delta \phistar(y^t) \ge c$. 

    If $y^{t+1} \in P_{i+2}$, we find by
    similar calculations that
    $\Delta \phistar(y^{t}) \le \frac{1}{4}(\eta a_{i+1})^2$,
    and the same arguments for the lower bound
    also apply. 
  \item
    \textit{Part (iv)}:
    As $y^t \in P_{i\sim(i+1)}$ 
    we have from expression~\eqref{eq:gd-velocity} that
    $y^{t+1}_{i+1} - y^t_{i+1} = \eta a_i x^t_i \le \eta a_i$
    and
    $y^{t+1}_{i+2} - y^t_{i+2} = \eta a_{i+1} x^t_{i+1}
    \le \eta a_{i+1}$.
    Then using the bounds from Lemma~\eqref{lem:gd-energy-max},
    we have
    $\phistar(y^{t+1}) \le y^{t+1}_{i+1} + \frac{1}{2}$
    and
    $\phistar(y^{t}) \ge y^{t}_{i+1} - \frac{3}{4}$,
    and thus 
    $\Delta \phistar(y^t)
    \le y^{t+1}_{i+1} - y^t_{i+1} + \frac{5}{4}
    \le \eta a_i + \frac{5}{4}$.
    Recall from Proposition
    \ref{prop:skew-gradient-descent}
    that in general $\Delta \phistar(y^t) \ge 0$.
    On the other hand, as $\Delta y^{t}_{i+1}$
    and $\Delta y^t_i$ are absolute constants,
    if $\Delta \phistar(y^t) > 0$, then
    there must be an absolute constant $c \ge 0$
    such that $\Delta \phistar(y^t) \ge c$. 
  \end{itemize}
\end{proof}

\subsubsection{Proof of Lemma~\ref{lem:gd-cycling-pl-overview}}

We now prove each of the three parts of
Lemma~\ref{lem:gd-cycling-pl-overview}: 

\medskip

\begin{proof}{\textbf{of Part (i).}}
  By the convergence property of Lemma~\ref{lem:gd-large-initial},
  if $x^0 \in \Delta_n$ satisfies $\gamma(x^0) > 0$,
  then $x^1 = e_i$  for some $i \in [n]$,
  By the equivalence from Proposition~\ref{prop:gd-pd-map},
  $y^1 \in P_i$.
  Moreover, we must have $\phi^*(y^1) = \Theta(1)$.  
  Now by the cycling order established in 
  Propositions~\ref{prop:pi-cycling}
  and~\ref{prop:pisim-cycling},
  for any $i \in [n]$,  the dual iterates are restricted 
  to transitioning only between the regions of
  Definition~\ref{def:gd-p-regions} in the following manner:
  \begin{align*}
    &\text{(a)}\;P_{i} \to P_{i+1}
    \qquad\qquad\qquad
    \text{(b)}\;P_i \to P_{i\sim(i+1)}
    \qquad\qquad
    \text{(c)}\;P_{i\sim(i+1)} \to P_{(i+1)} \\
    &\text{(d)}\;P_{i\sim(i+1)} \to P_{i+2} 
    \qquad\quad\;\;
    \text{(e)}P_{i\sim(i+1)} \to P_{(i+1)\sim(i+2)} \;.
  \end{align*}
  We show that there exists $t_0$
  with $t_0 \le O(n)$, such that
  for all $t \ge t_0$, transitions
  (d) and (e) never occur.
  Noting the structure of the transitions (a),
  (b), and (c), and recalling the primal-dual map
  of Proposition~\ref{prop:pisim-cycling},
  observe that this is sufficient for proving
  part (i) of the lemma. 
  
  Now assume $y^t \in P_{i\sim(i+1)}$,
  and observe by the statement of
  Proposition~\ref{prop:pisim-cycling}
  that transitions (d) and (e) cannot occur
  when $y^t_{i+1} - y^t_{i+2} \ge 1 + \eta a_{i+1}$.
  By Lemma~\ref{lem:gd-energy-max}, 
  note also that if $y^t \in P_{i\sim(i+1)}$,
  then $y^t_{i+1} \ge \phi^*(y^t) - \frac{1}{2}$.
  Moreover, among the types of transitions
  that can occur according to
  Propositions~\ref{prop:pi-cycling}
  and~\ref{prop:pisim-cycling},
  and using the velocity definition from
  expression~\eqref{eq:psim-veloc},
  observe that coordinate $(i+2)$ of the
  dual iterate can increase only in
  the regions $P_{(i)\sim(i+1)}, P_{i+1}$,
  and $P_{(i+1)\sim(i+2)}$,
  and it must be decreasing
  in the regions $P_{(i+2)\sim(i+3)}, P_{i+3}$,
  and $P_{(i+3)\sim(i+4)}$.
  Due to the transition orders of Propositions~\ref{prop:pi-cycling}
  and~\ref{prop:pisim-cycling},
  it follows that between the most recent
  time coordinate $(i+2)$ has decreased
  and the next time the coordinate increases,
  at least $\Omega(n)$ rounds have elapsed.

  We assume without loss of generality
  that within these $\Omega(n)$ rounds,
  the energy $\phi^*(y^t)$ is strictly increasing when
  transitioning between phases
  (otherwise, the energy remains constant for these rounds, and
  a similar argument can be applied once the energy has sufficiently increased).
  By Proposition~\ref{lem:gd-energy-growth},
  the energy increase between each region is
  an absolute constant, and thus
  the total energy of the dual iterates increases
  by at least $\Omega(n)$ before the
  next time coordinate $(i+2)$ increases.
  
  Putting these arguments together, suppose
  $y^t \in P_{i\sim(i+1)}$, where $1 \le t = cn$ 
  for some constant $c \ge 1$.
  Then for sufficiently large $c$, we must have
  $y^t_{i+2} \le d \cdot \phi^*(y^t)$ for some
  absolute constant $0 < d < 1$.
  Recalling the lower
  bound on $y^{t}_{i+1}$ from above, it follows that
  \begin{equation*}
    y^t_{i+1} - y^t_{i+2}
    \;\ge\;
    \phi^*(y^t) - \frac{1}{2} - d \cdot \phi^*(y^t)
    \;\ge\;
    \phi^*(y^t)(1-d) - \frac{1}{2} \;.
  \end{equation*}
  As we assumed that the energy has strictly increased
  by an absolute constant at each subsequent time
  the dual iterate re-enters $P_{i\sim(i+1)}$,
  the energy $\phistar(y^t)$ is strictly increasing with $c$.
  As $\eta a_{i+1}$ is a fixed absolute constant,
  then for $c$ sufficiently large, we have
  \begin{equation*}
    \phi^*(y^t)(1-d) - \frac{1}{2} > 1 + \eta a_{i+1} \;.
  \end{equation*}
  Proposition~\ref{prop:pi-cycling}
  further implies that upon the next time
  $t' > t$ where $y^{t'}$ re-enters $P_{i\sim(i+1)}$,
  then also $y^t_{i+1} - y^t_{i+2} > 1 + \eta a_{i+1}$. 
  Repeating this argument for all
  $i \in [n]$, we set $t_0$ as the maximum time
  among all such $t$ used.
  It follows that for any $i \in [n]$,
  and for any subsequent
  $t' > t_0$ where $y^{t'} \in P_{i\sim(i+1)}$,
  we must have $y^{t'}_{i+1} -y^{t'}_{i+2} \ge
  1+ \eta a_{i+1}$, and thus
  transitions of types (d) and (e) cease to occur.
\end{proof}

\begin{proof}{\textbf{of Part (ii).}}
  Using the cycling order from Part (i) of the lemma,
  recall that  between times $t_k$ and $t_{k+1}$, the dual iterates
  transition either from $P_{i+k\;(\mod n)}$ to 
  $P_{i+k+1\;(\mod n)}$ directly, 
  or from $P_{i+k\;(\mod n)}$ to  
  $P_{(i+k\;(\mod n))\sim(i+k+1\;(\mod n))}$
  to 
  $P_{i+k+1\;(\mod n)}$. 
  For consecutive iterates within the region
  $P_{i+k+1\;(\mod n)}$, the energy growth is zero.  
  Moreover, recall from Part (i) of the lemma
  that the dual iterate spends at most a single round
  in $P_{(i+k\;(\mod n))\sim(i+k+1\;(\mod n))}$.
  Then using the cases of Lemma~\ref{lem:gd-energy-growth}, 
  we can bound in the worst case
  \begin{equation*}
    \phistar(y^{t_{k+1}})
    -
    \phistar(y^{t_k})
    \;\le\;
    \eta a_{\max}
    +
    1
    +
    \frac{1}{4}(\eta a_{\max})^2 \;,
  \end{equation*}
  which is at most an absolute constant. 
\end{proof}

\begin{proof}{\textbf{of Part (iii).}}
Observe Part (i) of the lemma establishes
that from time $t_0$ onward, the dual iterates
perpetually cycle through the regions $\{P_k\}$
in the order
\begin{equation*}
  \dots P_i \to P_{i+1} \to \dots \to P_{n} \to P_1 \dots \;.
\end{equation*}
Moreover, between any two consecutive regions
$P_i$ and $P_{i+1}$, the only other
region the dual iterate ever enters
is $P_{i~\sim(i+1)}$.
From Proposition~\ref{prop:pisim-cycling},
under the large stepsize assumptions,
the dual iterate spends at most a single timestep
in any such  $P_{i\sim(i+1)}$.
Thus $\tau_k$ is at least as large
as the number of steps spent in $P_i$
before exiting.

Now for any $\ell \ge 0$, Part (i) of the
lemma establishes that
$y^{t_\ell} \in P_{i+\ell\;(\mod n)}$.
Then by Proposition~\ref{prop:pi-cycling},
observe that
\begin{equation}
  \tau_\ell
  \;\ge\;
  \frac{y^{t_\ell}_{i+\ell\;(\mod n)}
    - y^{t_\ell}_{i+\ell+1\;(\mod n)} - 1}
  {\eta \cdot a_{i+\ell\;(\mod n)}} \;.
\end{equation}
Recall further from Lemma~\ref{lem:gd-energy-max} that
if $y^t \in P_i$, then $y^t_i \ge \phi^*(y^t)$
Thus we can further write
\begin{equation}
  \tau_\ell
  \;\ge\;
  \frac{\phi^*(y^{t_\ell})
    - y^{t_\ell}_{i+\ell+1\;(\mod n)}}
  {\eta \cdot a_{i+\ell\;(\mod n)}}
  -
  \frac{1}{\eta \cdot a_{i+\ell\;(\mod n)}}
  \;.
  \label{eq:p3-1}
\end{equation}
As the second term in~\eqref{eq:p3-1}
is a positive absolute constant, it suffices
to show that
\begin{equation}
  \frac{\phi^*(y^{t_\ell})
    - y^{t_\ell}_{i+\ell+1\;(\mod n)}}
  {\eta \cdot a_{i+\ell\;(\mod n)}}
  \;\ge\;
  \alpha'_\ell \cdot \phi^*(y^{t_\ell})
  - \beta'_\ell
  \label{eq:p3-2}
\end{equation}
for positive constants $\alpha'_\ell > 0$
and $\beta'_\ell > 0$
for all $\ell = 0, \dots, K$.

We prove this claim by induction following
similarly to the proof of Lemma~\ref{lem:phase-length}
for Fictitious Play. 
The key property to note is that,
for any fixed coordinate $i$,
following the most recent time the dual
iterate was in the region $P_{i+1}$,
and until the dual iterate again enters
either $P_{(i-2)\sim(i-1)}$ or $P_{(i-1)}$, 
the velocity $\Delta y^t_i = 0$ at 
all intermediate times $t$.
Together, with Propositions~\ref{prop:pi-cycling}
and~\ref{prop:pisim-cycling},
this key property implies that
within the first $\ell = 0, \dots, n$ phases, 
at least $\Omega(n)$ rounds
have elapsed between the most recent time prior 
to $t_0$ that each coordinate $i$ was \textit{decreasing},
and the next time it \textit{increases}
within the first $n$ phases.
Moreover, $t_0 \le O(n)$ by definition,
and along with the energy growth bounds of
Lemma~\ref{lem:gd-energy-growth}
we have $\phi^*(y^0) = O(n)$.
By a similar argument
as in Part (i) of the lemma, 
for each $\ell = 0, \dots, n$,
it follows that 
$y^{t_\ell}_{i+\ell+1\;(\mod n)}
\le b_\ell \cdot \phi^*(y^{t_\ell})$
for some positive constant $0 < b_\ell < 1$. 
Thus expression~\eqref{eq:p3-2} holds
for these initial $n$ phases.

Now assume the claim holds through Phase $k-1$.
To prove the claim also holds for Phase $k$, we can use
the same calculations as in the proof of 
Lemma~\ref{lem:phase-length} for Fictitious Play.
Specifically, using the cycling property
of Part (i) of the lemma:
$y^{t_k}_{i+k+1\;(\mod n)}
\le y^{t_{k+3-n}}_{i+k+1\;(\mod n)} + d_k$
for some constant $d_k > 0$.
Moreover, by the energy growth bound of
Part (ii) and also Lemma~\ref{lem:gd-energy-max}, 
we have 
$y^{t_{k+2-n}}_{i+k+1\;(\mod n)} \le \phi^*(y^{t_{k+2-n}})
+ \rho$, where $\rho >0$ is an absolute constant.
Then following identical calculations as in expressions
\eqref{eq:pl-1} through~\eqref{eq:fp-pl-end}
in the proof of Lemma~\ref{lem:phase-length},
we conclude that
\begin{align*}
  \frac{\phi^*(y^{t_k})
    - y^{t_k}_{i+k+1\;(\mod n)}}
  {\eta \cdot a_{i+k\;(\mod n)}}
  &\;\ge\;
    \frac{\phi^*(y^{t_k}) \cdot a_{i+k+1\;(\mod n)}
    \cdot \alpha_{k+2-n}}{
     a_{i+k\;(\mod n)}}
    - \frac{\beta_{k+2-n} + \rho + d_k}
    {\eta \cdot a_{i+k\;(\mod n)}} \\
  &\;=\;
    \alpha'_k \cdot \phi^*(y^{t_k}) - \beta'_k
\end{align*}
for positive constants $\alpha'_k$
and $\beta'_k$. By expression~\eqref{eq:p3-2},
this proves the claim. 
\end{proof}

\subsection{Proof of Theorem~\ref{thm:gd-large-stepsize}}
\label{app:gd:large:regret}

Using the machinery of Lemma~\ref{lem:gd-cycling-pl-overview},
we prove the regret bound of Theorem~\ref{thm:gd-large-stepsize}:

\medskip

\begin{proof}
  The proof follows similarly to that of 
  Theorem~\ref{thm:fp-rps} for Fictitious Play.
  First, let $x^0 \in \Delta_n$ be such that $\gamma(x^0) > 0$
  (for $\gamma$ as defined in expression~\eqref{eq:gamma-main}).
  As established in the proof of Lemma~\ref{lem:gd-large-initial},
  the set of points $x \in \Delta_n$ with $\gamma(x) > 0$
  has full Lebesgue measure. 
  Then under the large step size assumptions of
  the theorem, $x^1 = e_i$ for some $i \in [n]$,
  and the cycling, energy growth,
  and phase length bounds of
  Lemma~\ref{lem:gd-cycling-pl-overview} apply. 

  Recall by Proposition~\ref{prop:regret-energy-fp}
  that for constant $\eta > 0$, 
  to bound the regret $\reg(T)$, it suffices to
  obtain an upper bound on the energy $\phistar(y^{T+1})$.
  By Lemma~\ref{lem:gd-cycling-pl-overview},
  the energy $\phistar$ increases by at most
  a constant between each phase $k = 0, \dots, K$.
  On the other hand, depending on the magnitude of $\eta$,
  and due to certain boundary cases between regions 
  $P_i$ and $P_{i~\sim(i+1)}$, the energy may not be
  \textit{strictly} increasing between every phase
  (e.g., see cases (iii), (iv) of
  Lemma~\ref{lem:gd-energy-growth}).
  Thus similar to the proof of Theorem~\ref{thm:fp-rps},
  for each phase $k=0, 1, \dots, K$, we define the
  indicator variable $c_k \in \{0, 1\}$ as follows:
  \begin{equation*}
    c_k
    \;:=\;
    \begin{cases}
      0 &\text{if $\phistar(y^{t_k}) = \phistar(y^{t_{k-1}})$} \\
      1 &\text{if $\phistar(y^{t_k}) > \phistar(y^{t_{k-1}})$} 
    \end{cases} \;.
  \end{equation*}
  Lemmas~\ref{lem:gd-energy-growth}
  and~\ref{lem:gd-cycling-pl-overview} imply that
  if $c_k = 1$, then
  $\phistar(y^{t_k})  - \phistar(y^{t_{k-1}}) = \Theta(1)
  = \Theta(c_k)$. 
  Thus it holds that
  \begin{equation}
    \phistar(y^{T+1})
    =
    \Theta\Big(\sum_{k=1}^K c_k\Big) \;.
    \label{eq:gd-reg-1}
  \end{equation}

  Now recall that for readability we use the
  notation $f \lesssim g$
  and $f \gtrsim g$ to
  indicate $f = O(g)$ or $f = \Omega(g)$ respectively.
  Then using the phase length bound in
  Part (iii) of Lemma~\ref{lem:gd-cycling-pl-overview},
  together with the fact that
  $\phistar(y^{t_k})  - \phistar(y^{t_{k-1}}) = \Theta(c_k)$, 
  we have
  \begin{equation}
    \tau_k
    \;\ge\;
    \alpha_k \cdot \phistar(y^{t_k}) - \beta_k
    \gtrsim \sum_{i=1}^k c_i - \beta_k 
  \end{equation}
  for positive constants $\alpha_k, \beta_k > 0$
  
  From here, we repeat identical calculations
  as in expressions~\eqref{eq:regret-1} through
  \eqref{eq:fp-regret-end} from the proof of
  Theorem~\ref{thm:fp-rps}. 
  As a consequence, we similarly obtain
  \begin{equation*}
    \Big(\sum_{k=0}^K c_k\Big)^2
    \lesssim T
    \quad\implies\quad
    \sum_{k=0}^K c_k
    \lesssim \sqrt{T} \;.
  \end{equation*}
  By expression~\eqref{eq:gd-reg-1},
  this in turn gives $\phistar(y^{T+1}) \lesssim O(\sqrt{T})$.
  By Proposition~\ref{prop:regret-energy-ftrl},
  we then conclude that
  \begin{equation*}
    \reg(T)
    \;\le\;
    \frac{2 \phistar(y^{T+1})}{\eta} + \frac{2M}{\eta}
    \;\le\;
    O(\sqrt{T})\;.
  \end{equation*}
  Here, the last inequality comes from the fact that
  $M = \max_{x \in \Delta_n} \frac{\|x\|^2_2}{2} \le 1$
  and $\eta$ is an absolute constant. 
\end{proof}

\section{Boundary Behavior of GD in Small Stepsize Regime}
\label{app:gd:small}

\subsection{Overview of Results}

In this section, we present several auxiliary results
related to the behavior of Gradient Descent under
different stepsize regimes. We start by presenting
an overview of these results:
First, recall that in the large stepsize regime,
Lemmas~\ref{lem:gd-large-initial} and \ref{lem:gd-cycling-pl-overview}
together imply that after at most a single iteration, 
every primal iterate of Gradient Descent will lie on the boundary of $\Delta_n$.
In the following theorem, we show that a similar \textit{boundary invariance}
property holds for any $\eta > 0$: if the energy ever exceeds a
(game-dependent) constant value, every subsequent primal iterate
of Gradient Descent must lie on the boundary of $\Delta_n$. 

\begin{restatable}{theorem}{gdrpsmideta}
  \label{thm:gda-rps-midstepsize}
  Let $\{x^t\}$ and $\{y^t\}$ be the iterates
  of~\eqref{eq:gd-pd} on an $n$-dimensional RPS 
  matrix $A$ with $\eta > 0$, 
  and from any initial $x^0 \in \Delta_n$. 
  Then there exists a constant $D$ (depending on 
  the entries of $A$) such that,
  if the energy $\phi^*(y^{t'}) > D$ for some
  $t' \ge 0$, then $x^t$ is on the boundary of $\Delta_n$
  at all subsequent iterations $t > t'$. 
\end{restatable}

\noindent
Theorem~\ref{thm:gda-rps-midstepsize}
extends the boundary invariance result of~\cite{bailey2019fast} 
for $2\times2$ games to the present $n$-dimensional
symmetric setting.
Note that while boundary invariance was the key driver in establishing
$O(\sqrt{T})$ regret in the large stepsize regime
(specifically, the cycling between vertices of $\Delta_n$
as seen in the right plot of Figure~\ref{fig:gd-eta-compare}),
this property alone appears insufficient for 
establishing sublinear regret bounds using smaller 
constant stepsizes:
as seen in the middle plot of Figure~\ref{fig:gd-eta-compare},
even for moderately small constant $\eta$, 
the primal iterates of Gradient Descent may transition between
different boundary faces of $\Delta_n$ in an irregular manner,
making the energy growth harder to control. 
Because of this, for smaller constant $\eta$,
simply ensuring that the primal iterates 
are on the boundary of $\Delta_n$ is insufficient for 
establishing a similar phase behavior
and regret guarantee as in 
Lemma~\ref{lem:gd-cycling-pl-overview}
and Theorem~\ref{thm:gd-large-stepsize}. 
Obtaining $O(\sqrt{T})$ regret bounds for Gradient Descent 
using \textit{any} constant stepsize thus remains open.

At the other extreme, while using a small 
time-horizon-dependent stepsize
(e.g., with $\eta = \Theta(1/\sqrt{T})$)
might result in all primal iterates remaining \textit{interior} 
(as seen in the left plot of Figure~\ref{fig:gd-eta-compare}), 
note that in this case, Proposition~\ref{prop:regret-energy-ftrl} 
(which establishes $\reg(T) \le O(\phistar(y^{T+1})/ \eta)$)
still only guarantees a worst case regret bound scaling with $\sqrt{T}$. In particular, we prove the following guarantee:

\begin{restatable}{lem}{gdsmallsteps}
  \label{lem:gd-smallsteps}
  Let $\{x^t\}$ and $\{y^t\}$ be the iterates
  of~\eqref{eq:gd-pd} on an $n$-dimensional RPS 
  matrix $A$ with $\eta = \frac{1}{\sqrt{T}}$, 
  and from an interior $x^0 \in \Delta_n$. 
  Then if $x^t$ is interior for all $t \ge 1$,
  then 
  \begin{equation*}
    \phistar(y^{T+1}) \le \frac{n a_{\max}^2}{2}
    \quad\text{and}\quad
    \reg(T) \le \sqrt{T} \cdot \Big(\frac{n a_{\max}^2}{2} + 1\Big) \;.
  \end{equation*}
\end{restatable}

In the following subsections, we develop
the proofs of Theorem~\ref{thm:gda-rps-midstepsize} 
and Lemma~\ref{lem:gd-smallsteps}, but we begin
by first establishing a certain low-dimensional representation
of the dual iterates of Gradient Descent on RPS games.

\begin{figure}[tb!]
  \centering
  \includegraphics[width=0.8\textwidth]{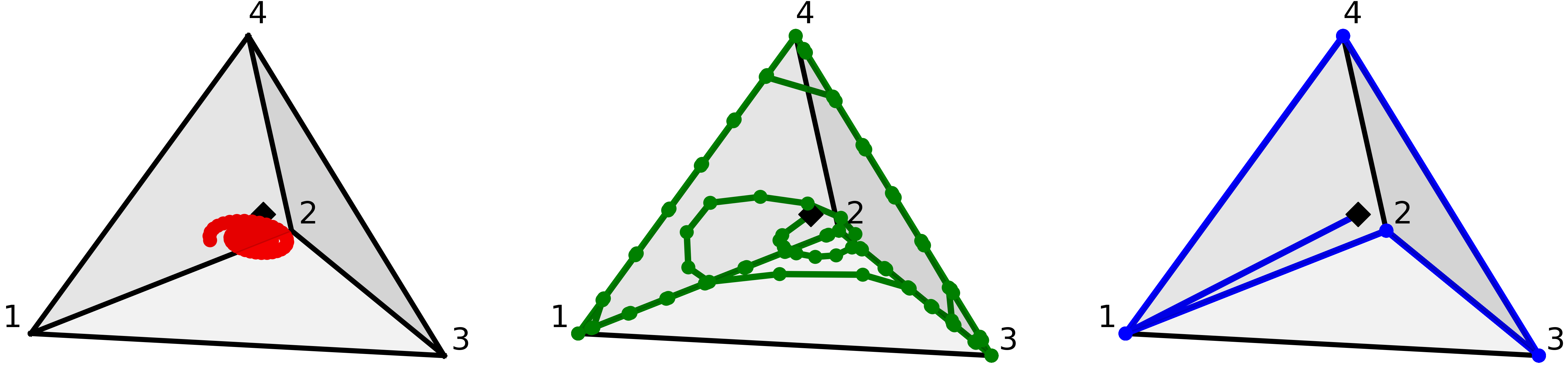}
  \caption{
    \footnotesize
    The primal iterates of Gradient Descent
    in $n=4$ unweighted RPS run for $T=100$ steps, 
    and initialized at $x^0 = [0.2, 0.2, 0.25,0.35]$.
    From left to right, using stepsize
    $\eta = \{1/\sqrt{T}, 0.3, 10\}$.}
  \label{fig:gd-eta-compare}
\end{figure}

\subsection{Lower-dimensional Representation of Dual Iterates}

Here, we establish the following
lower dimensional representation of the dual
iterates $\{y^t\}$ of Gradient Descent on a
fixed $(n-1)$-dimensional subspace.
This representation is leveraged in the proof of
Theorem~\ref{thm:gda-rps-midstepsize},
and it follows by the existence
of an interior equilibrium $x^*$
for every RPS matrix $A$ from Definition~\ref{def:rps}.
Stated formally:

\begin{restatable}{lem}{duallowdim}
  \label{lem:duallowdim}
  Let $A$ be an $n$-dimensional RPS matrix
  from Definition~\ref{def:rps}, and let
  $x^* \in \Delta_n$ denote
  an interior NE for $A$.
  Let $\{x^t\}$ and $\{y^t\}$ denote the
  iterates of~\eqref{eq:gd-pd} on $A$
  using any stepsize $\eta >  0$ and
  from any $x^0 \in \Delta_n$. 
  Then there exist positive constants
  $\alpha_1, \dots, \alpha_{n-1} > 0$
  (depending only on $A$) such that,
  for all $t$:
  \begin{equation*}
    y^t_n
    \;=\;
    - \big(
    \alpha_1 \cdot y^t_1 +
    \dots
    +
    \alpha_{n-1} \cdot y^t_{n-1} 
    \big)
    \;.
  \end{equation*}   
\end{restatable}

\begin{proof}
  Recall from Proposition~\ref{app:prelims:interiornash},
  that $Ax^* = 0$.
  Then by the skew-symmetry of $A$, observe that
  for any $x \in \Delta_n$:
  $\langle x^*, Ax \rangle
  = \langle x, A^\top x^* \rangle
  = -\langle x, Ax^* \rangle = 0$.
  Now by definition of~\eqref{eq:gd-pd},
  $y^t = \eta \sum_{k=0}^{t-1} Ax^k$, and thus
  $ 
    \langle x^*, y^t \rangle
    = 
    \eta \sum_{k=0}^{t-1}
    \langle x^*, Ax^k \rangle
    = 
    0
    $.
  It follows that we can write
  \begin{equation*}
    y^t_n x^*_n
    \;=\;
    - (x^*_1 y^t_1 + \dots + x^*_{n-1} y^t_{n-1}) \;.
  \end{equation*}
  Rearranging, and letting $\alpha_i := (x^*_i/x^*_n) > 0$ for
  all $i = 1, \dots, n-1$
  (where positivity of each $\alpha_i$ follows
  from the fact that $x^*$ is interior),
  then completes the proof. 
\end{proof}

\subsection{Sufficient Energy Growth Implies Primal Iterates on Boundary}

In this section, we give the proof of
Theorem~\ref{thm:gda-rps-midstepsize}, which
we restate here:

\gdrpsmideta*

\begin{proof}
  Let $\bar P \subseteq \R^n$ be the set defined by:
  \begin{equation}
    \bar P 
    \;=\;
    \Big\{
    y \in \R^n : 
    y_i - \frac{1}{n}\big(\langle y, \1\rangle\big)
    + \frac{1}{n} \ge 0
    \;\text{for all $i \in [n]$}\Big\} \;.
    \label{eq:Pbar}
  \end{equation}
  By the closed-form expression for the primal iterates of 
  Gradient Descent from Proposition~\ref{prop:gd-closed}, 
  observe by definition of $\bar P$
  that $\supp(x^t) = [n]$ if and only if $y^t \in \bar P$. 

  Now recall by  Lemma~\ref{lem:duallowdim} that 
  there exist positive constants $\alpha_1, \dots, \alpha_{n-1} > 0$
  such that $y^t_n = - \Big(\sum_{i=1}^{n-1} \alpha_i y^t_i\Big)$
  for all times $t$. In other words, under~\eqref{eq:gd-pd},
  the dual iterates $\{y^t\}$ all lie on an ($n-1$)-dimensional subspace
  $U$ of $\R^n$, where 
  \begin{equation}
    U
    \;=\;
    \Big\{ 
      y \in \R^n : 
      y_n = -\big(\sum\nolimits_{i=1}^{n-1} \alpha_i y_i\big)
    \Big\} 
    \label{eq:U}
  \end{equation}
  Now by slight abuse of notation, for a vector $y \in \R^n$,
  let $[y]_n \in \R^{n-1}$ denote its first $(n-1)$ coordinates.
  Then we define the subset $\bar L \subset \R^{n-1}$ 
  as the intersection of $\bar P$ and the subspace $U$:
  \begin{equation}
    \bar L
    \;=\;
    \Big\{
      z \in \R^{n-1}: 
      z = [y]_n, \;\text{for $y \in \bar P$}
    \Big\}
    \;.
    \label{eq:Lbar}
  \end{equation}
  For any $z \in \R^{n-1}$, let $y \in \R^n$  be the vector 
  $\Gamma(z) := 
  \big(z_1,  \dots, z_{n-1}, -(\sum_{i=1}^{n-1} \alpha_i z_i)\big)
  \in \R^{n}$. 
  It follows that $z \in \bar L \iff \Gamma(z) \in \bar P$. 
  Moreover, by construction $\bar L$ is compact.

  Now for $z \in \R^{n-1}$ (again by slight abuse of notation)
  let $H(z)$ be the function $\phistar(\Gamma(z))$
  (e.g., using the variable substitution
  $y_n = -(\sum_{i=1}^{n-1} \alpha_i z_i)$ 
  in the energy function $\phistar(y)$). 
  It follows that $H$ is continuous over $\R^n$,
  and by the compactness of $\bar L$, 
  it follows that $H$ has a maximum value
  $D = \max_{z \in \bar L} H(z)$ over $\bar L$. 
  Observe that $D$ is a constant depending only
  on $\{\alpha_i\}$, which depend only 
  on the matrix $A$.

  For the dual iterates $\{y^t\}$
  of~\eqref{eq:gd-pd}, 
  let $z^t$ be the corresponding iterates
  $z^t = [y^t]_n$. 
  Recall by 
  Proposition~\ref{prop:skew-gradient-descent}
  that $\phistar(y^{t+1}) - \phistar(y^t) \ge 0$
  for all $t$, which also means that 
  $H(z^{t+1}) - H(z^t) \ge 0$ for all $t$.
  Then by the arguments above, suppose 
  at some time $t' > 0$ that $\phistar(y^{t'}) = H(z^{t'}) > D$. 
  It follows that $z^{t'}\notin \bar L$,
  which implies $y^{t'} \notin \bar P$
  and thus $\supp(x^{t'}) \neq [n]$
  (meaning $x^{t'}$ is on the boundary of $\Delta_n$). 
  As $H$ is non-decreasing over time, 
  for every subsequent iterate $z^t$ for $t \ge t'$,
  we must also have $H(z^t) > D$.
  By the same reasoning, $x^t$ is on the boundary of $\Delta_n$,
  which yields the statement of the theorem.
\end{proof}

\subsection{Regret Bound for Time-Vanishing Stepsize}

In this section, we give the proof of
Lemma~\ref{lem:gd-smallsteps}, restated below: 

\gdsmallsteps*

\begin{proof}
  Recall Proposition~\ref{prop:gd-closed} 
  and expression~\eqref{eq:energy-interior} that
  under~\eqref{eq:gd-pd},
  if $x^t$ is interior (meaning $\supp(x^t) = [n])$,
  then the energy of the corresponding dual iterate $y^t$ 
  is given by:
  \begin{equation}
    \phistar(y^t) 
    \;=\; 
    \frac{\|y^t\|^2_2}{2} 
    + 
    \frac{\langle y^t, \1 \rangle}{2n}
    - 
    \frac{(\langle y^t, \1 \rangle)^2}{2n}
    - \frac{1}{2n} 
    \;,
    \label{eq:phistar-interior}
  \end{equation}
  which is 1-smooth. 
  As $x^t$ is interior for all $t \ge 0$ by assumption, 
  the energy function $\phistar(y^t)$ 
  is as given in~\eqref{eq:phistar-interior} for all $t \ge 1$.
  Then by smoothness, we have for all $t \ge 1$ that 
   \begin{align*}
    \phistar(y^{t+1})-\phistar(y^t) &\leq \langle \nabla \Psi(y^t), y^{t+1}-y^t\rangle + \frac{1}{2}\|y^{t+1}-y^t\|^2_2\\
    &= \langle \nabla \phistar(y^t), \eta A\nabla\phistar(y^t)\rangle
    + \frac{\eta^2}{2}\|A\nabla\phistar(y^t)\|^2_2\\
    &= \frac{\eta^2}{2}\|A\nabla\phistar(y^t)\|^2_2\\
    &\leq  \frac{\eta^2 \cdot n a_{\max}^2}{2}\;,
  \end{align*}
  where the second equality follows by skew-symmetry of $A$,
  and the final inequality follows by the fact that
  $\nabla \phistar(y^t) = x^t$ is a probability distribution.
  By definition, $\phistar(y^0) = \phistar(0) = -1/2n$, 
  and it  follows that
  \begin{equation*}
    \phistar(y^{T+1}) 
    \;=\;
    \sum_{t=0}^{T} \phistar(y^{t+1}) - \phistar(y^t)
    \;\le\;
    T \cdot
     \frac{\eta^2 \cdot n a_{\max}^2}{2}
     \;=\;
      \frac{n a_{\max}^2}{2} \;,
  \end{equation*}
  where the last inequality follows from the setting of 
  $\eta = \frac{1}{\sqrt{T}}$.
  This proves the first claim of the lemma. 

  On the other hand, by Proposition~\ref{prop:regret-energy-ftrl}, 
  we find (recalling that $\max_{x \in \Delta_n} (\|x\|^2_2) / 2 = 1$) that
  \begin{equation*}
    \reg(T) 
    \;\le\;
    \frac{2\cdot \phistar(y^{T+1})}{\eta} + \frac{2}{\eta}
    \;\le\;
    \sqrt{T}\cdot
    \Big(\frac{n a_{\max}^2}{2} + 1\Big) \;,
  \end{equation*}
  which proves the second statement of the lemma. 
\end{proof}

\section{Experimental Results}
\label{app:simulations}

Throughout the plots in this paper, we use (\scalebox{0.6}{$\blackdiam$}) 
to denote the \emph{initial condition} of the dynamics.

\subsection{Behavior of Primal Iterates}

\paragraph{Comparing the behavior of GD in three 
and four-dimensional RPS.}
In Figure \ref{fig:RPSsimplex} we show the
behavior of the primal iterates of Gradient Descent 
on unweighted $n=3$ RPS, analogous to plots 
for unweighted $n=4$ RPS from Figure \ref{fig:gd-eta-compare}. 
Comparing the $n=3$ and $n=4$ experiments, we observe that 
the small ($\eta=1/\sqrt{T}$) and large ($\eta=10$) stepsize
regimes result in qualitatively similar behavior: 
the GD iterates stay interior in the small stepsize regime, 
and they converge to the boundary in one step in 
the large stepsize regime. 

\begin{figure}[htb!]
    \centering
\includegraphics[width=.75\textwidth]{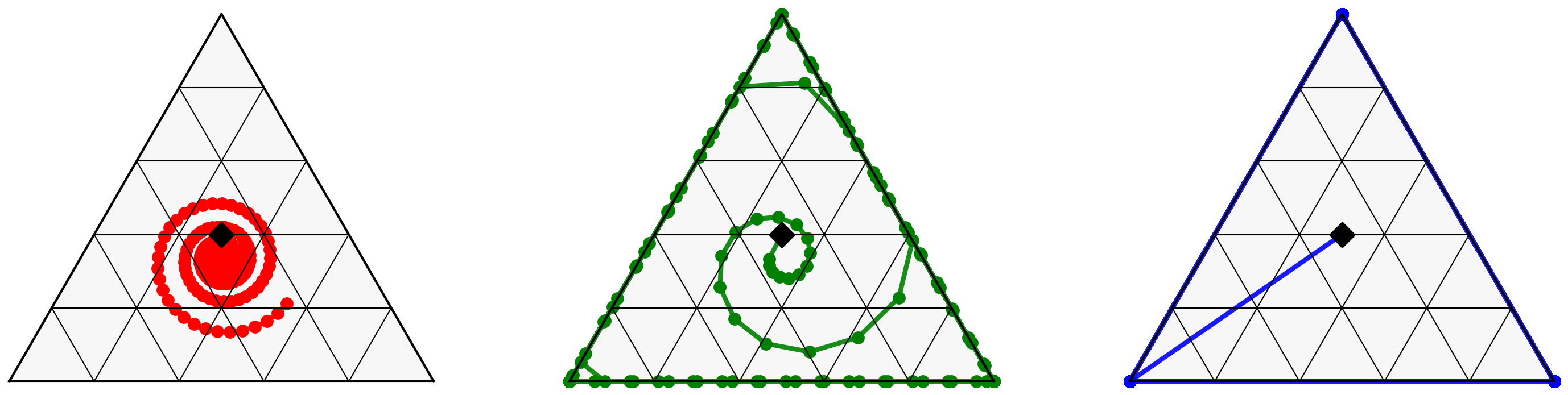}
    \caption{%
    \small 
    Primal iterates of GD in $n=3$ unweighted RPS
    initialized at $x^0 = [0.3, 0.4, 0.3]$ and run for $T=100$ steps.
    From left to right, we show the iterates using stepsizes 
    $\eta = \{1/\sqrt{T}, 0.3, 10\}$.}
    \label{fig:RPSsimplex}
\end{figure}

\begin{figure}[htb!]
\centering
\subfigure[FP]{\label{fig:FPtetra}
    \includegraphics[width=0.25\textwidth]{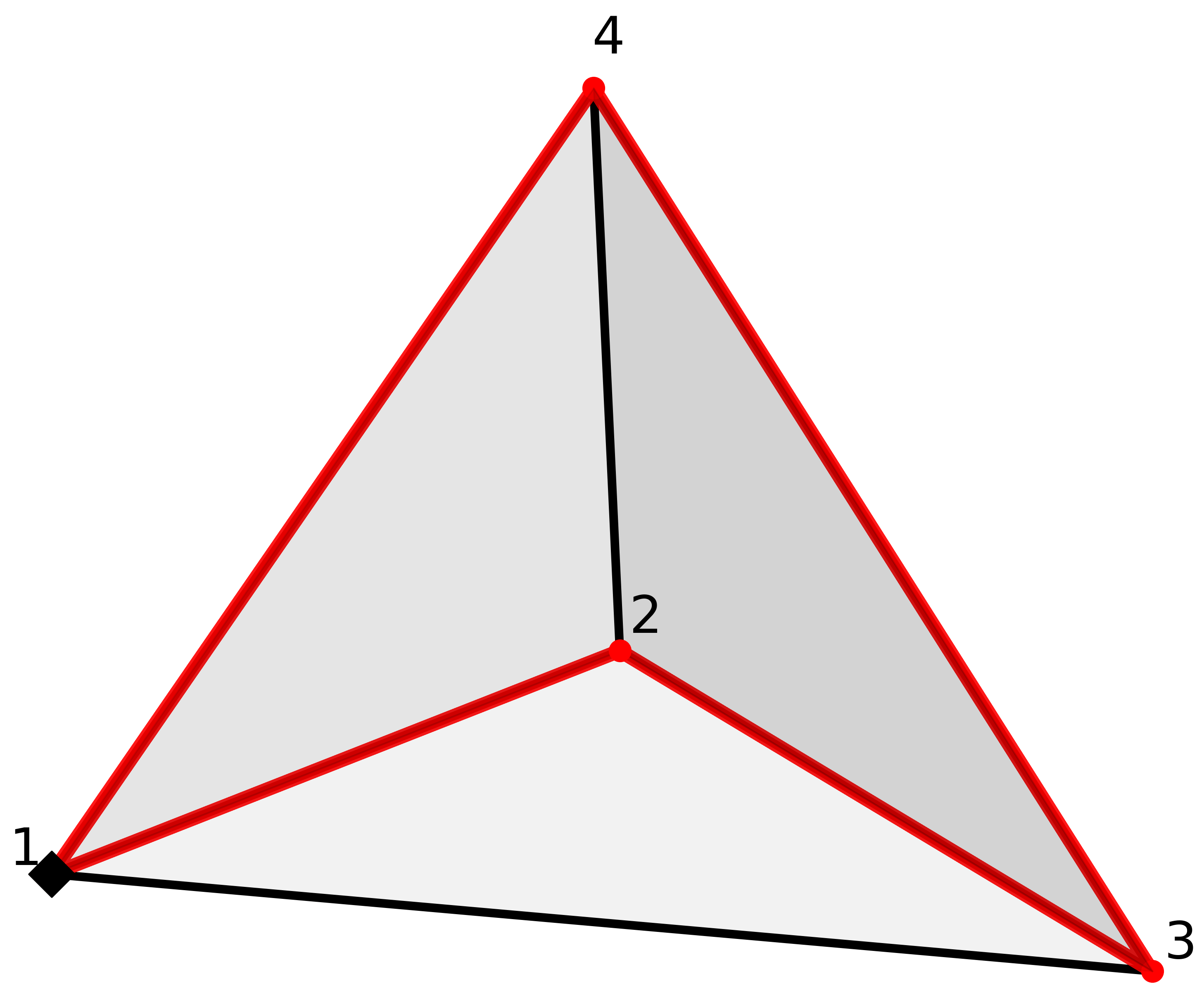}}
\subfigure[GD ($\eta=0.3$)]{\label{fig:GDMed}
    \includegraphics[width=0.25\textwidth]{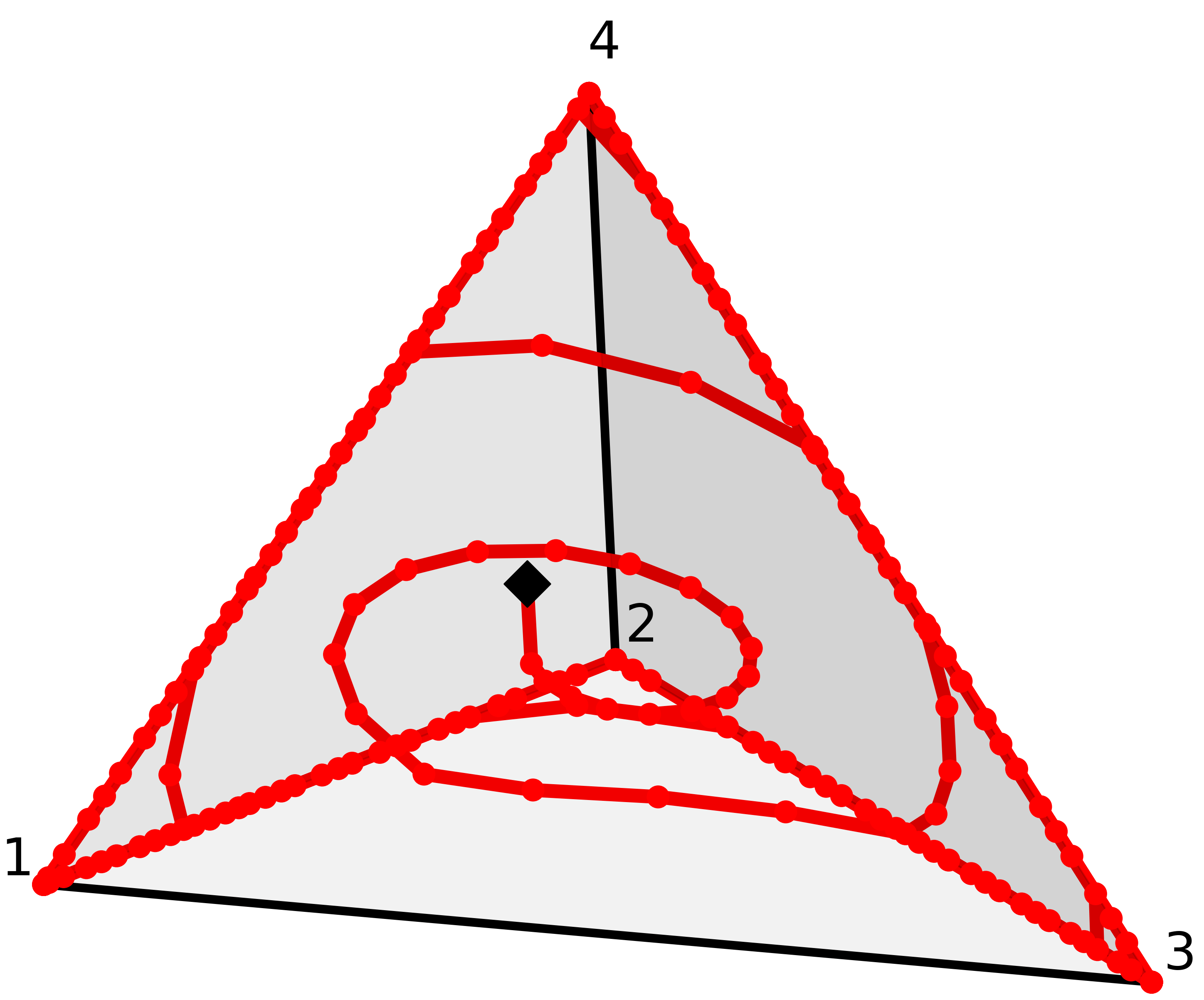}}
\subfigure[GD ($\eta=10$)]{\label{fig:GDHigh}
    \includegraphics[width=0.25\textwidth]{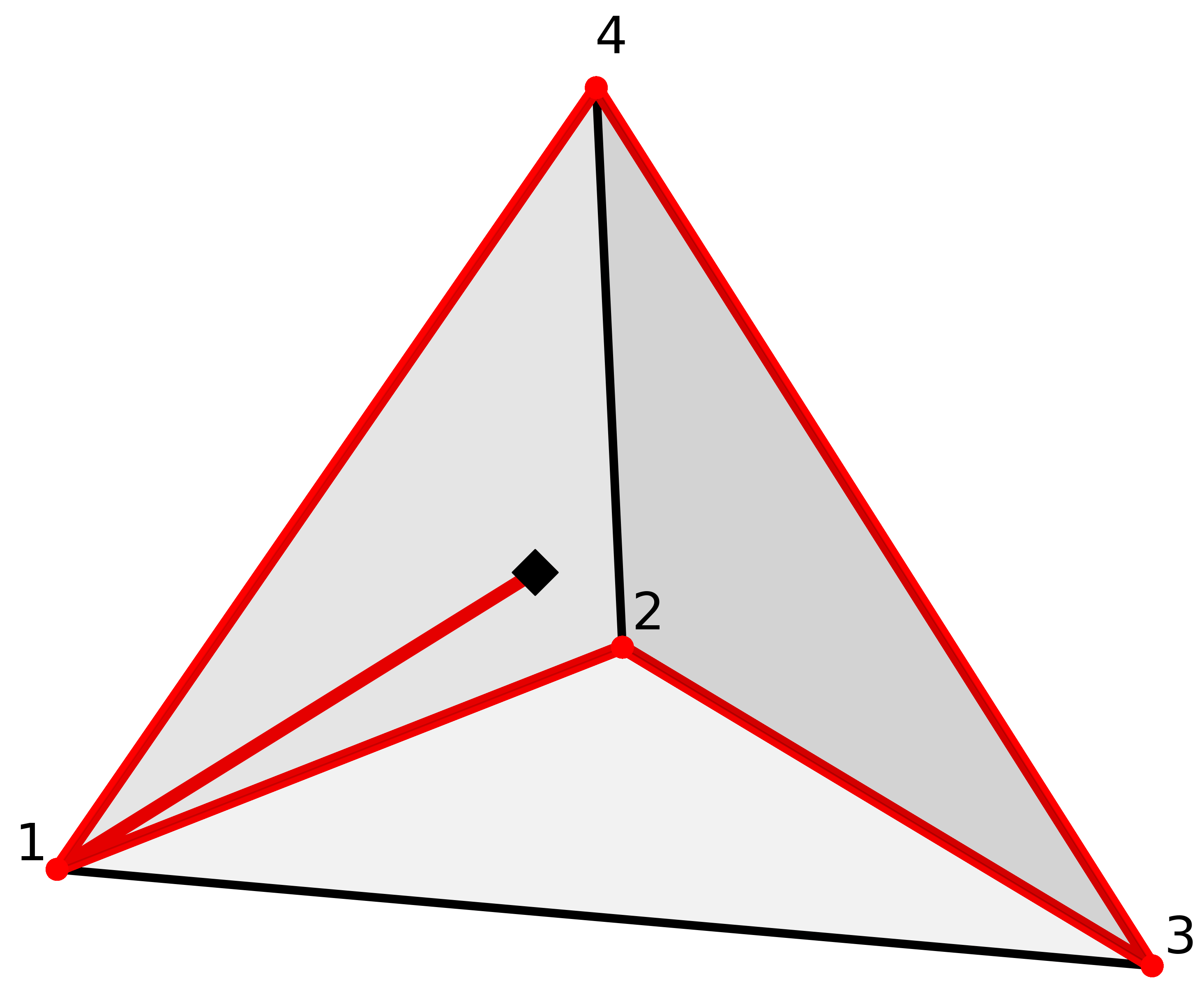}}
\caption{%
    \small
    Primal iterates of FP initialized at $x^0 = [1,0,0,0]$ 
    and GD initialized at $x^0 = [0.35, 0.05, 0.21, 0.39]$,
    run for $T=200$ steps.}
    \label{fig:RPS4primal}
\end{figure}

\paragraph{Comparing the behavior of FP and GD in four-dimensional RPS.}
In Figure \ref{fig:RPS4primal}, we plot the primal iterates
of Fictitious Play and Gradient Descent in $n=4$ unweighted RPS. 
The left plot shows Fictitious Play initialized at a vertex,
and its subsequent cycling between the four vertices. 
The right plot shows Gradient Descent initialized from an interior
distribution, with large stepsize $\eta = 10$. 
As expected from Lemma~\ref{lem:gd-large-initial}, 
the first step of the dynamic goes to vertex $e_1$,
and the iterates subsequently cycle similarly to the
FP dynamic thereafter. In the middle plot, using 
a moderate constant setting of $\eta = 0.3$, the iterates of 
Gradient Descent eventually converge and remain on the boundary, 
but in (an initially) less regular manner compared to
when $\eta = 10$.

\subsection{Empirical Regret}
To corroborate our regret bounds from
Sections \ref{sec:fp} and \ref{sec:gd}, in this section
we empirically examine the regret of Fictitious Play 
and Gradient Descent on RPS matrices. 

\paragraph{Fictitious Play.} 
Figure \ref{fig:FPRegret} shows the total regret of FP in
$n=3$ and $4$ RPS scaling like $O(\sqrt{T})$, corroborating 
Theorem~\ref{thm:fp-rps}.
In Figure~\ref{fig:tournament}, we also corroborate the
result of Theorem~\ref{thm:tournament-regret}, 
showing that when using
the `tournament' tiebreaking rule of 
Definition~\ref{def:tournament-tiebreak}, 
FP achieves constant regret.

\paragraph{Gradient Descent.} 
In Figure~\ref{fig:GDRegret}, we plot the regret obtained by GD for stepsizes 
$\eta = \{\frac{1}{\sqrt{T}}, 0.3, 10\}$ and run for $T=1000$ steps. 
Figure~\ref{fig:high} corroborates the $O(\sqrt{T})$ bound of 
Theorem~\ref{thm:gd-large-stepsize}. 
Finally, in Figure~\ref{fig:decreasing}, we compare the regret 
of Gradient Descent with \textit{time-decreasing} stepsizes 
(specifically, $\eta_t = 1/\sqrt{t}$ for $t = 1,\dots,T$) 
to Gradient Descent with large, constant stepsize ($\eta=10$). 
Surprisingly, in both Figures~\ref{fig:GDRegret} and~\ref{fig:decreasing}, 
we observe that the large stepsize setting obtains lower empirical regret 
than the time-horizon dependent/decreasing stepsize setting.
Investigating this  behavior more thoroughly is an 
interesting direction for future work.

\begin{figure}[htb!]
\centering
\subfigure[$n=3$ RPS]{\label{fig:FP3}
    \includegraphics[width=0.4\textwidth]{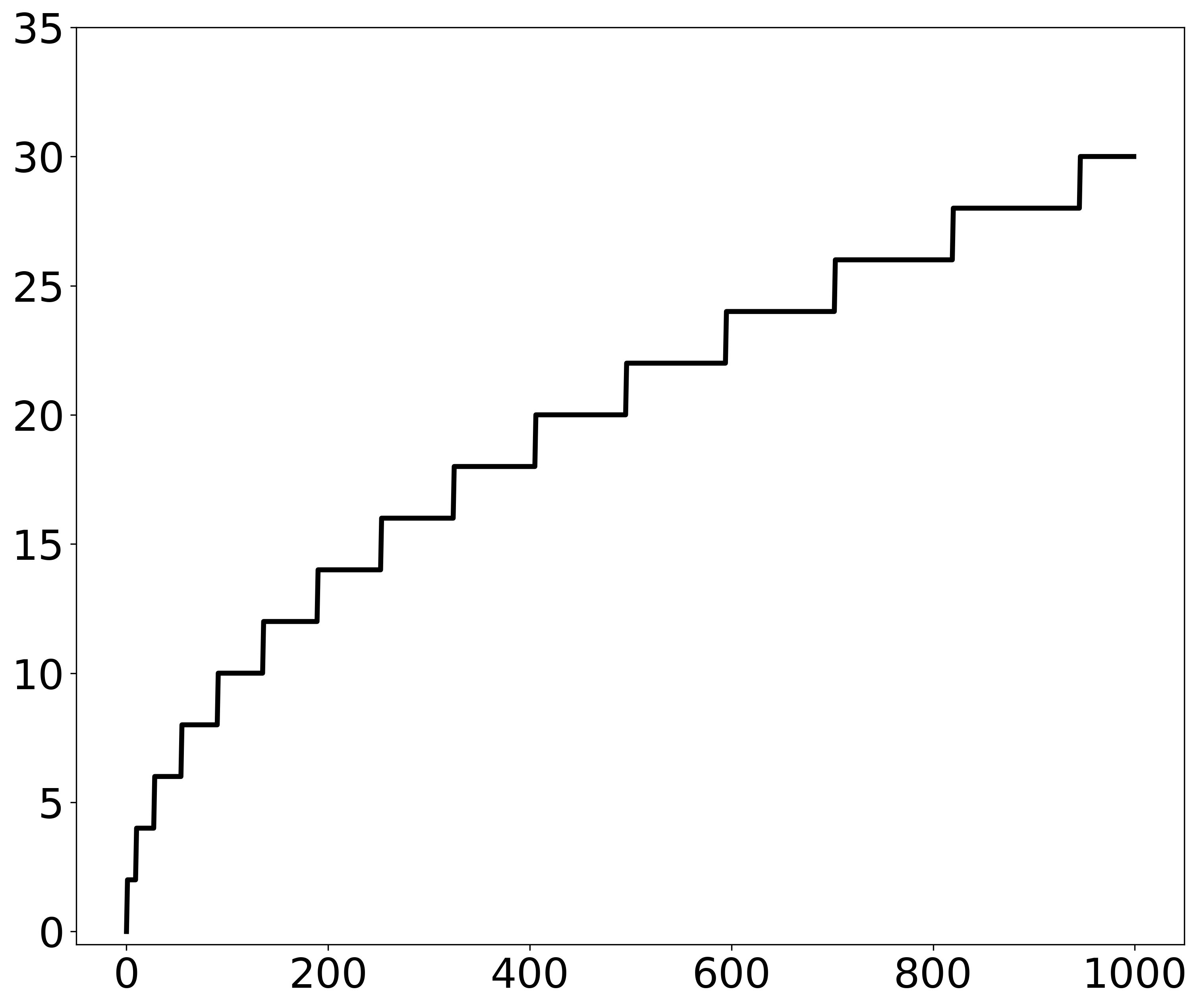}}
\subfigure[$n=4$ RPS]{\label{fig:FP4}
    \includegraphics[width=0.4\textwidth]{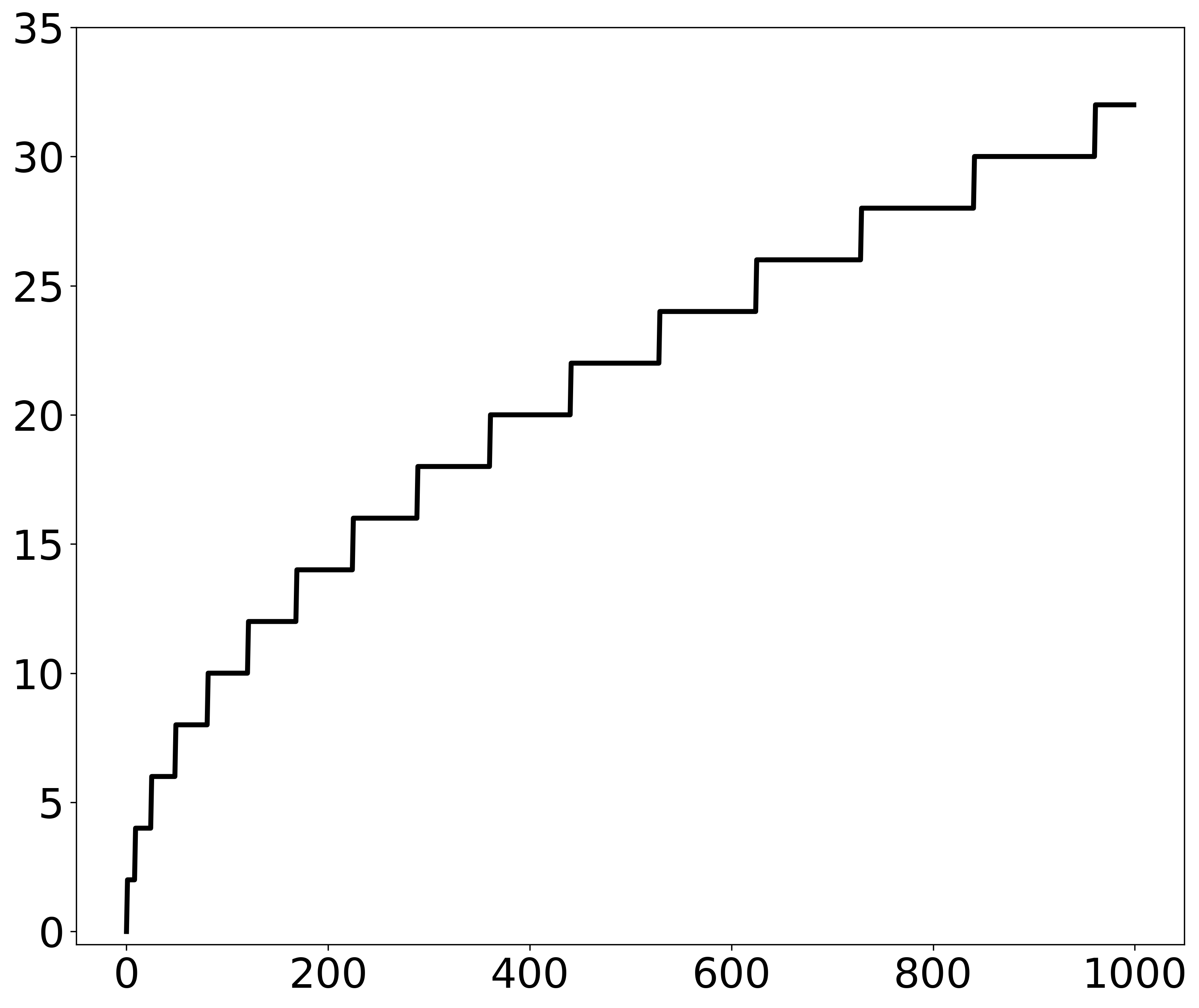}}
\caption{
    \small
    Total regret vs iteration for FP initialized at 
    $x^0 = [1,0,0,0]$ in $n=3$ and $4$ RPS, run for $T=1000$ steps.
    }
    \label{fig:FPRegret}
\end{figure}

\vspace*{1em}
\begin{figure}[htb!]
\centering
\subfigure[Lexicographic tiebreaking]{\label{fig:fplexi}
    \includegraphics[width=0.4\textwidth]{figures/RegFP3.png}}
\subfigure[Tournament tiebreaking]{\label{fig:fptournament}
    \includegraphics[width=0.4\textwidth]{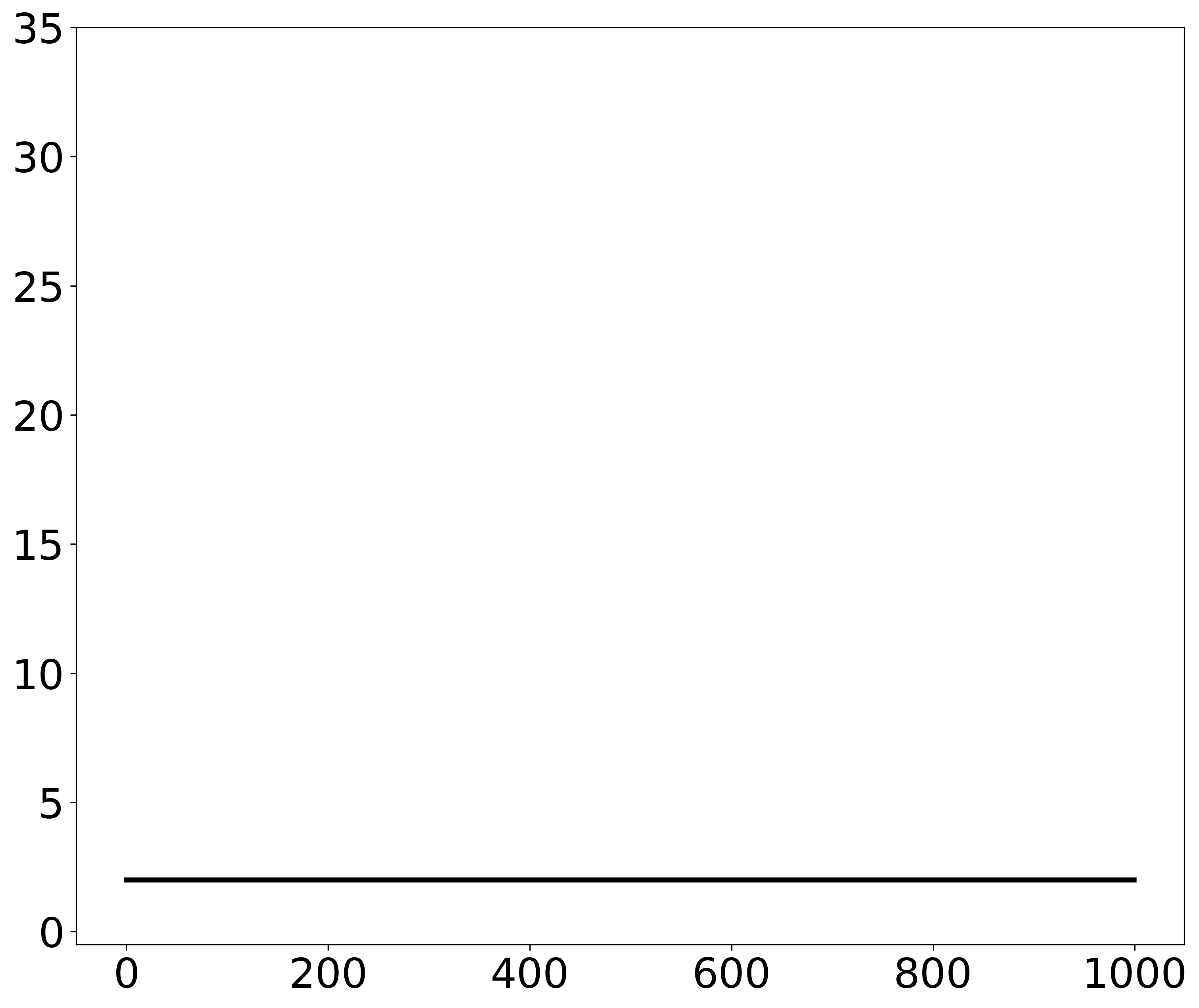}}
\caption{
    Total regret vs iteration for FP with different
    tiebreaking rules initialized at $x^0 = [1,0,0,0]$ 
    in $n=3$ RPS, run for $T=1000$ steps.}
    \label{fig:tournament}
\end{figure}

\begin{figure}[htb!]
\centering
\subfigure[$\eta=1/\sqrt{T}$]{\label{fig:low}
    \includegraphics[width=0.32\textwidth]{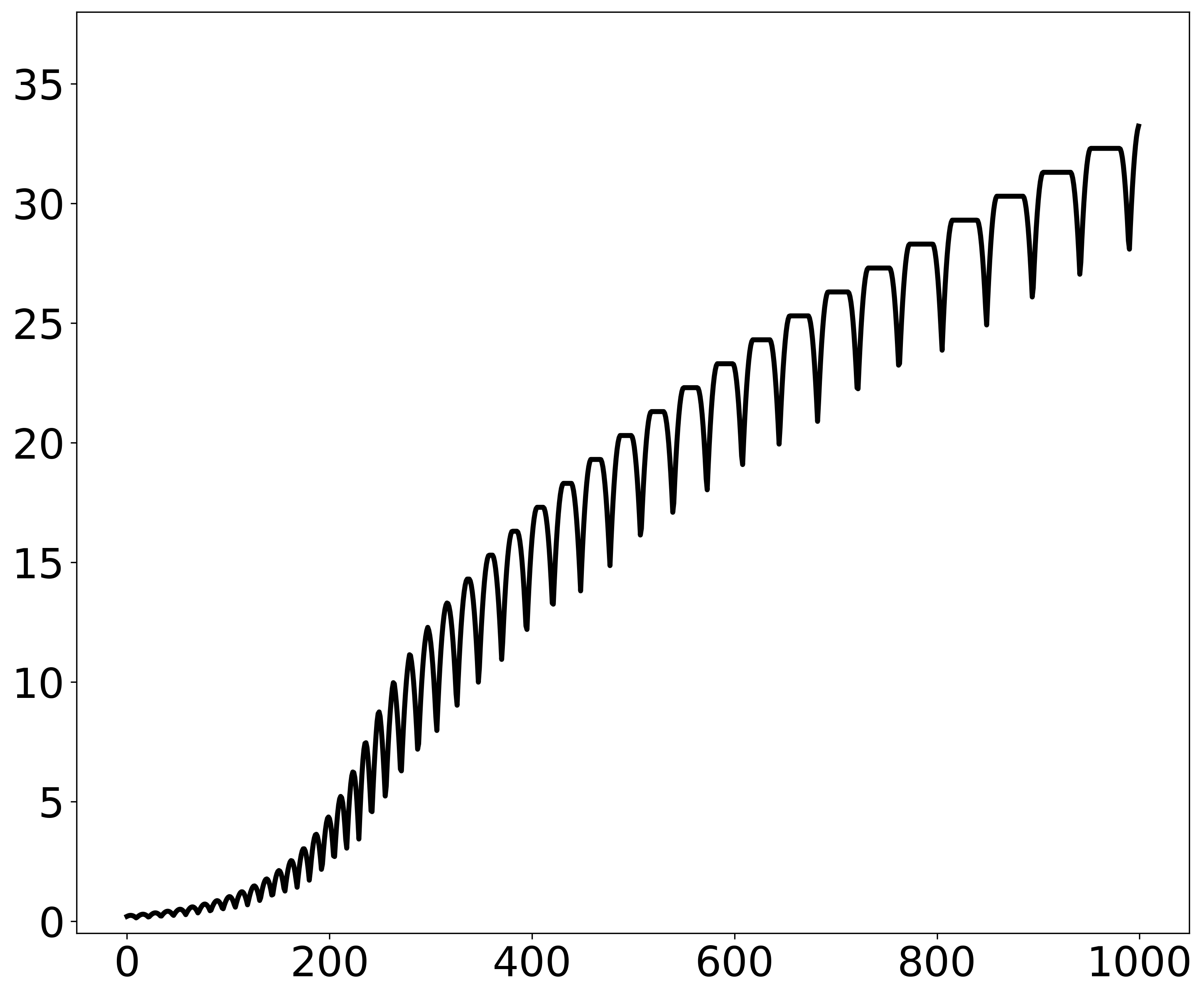}}
\subfigure[$\eta=0.3$]{\label{fig:mid}
    \includegraphics[width=0.32\textwidth]{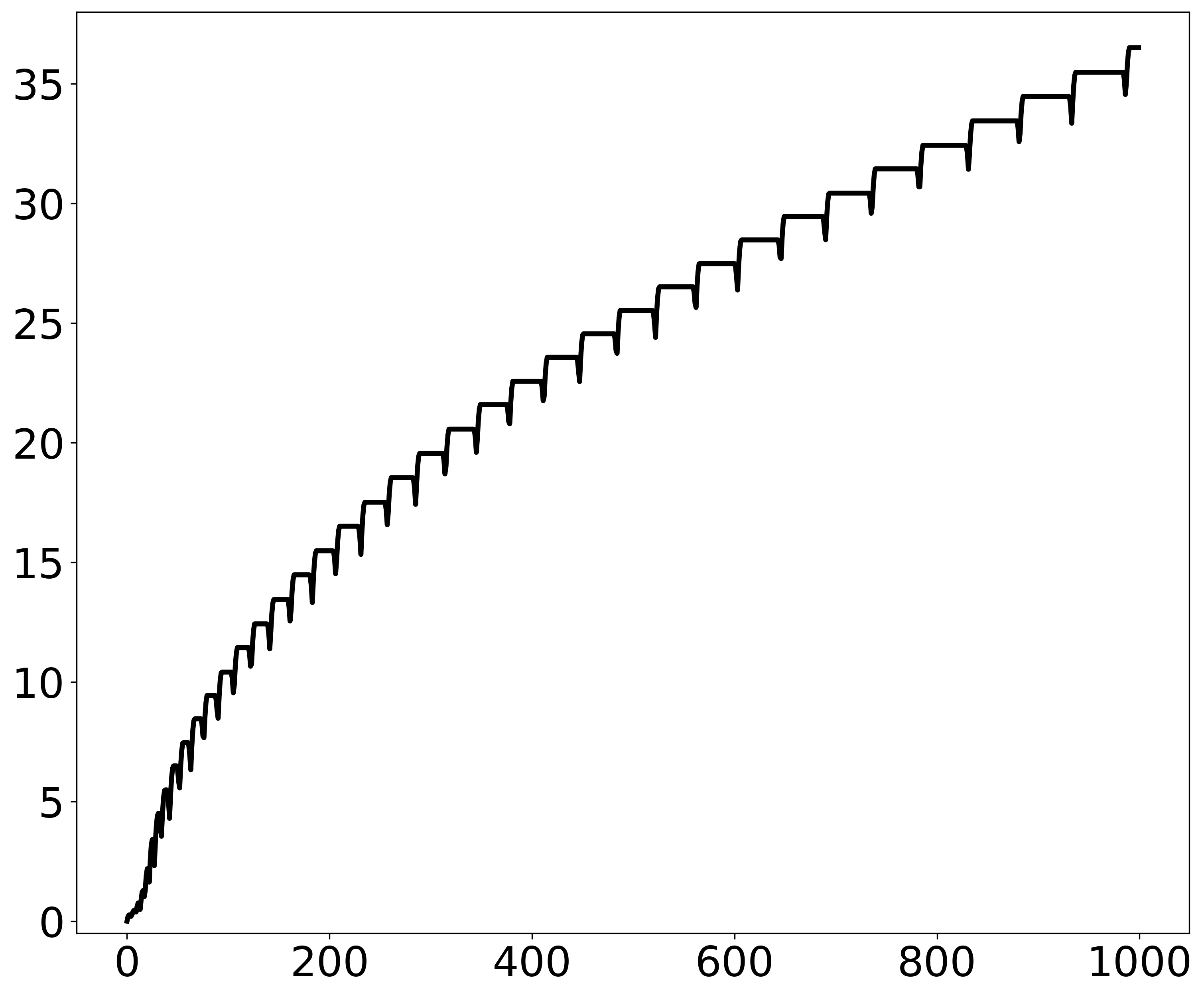}}
\subfigure[$\eta=10$]{\label{fig:high}
    \includegraphics[width=0.32\textwidth]{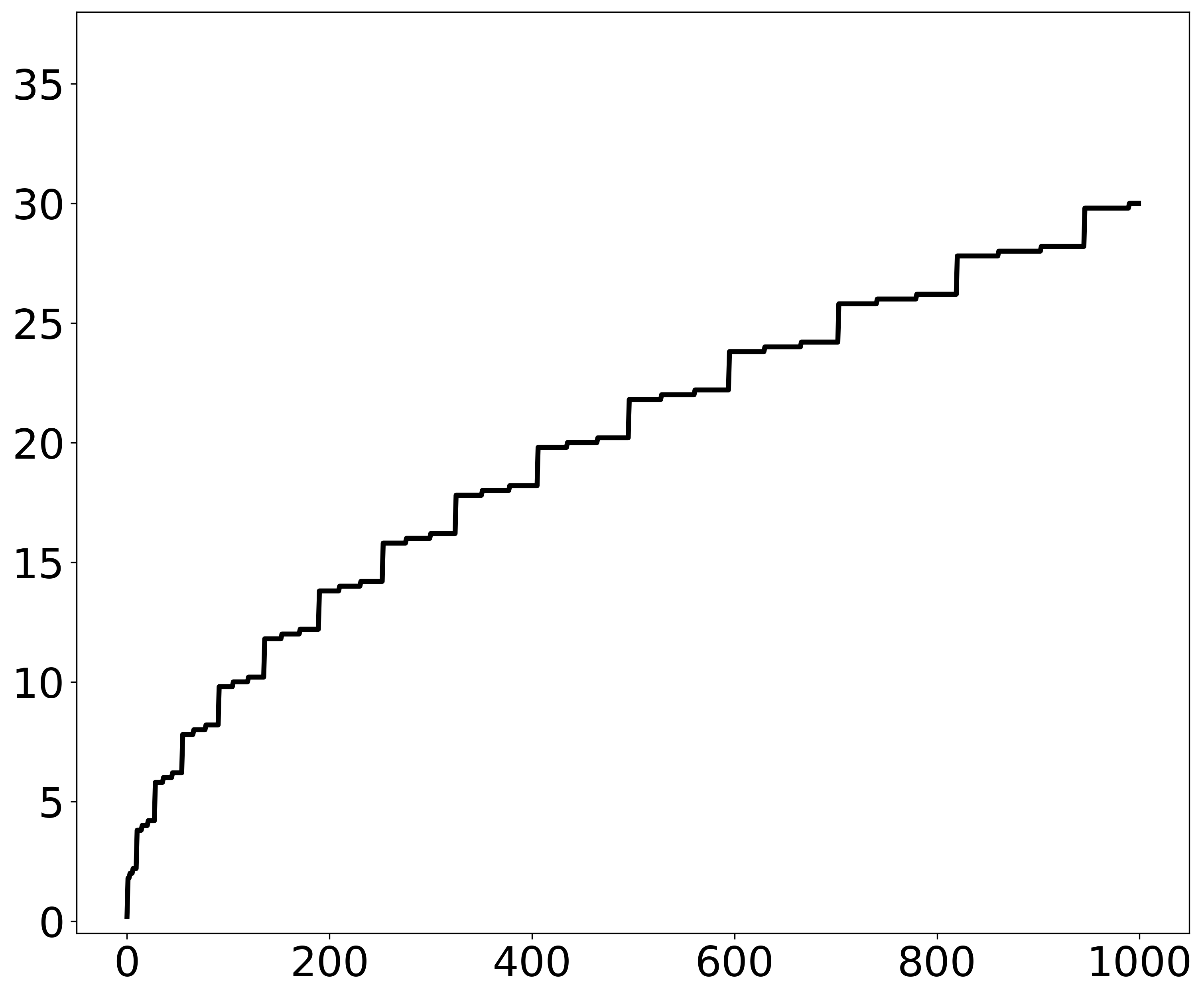}}
\caption{%
    \small
    Total regret vs iteration for Gradient Descent initialized at
    $x^0=[0.3,0.4,0.3]$ in $n=3$ RPS with stepsizes
    $\eta = \{1/\sqrt{T}, 0.3, 10\}$, run for $T=1000$ steps.}
    \label{fig:GDRegret}
    \vspace*{3em}
\end{figure}

\begin{figure}[htb!]
\centering
\subfigure[$\eta=1/\sqrt{t}$]{\label{fig:GDReg5000Dec}
    \includegraphics[width=0.42\textwidth]{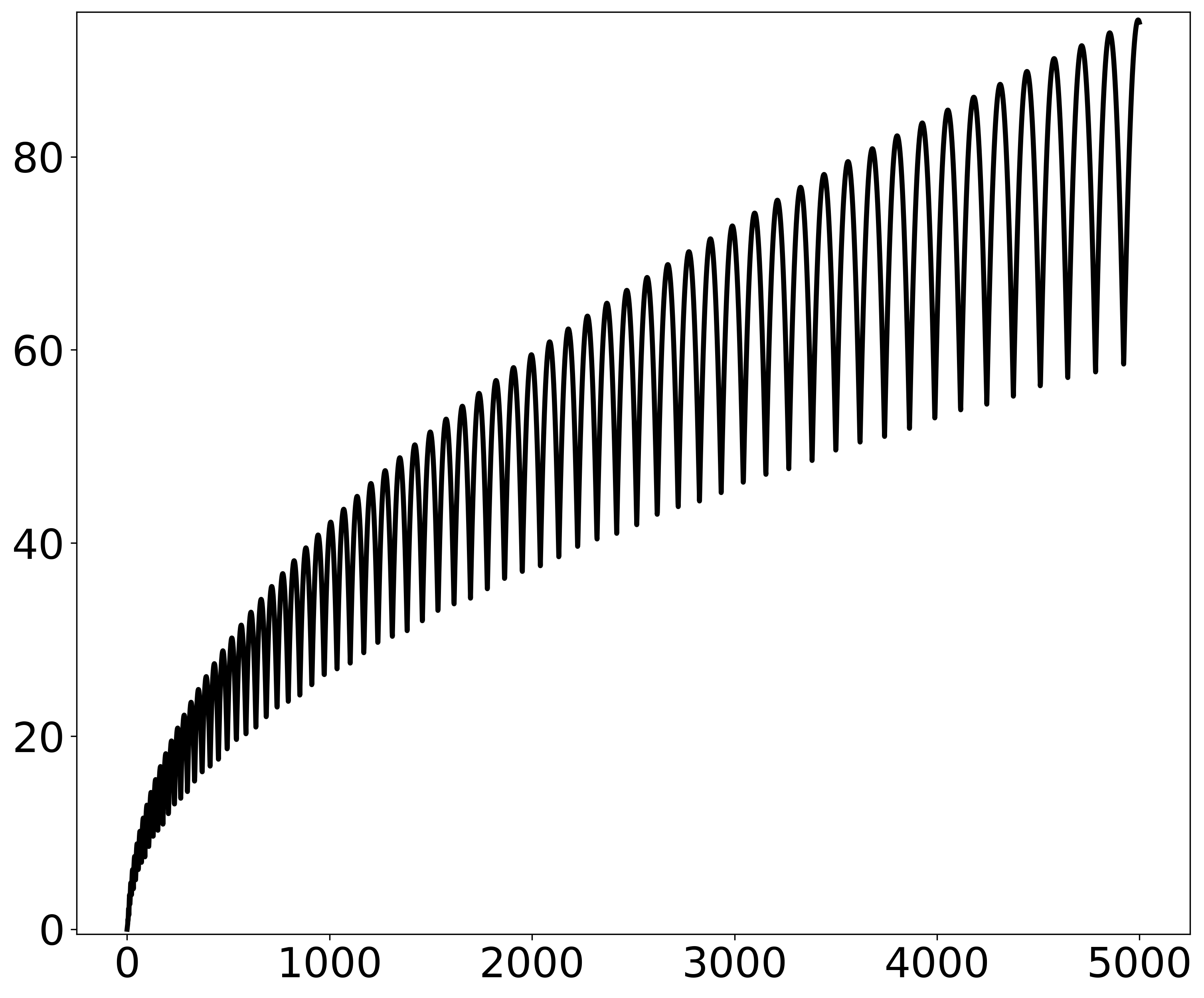}}
\subfigure[$\eta=10$]{\label{fig:GDReg5000}
    \includegraphics[width=0.42\textwidth]{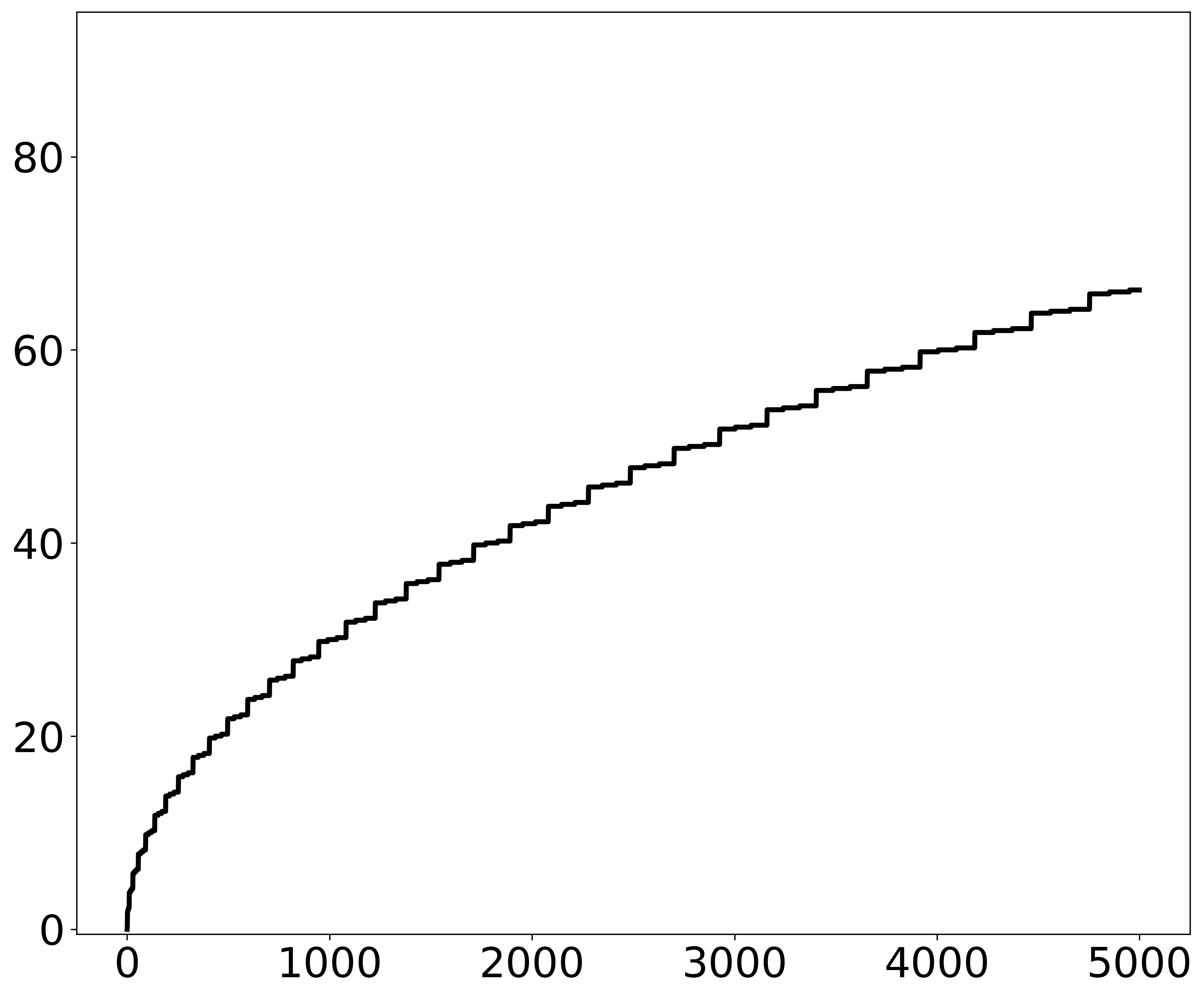}}
\caption{%
    \small
    Total regret vs iteration for Gradient Descent 
    initialized at $x^0=[0.3,0.4,0.3]$ in $n=3$ RPS with stepsizes 
    $\eta = \{1/\sqrt{t}, 10\}$, run for $T=5000$ steps.
    }
    \label{fig:decreasing}
\end{figure}

\end{document}